\RequirePackage{fix-cm}
\documentclass[natbib,smallcondensed]{svjour3}       
\usepackage{graphicx}

\usepackage{natbib}
\setcitestyle{authoryear}
\usepackage[colorlinks=true,linkcolor=blue,citecolor=blue]{hyperref}
\usepackage{url}
\usepackage{booktabs}
\usepackage{multirow}
\usepackage{amssymb}
\usepackage[cmex10]{amsmath}
\usepackage{array}
\usepackage{multicol}
\usepackage{wrapfig}

\usepackage{floatrow}
\usepackage{algorithmic}
\usepackage{algorithm}

\usepackage[small]{caption}

\usepackage{epstopdf}
\usepackage{xcolor}

\newcolumntype{C}[1]{>{\centering\arraybackslash}p{#1}}

\DeclareMathOperator*{\argmax}{arg\,max}
\journalname{Data Mining and Knowledge Discovery}
\begin{document}

\title{Homophily Outlier Detection in Non-IID Categorical Data}

\titlerunning{Homophily Outlier Detection in Non-IID Categorical Data}        

\author{Guansong Pang \and Longbing Cao \and Ling Chen}


\institute{
        Guansong Pang  \at Australian Institute for Machine Learning, University of Adelaide\\
        Adelaide, SA 5000, Australia\\
        \email{pangguansong@gmail.com}\\
        This work was mainly done when Guansong Pang was with University of Technology Sydney
            \and 
            Longbing Cao (corresponding author) \at
             Data Science Lab, University of Technology Sydney\\ Sydney, NSW 2007, Australia\\
              \email{longbing.cao@uts.edu.au}           
           \and 
           Ling Chen  \at
              Center of Artificial Intelligence, University of Technology Sydney\\ Sydney, NSW 2007, Australia\\
              \email{ling.chen@uts.edu.au}
}

\date{Received: date / Accepted: date}

\maketitle

\begin{abstract}

Most of existing outlier detection methods assume that the \textit{outlier factors} (i.e., outlierness scoring measures) of data entities (e.g., feature values and data objects) are Independent and Identically Distributed (IID). This assumption does not hold in real-world applications where the outlierness of different entities is dependent on each other and/or taken from different probability distributions (non-IID). This may lead to the failure of detecting important outliers that are too subtle to be identified without considering the non-IID nature. The issue is even intensified in more challenging contexts, e.g., high-dimensional data with many noisy features. This work introduces a novel outlier detection framework and its two instances to identify outliers in \textit{categorical data} by capturing non-IID outlier factors. Our approach first defines and incorporates distribution-sensitive outlier factors and their interdependence into a value-value graph-based representation. It then models an outlierness propagation process in the value graph to learn the outlierness of \textit{feature values}. The learned value outlierness allows for either direct outlier detection or outlying feature selection. The graph representation and mining approach is employed here to well capture the rich non-IID characteristics. Our empirical results on 15 real-world data sets with different levels of data complexities show that (i) the proposed outlier detection methods significantly outperform five state-of-the-art methods at the 95\%/99\% confidence level, achieving 10\%-28\% AUC improvement on the 10 most complex data sets; and (ii) the proposed feature selection methods significantly outperform three competing methods in enabling subsequent outlier detection of two different existing detectors.

\end{abstract}

\section{Introduction}

Outliers are data objects that are rare or inconsistent from the majority of objects in a data set \citep{aggarwal2017outlieranalysis}. A broad range of applications, such as intrusion detection, fraud detection, terrorist detection and early detection of diseases, require the detection of outliers in categorical data, which is described by categorical features. Examples of categorical features are network socket features, like Internet services and protocols; demographic features, like language, nationality and profession; and some symptoms of diseases.

\subsection{Research Problems}

Numerous outlier detection methods have been introduced over the decades. However, most of existing outlier detection methods implicitly assume that the \textit{outlier factors} (i.e., outlierness scoring measures that quantify the extent of deviation from some specific norms) of data entities (e.g., feature values, combinations of multiple feature values, and data objects) are Independent and Identically Distributed (IID). This is challenged by non-IID data characteristics widely appearing in many real-world applications, i.e., the outlierness of the data entities is interdependent and drawn from heterogeneous distributions. For example, suppose people diagnosed with type-2 diabetes are outliers, then the outlierness of the symptoms in three highly relevant features `weight loss', `thirsty' and `urination' are interdependent rather than independent, e.g., the outlierness of having the symptom of excessive thirsty can be largely influenced by the outlierness of some other co-occurred symptoms such as weight loss and frequent urination\footnote{Having excessive thirsty, weight loss and frequent urination are the abnormal concurrent symptoms in diagnosing type-2 diabetes according to https://www.diabetesaustralia.com.au/.}. As a result, existing methods likely fail to detect some important outliers, e.g., outliers that are too subtle to be identified without analyzing the coupling relationships between \textit{outlying behaviors} (i.e., exceptional values that deviate significantly from other values in the same features). This issue may be much more severe in challenging outlier detection environments, such as in high-dimensional data sets with many noisy features (i.e., features in which normal objects manifest infrequent behaviors while outliers show frequent behaviors).

In this work, we are interested in detecting outliers in \textit{categorical data} with non-IID outlier factors. Learning non-IID information involves a broad range of explicit and/or implicit couplings and heterogeneities \citep{cao2012coupled,cao2014computer,cao2015coupling}, which fundamentally challenge traditional IID learning techniques. This work focuses on exploring a particular coupling and heterogeneity, homophily coupling and heterogeneous probability distributions, embedded in the outlier factors.

\textit{Homophily couplings} refer to the phenomenon that an entity tends to bind with other entities that have similar characteristics and consequently the entities have mutually positive influence on their characteristics  \citep{fowler2008dynamic,mcpherson2001homophily}. For example, people who are surrounded by many other happy people are more likely to be happy people or become happy in the near future \citep{fowler2008dynamic}. By having homophily couplings in outlier detection, we posit that the outlying behaviors are explicitly and/or implicitly coupled with each other, and the outlierness of one behavior is influenced by the outlierness of other behaviors. As a result, the outlierness of a behavior is dependent on the outlierness of its coupled behaviors, and the outlierness of these coupled behaviors are further dependent on that of their own coupled behaviors. 

Existing outlier detection methods for categorical data \citep{akoglu2012comprex,das2007marp,he2005fp,otey2006fast,smets2011krimp,tang2015contextualoutliers,wu2013information,ienco2017semisupervised,aggarwal2017outlieranalysis,angiulli2008outlier,angiulli2010outlier} take the IID assumption. Such methods identify a set of normal/outlying patterns from all possible patterns and compute the outlierness of the identified patterns \textit{individually}. In doing this, they ignore the couplings between the outlierness of the patterns. Consequently, they may face critical problems, e.g., they may treat the wrongly identified patterns as important as the genuine ones and result in high detection errors. Accordingly, properly modeling the outlying behaviors with homophily couplings is critical and can iteratively reinforce the outlierness of genuine outlying patterns, which may consequently reduce the impact of the erroneous patterns.

Additionally, outlier factors may be taken from heterogeneous distributions due to \textit{heterogeneous probability distributions} in different features, e.g., features with two values follow a Bernoulli distribution while features with multiple values follow a categorical distribution. As a result, heterogeneous outlier factors w.r.t. different distributions are required to accurately capture the outlierness of the entities. However, existing methods generally assume an \textit{identical} distribution for all identified outlying/normal patterns and assign the same outlierness to patterns of the same frequencies, leading to incorrect estimation of the outlierness of patterns. 

\subsection{Our Solution - Coupled Unsupervised Outlier Detection}

This paper introduces a novel framework, called Coupled Unsupervised OuTlier detection (CUOT), as illustrated in Fig. \ref{fig:framework}, to detect outliers in categorical data with the aforementioned characteristics using \textit{non-IID outlier factors}, i.e., outlier factors that are coupled with each other and have heterogeneous settings for different feature contexts. CUOT estimates an initial outlierness of each value by modeling intra-feature value couplings using \textit{intra-feature outlier factors}, which consider the local context within a feature to yield value outlierness that is normalized and thus comparable across heterogeneous features. CUOT then leverages the inter-feature value couplings to model \textit{outlierness influence} between different values. It further integrates these two components into a \textit{value-value graph} and subsequently learns the outlierness of values by off-the-shelf graph mining techniques. The value graph representation is employed to support flexible integration and effective modeling of the outlier factors and their homophily couplings.

The defined outlierness of values can detect outliers in two ways: (i) by directly computing the outlier scores of objects through consolidating the outlierness of their values; (ii) by first measuring the relevance of a feature to outliers through consolidating the outlierness of the values in the feature, i.e., features with high outlierness are considered to be \textit{outlying features} and then selecting important features for subsequent outlier detection.

The CUOT framework is further instantiated into two instances by modeling the outlierness propagation on attributed value-value graphs using biased random walks. The first instance, namely Coupled Biased Random Walks (CBRW), defines an intra-feature outlier factor via a feature mode-based normalization and considers the mutual dependency of the outlierness of values from different features using the conditional probabilities of those values. The intra-feature outlier factors and their couplings are mapped onto a \textit{directed} attributed value-value graph and modeled by biased random walks \citep{gomez2008brws} to estimate the outlierness of all values. 

Another instance, called multiple-granularity Subgraph Densities-augmented Random Walks (SDRW), further handles noisy features. SDRW improves CBRW in the following aspects. SDRW defines a different outlierness influence scheme and works on an \textit{undirected} value graph. SDRW can therefore obtain an efficient parameter-free closed-form solution for learning value outlierness, whereas CBRW relies on power iterations for the learning. Moreover, SDRW adds a new subgraph density-based outlier factor to capture high-order homophily couplings to further enhance its tolerance to noisy features.

\subsection{Our Contributions}
Accordingly, this paper makes the following major contributions.
\renewcommand{\theenumi}{\roman{enumi}}
\begin{enumerate}
\item We introduce a new outlier detection task, outlier detection in non-IID multidimensional data. This task aims to leverage interdependent and heterogeneous outlier factors to identify outliers with non-IID outlying behaviors which are otherwise overlooked by traditional approaches using IID outlier factors.

\item A novel CUOT framework for this new task is proposed to estimate the outlier score of each \textit{value} by modeling the homophily couplings and heterogeneous distributions of value outlierness. Learning the value outlierness at the value level provides an effective and efficient way to model the non-IID outlying behaviors. Moreover, the value-level outlier scores are more fine-grained and flexible than the pattern-level scores. This approach makes outlying feature selection possible in addition to direct outlier detection.

\item CUOT is further instantiated by two methods, CBRW and SDRW, to model outlierness propagation on directed attributed value-value graphs in the case of CBRW and undirected value graphs in the case of SDRW. Both methods integrate intra-feature outlier factors and the couplings between the outlier factors in a seamless manner. Learning with such different outlier factors and their interactions enables the models to capture couplings that are genuinely relevant to outlier detection. We theoretically and empirically show that these models not only handle non-IID outlying behaviors but also handle data with many noisy features and/or low outlier separability (i.e., data containing only weakly relevant features). SDRW significantly enhances CBRW to be parameter-free, more computationally efficient, and more effective at handling homophily couplings and noisy features.



\item We quantify the data complexities of categorical data by four value-feature-object hierarchical data indicators from four different aspects. Further, a collection of data sets with such quantitative complexities is made available to promote the development and evaluation of outlier detection on complex data.  

\item This paper is built on its preliminary version \citep{pang2016outlier} and has made significant additional contributions. These include the generalization of CBRW to a novel generic framework, a new instantiation of the CUOT framework that significantly improves CBRW, and the comprehensive empirical evaluations using hierarchical data indicators from different aspects.

\end{enumerate}

Extensive experiments show that: (i) our SDRW- and CBRW-based outlier detection methods significantly outperform five state-of-the-art methods on 15 real-world data sets with different levels of non-IID outlying behaviors, outlier separability, and feature noise, which respectively achieve 16\%-28\% and 10\%-21\% AUC improvement on the 10 most complex data sets; (ii) the SDRW and CBRW methods run substantially faster than pattern-based methods; (iii) the SDRW- and CBRW-based feature selection methods can be used to significantly improve two different types of outlier detectors; (iv) the SDRW-based outlier detector performs significantly better than the CBRW-based detector, achieving more than 5\% average improvement on complex data sets; and (v) the proposed complexity indication of data sets is verified by the detection performance of various outlier detectors.

The rest of this paper is organized as follows. The problem statement and our proposed CUOT framework are detailed in Section \ref{sec:framework}. The two instances, CBRW and SDRW, are introduced in Section \ref{sec:cbrw} and \ref{sec:sdrw}, respectively. An analysis of how CBRW and SDRW handle non-IID outlying behaviors and their ability to handle low-separable/noisy data is presented in Section \ref{sec:thm}. Section \ref{sec:exp} outlines the experiment design, followed by the evaluation results in Section \ref{sec:results}. We discuss the related work in Section \ref{sec:relatedwork}. This work is then concluded with future research directions in Section \ref{sec:con}.

\section{Learning Non-IID Value Outlierness}\label{sec:framework}

\subsection{Problem Statement}

\textit{Outlier factor} is referred to as a function that assigns outlierness values to data entities, in which the \textit{entity} can be feature values, combinations of multiple values, features, and data objects, etc. For example, the inverse of the frequency of frequent patterns is a widely-used feature-value-level outlier factor in pattern-based outlier detection methods; $k$-th nearest-neighbor distance is a commonly used object-level outlier factor in distance-based methods. Non-IID outlier detection aims to learn the outlierness of a given entity by modeling the non-IID characteristics of a set of outlier factors, which is formally defined as follows.

\begin{definition}[Non-IID Outlier Detection]\label{def:noniidod}
Let $\mathsf{X} \in \mathbb{R}^{M}$ be a multivariate random variable composed by $M$ outlier factors. Then given an entity $e_i$, non-IID outlier detection methods define:
\begin{equation}\label{eqn:coupling}
	\mathsf{X}_{e_i} \not\!\perp\!\!\!\perp \mathsf{X}_{e_j},\; \exists j, 1 \leq j \leq M \; \& \; i \neq j,
\end{equation}
 or
 \begin{equation}\label{eqn:heterogeneity}
 	\mathsf{X}_{e_i} \thicksim \mathcal{D}_{i},\; \forall i, 1 \leq i \leq M.
 \end{equation}
 where $\mathcal{D}_{i}$ is an unknown distribution.
\end{definition}

Unlike most outlier detection methods that treat the outlier factors of the entities in an IID way, this definition considers the coupling relation between the outlier factors of different entities in Equation (\ref{eqn:coupling}); it also considers the heterogeneous distributions taken by different outlier factors in Equation (\ref{eqn:heterogeneity}).  The definition can be applied to different types of entities. For example, at the feature level, we may examine how the outlierness of features is affected by each other in data with interdependent mixed numeric and categorical features, e.g., how the outlierness of values in numeric features is dependent on the outlierness of their associated categorical values; at the object level, we may examine what sort of heterogeneities or couplings the object outlierness has in different clusters, e.g., how to understand the outlierness of centers in clusters of different sizes/shapes (heterogeneity) and how the outlierness of the objects in one cluster influences that in other clusters (interdependence). In an attempt to understand the low-level non-IID characteristics, this work focuses on exploring the non-IID outlier factors w.r.t. feature values. 

Let $\mathcal{X}$ be a set of data objects with size $N$, described by a set of $D$ categorical features  $\mathcal{F}=\{\mathsf{F}_1,\mathsf{F}_2,\cdots,\mathsf{F}_D\}$. Each feature $\mathsf{F} \in \mathcal{F}$ has a domain $\mathit{dom}(\mathsf{F})=\{v_{1},v_{2},\cdots\}$, which consists of a finite set of possible feature values. Note that the semantic of the domain in different features is different from each other, since each feature has a different context. We therefore assume that the domains between features are distinct, i.e., $\mathit{dom}(\mathsf{F}_i) \cap \mathit{dom}(\mathsf{F}_j) = \emptyset, \forall i \neq j$. The entire set of feature values, $\mathcal{V}$, is the union of all the feature domains: $\mathcal{V} = \cup_{\mathsf{F} \in \mathcal{F}} \mathit{dom}(\mathsf{F})$. Our problem can be stated as follows.

\begin{problem} To learn an outlierness scoring function $\phi:\mathcal{V} \mapsto \mathbb{R}$ with the assumption that: given a value $u \in \mathit{dom}(\mathsf{F}_i)$, its outlierness is based on an outlier factor $\mathsf{X}_{\mathsf{F}_i}$ with $\mathsf{X}_{\mathsf{F}_i} \not\!\perp\!\!\!\perp \mathsf{X}_{\mathsf{F}_j},\; \exists j, 1 \leq j \leq D \; \& \; i \neq j$ or $\mathsf{X}_{\mathsf{F}_i} \thicksim \mathcal{D}_{i},\; \forall i, 1 \leq i \leq D$, where $\mathsf{X}_{\mathsf{F}_i}$ characterizes the outlierness of the values in feature $\mathsf{F}_i$.
\end{problem}


\subsection{The Proposed CUOT Framework: A Graph Mining Approach}
The CUOT framework aims to incorporate intrinsic value interactions into graph representations to learn interdependent and heterogeneous outlier factors at the value level. The basic idea is to first map relevant heterogeneity and coupling information into a value-value graph, such as the graph in Fig. \ref{fig:toyexample} derived from the toy example in Table \ref{tab:toyexample} where we aim at detecting fraudulent users, and then we learn the outlierness of each node of the graph by using graph mining techniques. The graph representations are used because (i) learning from graph representations is a straightforward and effective way to capture homophily couplings, and many off-the-shelf graph mining techniques and theories can be used to support such learning; and (ii) a variety of graph representations, such as directed/undirected graphs and attributed/plain graphs, provides a multitude of options for fusing a collection of the above two different components.

\begin{figure}[h!]
\CenterFloatBoxes
\begin{floatrow}
\ttabbox
  {\caption{A Toy Example of Fraud Detection, which contains one outlier and 11 normal objects.}\label{tab:toyexample}}
  {
  \scalebox{0.75}{
  \begin{tabular}{cp{1.0cm}p{1.2cm}p{1.0cm}p{1.0cm}p{1.0cm}}
    \toprule
    \textbf{ID} & \textbf{Gender} & \textbf{Education} & \textbf{Marriage} & \textbf{Income} & \textbf{Cheat?} \\
    \midrule
    1     & male  & master & divorced & low   & yes \\
    2     & female & master & married & medium & no \\
    3     & male  & master & single & high  & no \\
    4     & male  & bachelor & married & medium & no \\
    5     & female & master & divorced & medium  & no \\
    6     & male  & PhD   & married & high  & no \\
    7     & male  & master & single & high  & no \\
    8     & female & PhD   & single & medium & no \\
    9     & male  & PhD   & married & medium & no \\
    10    & male  & bachelor & single & low   & no \\
    11    & female & PhD   & married & medium & no \\
    12    & male  & master & single & low   & no \\
    \bottomrule
    \end{tabular}%
  }
  }

\ffigbox
  {\includegraphics[width=0.40\textwidth]{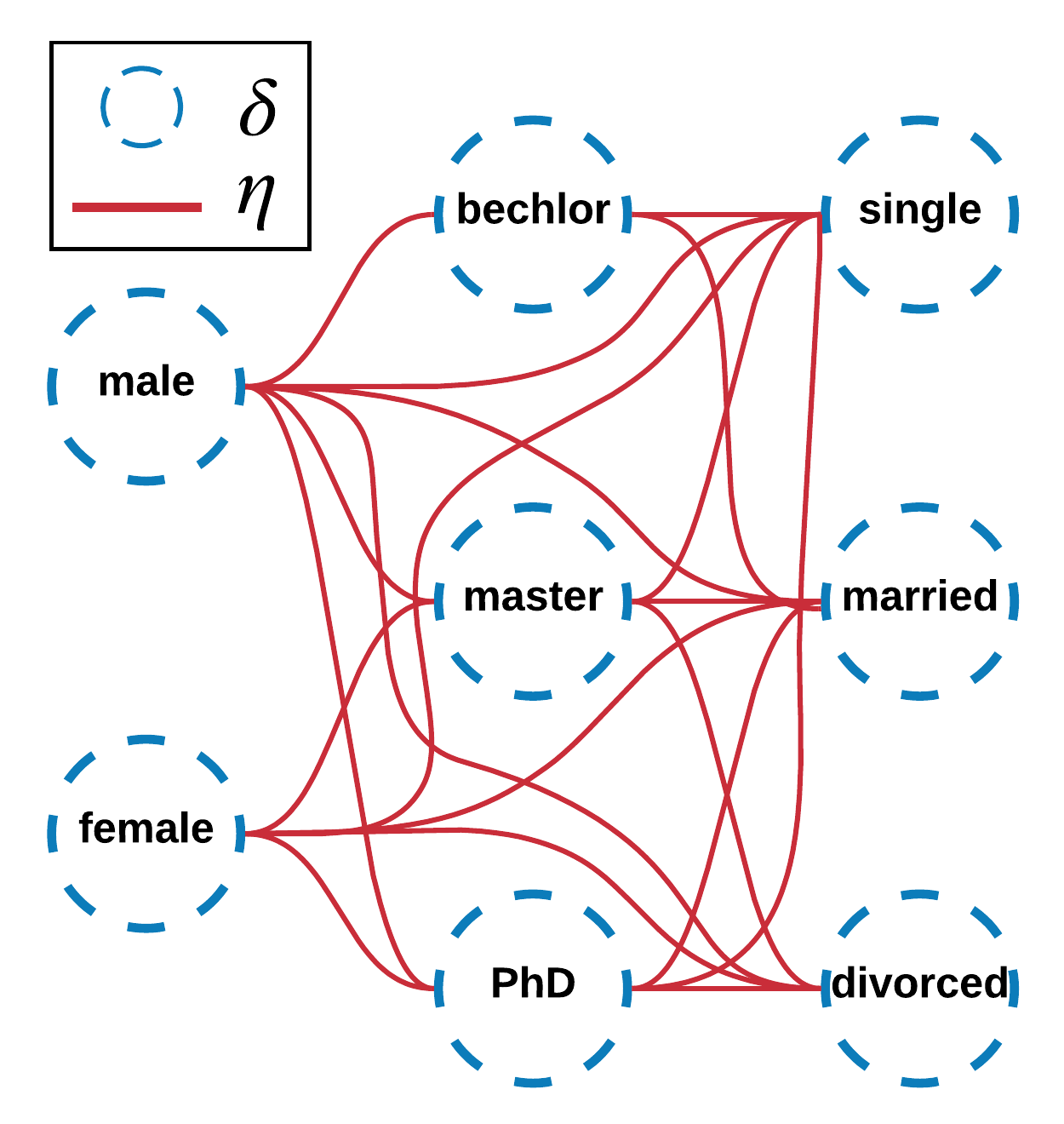}}
  {\caption{A (partial) Value-Value Graph. The $\delta$ outlier factor is applied to each node, while the $\eta$ factor is applied to each edge.}
\label{fig:toyexample}}

\end{floatrow}
\end{figure}

The procedure of the proposed framework is shown in Fig. \ref{fig:framework}. CUOT first leverages intra- and inter-feature value coupling information to capture the intrinsic data characteristics w.r.t. the outlierness of values. Specifically, given a feature $\mathsf{F} \in \mathcal{F}$, $\delta:\mathit{dom}(\mathsf{F})\mapsto \mathbb{R}$ is defined as an \textit{intra-feature outlier factor} that computes an initial outlierness of each value based on the local value couplings within the given feature. Given a feature subspaces of size two,  $\mathcal{S} \subseteq \mathcal{F}$, which is a \textit{Cartesian} product set of two features, i.e., $\mathit{dom}(\mathsf{F}_{k_{1}})\times \mathit{dom}(\mathsf{F}_{k_{2}})$, we define an \textit{inter-feature outlierness influence} between values as $\eta:\mathcal{S} \mapsto \mathbb{R}$ which considers the pairwise relations of $D$ intra-feature outlier factors w.r.t. the outlierness of a value. This results in a square matrix $\mathbf{M}_{\eta}$ that contains all pairs of value-value outlierness couplings. CUOT then maps these two components, $\delta$ and $\mathbf{M}_{\eta}$, to a \textit{value-value graph} $\mathsf{G}=<\mathcal{V},\mathcal{E},\omega_{\delta,\eta}>$ where $\mathcal{V}$ represents a set of nodes and each node represents a feature value, $\mathcal{E}$ denotes a set of edges, and $\omega:\mathcal{V}\times\mathcal{V}  \mapsto \mathbb{R} $ is an edge weighting function based on $\delta$ and $\eta$. We further define a graph-based scoring function $\phi$ for the final \textit{value outlierness estimation}. Additionally, before learning the final value outlierness, CUOT offers an option to extract a set of subgraphs, $\{\mathsf{H}_1, \mathsf{H}_2, \cdots\}$, from $\mathsf{G}$ to refine the value graph representation to have a better estimation of value outlierness.

 \begin{figure}[h!]
  \centering
  \includegraphics[width=0.45\textwidth]{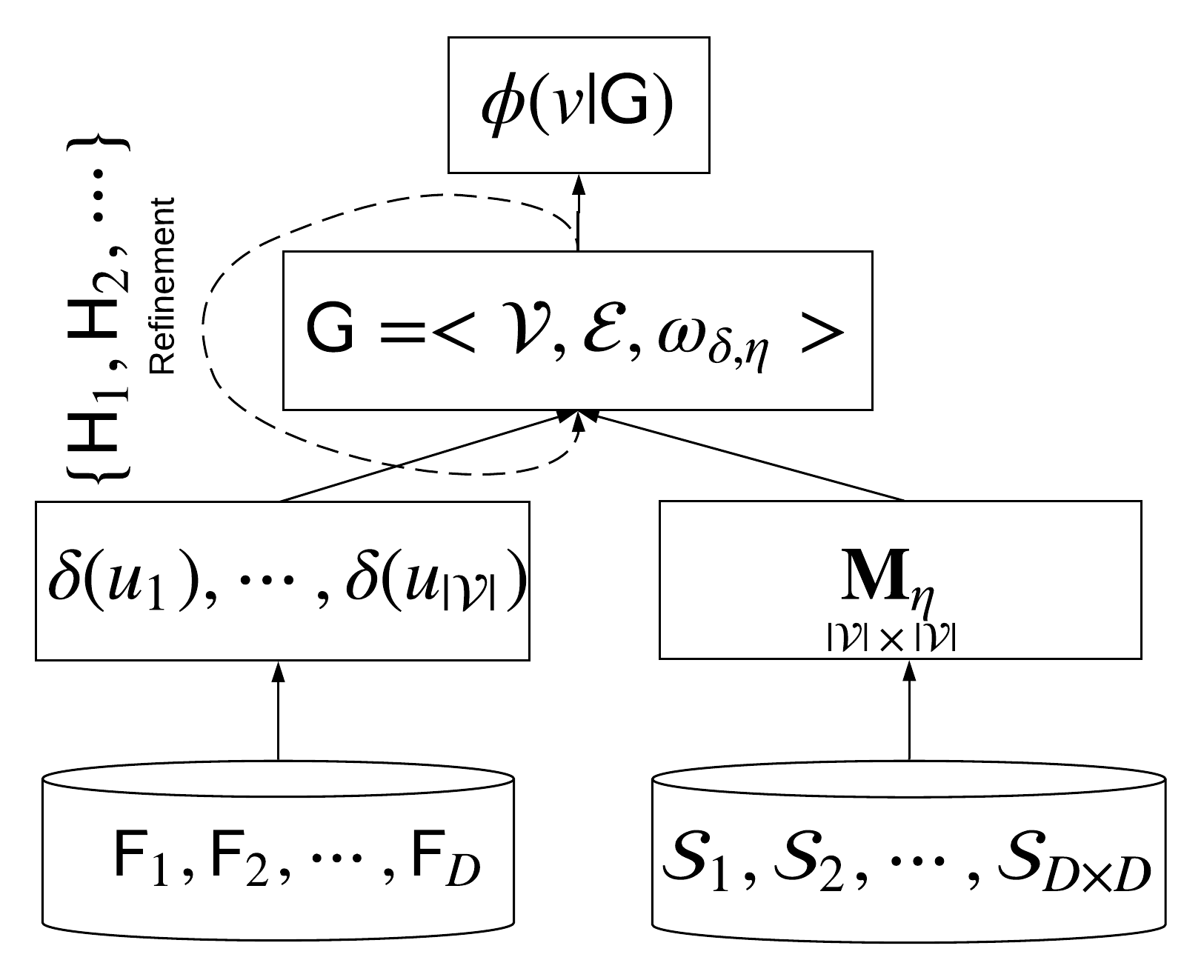}
   \caption{The Proposed CUOT Framework. $\mathsf{F}_i$ denotes an individual feature. $\mathcal{S}_{j}$ is a feature subset that contains a pair of features. $\mathcal{V}$ denotes the entire value set. $u_i \in \mathcal{V}$ is a value. $\delta$ computes an initial outlierness of feature values. $\eta$ considers the inter-feature value couplings that highlight the homophily relations between outlying values. $\mathbf{M}_{\eta}$ is a $|\mathcal{V}| \times |\mathcal{V}|$ matrix whose entries are determined by $\eta$. $\mathsf{G}$ denotes the value-value graph, $\omega$ is an edge weighting function based on $\delta$ and $\eta$, and $\phi$ is the value outlierness learning function on the value graph.}
   \label{fig:framework}
 \end{figure}

Essentially, the intra-feature outlier factor $\delta$ is used to address heterogeneous distributions, such as Bernoulli distribution taken by feature \textit{Gender} vs. categorical distributions in the other features in Table \ref{tab:toyexample}. On the other hand, inter-feature outlierness influence $\eta$ is used to link the $D$ intra-feature outlier factors and project all these relevant information into a graph. Since outlierness influence is designed to capture the homophily couplings among outlying behaviors, it is assumed that outlying behaviors are coupled with each other by co-occurrence. 
Due to their rarity, outlying behaviors exhibit strong homophily couplings as long as they are contained by a few outliers. Learning from such an outlier factor-enriched graph enables us to effectively estimate the interdependent outlierness of heterogeneous values, e.g., how the outlierness of \textit{getting divorced} and \textit{low income} affect on each other w.r.t. credit card fraud detection.

CUOT detects outliers in a way fundamentally different from existing frameworks in terms of three major aspects. First, CUOT leverages intrinsic value couplings and graph representations to capture the intra-feature value outlierness and the outlierness interdependence, resulting in a more reliable outlierness estimation in real-world data with non-IID outlying behaviors, while the existing frameworks compute the outlierness of different behaviors in an IID way. Second, CUOT learns the outlierness of \textit{values}, which is more fine-grained and flexible than the pattern-level outlierness. Third, CUOT produces value outlierness that can determine feature selection for subsequent outlier detection or directly identify outliers, whereas the existing frameworks are only aimed for direct outlier detection.

\section{CUOT's Instance I: CBRW for Estimation of Non-IID Value Outlierness} \label{sec:cbrw}

This section introduces an instantiation of CUOT, called Coupled Biased Random Walks (CBRW). CBRW works as follows. CBRW computes an initial outlierness based on the deviation of the value's frequency from the mode's frequency, and then defines a conditional probability-based outlierness influence vector. CBRW further integrates the two components in a seamless manner via a directed and attributed value graph. It finally estimates the value outlierness according to the stationary probabilities of biased random walks over the value graph. CBRW addresses the non-IID outlier factors at the value level. The mode-based normalization in the intra-feature outlier factor is designed to alleviate the heterogeneity problem, while the biased random walks on the value graph is to model the homophily couplings between outlier factors.

\subsection{Mode-based Intra-feature Outlier Factor} \label{subsec:nodeproperty}

Per the definition of outliers, the outlierness of a feature value is dependent on its rarity. CBRW employs the frequency of the \textit{mode} of a feature as a rarity comparison benchmark and examines the deviation of the value frequency to evaluate the intra-feature outlierness of a value.

Let $\mathit{supp}(v)=\left | \{ \mathbf{x}_i \in \mathcal{X}| x_{ij} = v \}\right|$ , $v \in \mathit{dom}(\mathsf{F}_{j})$, be the \textit{support} of the value $v$. Each feature $\mathsf{F}$ is associated with either a categorical distribution or a generalized Bernoulli distribution, where $\mathsf{F}$ takes on one of the possible values $v\in \mathit{dom}(\mathsf{F})$ with a frequency $\mathit{freq}(v)=\frac{\mathit{supp}(v)}{N}$. The intra-feature outlierness serves as an initial value outlierness and is specified as follows.

\begin{definition}[Mode]\label{def:mode} A \emph{mode} of a categorical distribution of a feature $\mathsf{F} \in \mathcal{F}$, denoted as $m$, is defined as a value $u_{i} \in \mathit{dom}(\mathsf{F})$ such that

\begin{equation}
   \mathit{freq}(u_{i})= max(\mathit{freq}(u_1),\mathit{freq}(u_2),\cdots,\mathit{freq}(u_O)), 
\end{equation} 
where $O$ is the number of possible values in $\mathsf{F}$. 
\end{definition}

\begin{definition}[Mode-based Initial Outlierness]\label{def:intraod} The mode-based intra-feature outlierness of a feature value $v\in \mathit{dom}(\mathsf{F})$ is defined by the frequency of the mode and the extent that the value's frequency deviates from the mode's frequency
\begin{equation}\label{eqn:mad}
\delta(v)= \mathit{base}(m) +  \mathit{dev}(v),
\end{equation}
where $base(m)=1-\mathit{freq}(m)$ denotes the outlierness of the feature mode $m$ and $dev(v)=\frac{\mathit{freq}(m)-\mathit{freq}(v)}{\mathit{freq}(m)}$ denotes the outlierness of value $v$ compared to the mode. Note that since we only consider the coupling with the mode of the features, $\delta_{m}(v)$ is hereafter simplified to $\delta(v)$ for brevity.
\end{definition}

As the location parameter (or the center) of a categorical distribution, the mode has the same semantic for different features. As shown in the following two key properties of function $\delta(\cdot)$, this specification not only guarantees the efficiency but also helps normalize the initial outlierness.

\begin{enumerate}
\item $\forall v \in \mathcal{V}$, $\delta(v) \in (0,2)$\footnote{We have ignored features with $\mathit{freq}(m)=1$, as those features contain no useful information relevant to outlier detection.}.
\item $\delta(\cdot)$ makes the intra-feature outlierness of values from features with different categorical distributions semantically comparable.
\end{enumerate}

Since $\mathit{base}(m)\in (0,1)$ and $\mathit{dev}(v)\in [0,1)$, we have $\delta(v) \in (0,2)$. For the second property, when two distributions are different in terms of their location parameters, the values drawn from these two distributions are not comparable without proper normalization. In $\delta(\cdot)$, the outlierness of the mode serves as a base, and the more the frequency of a feature value deviates from the mode frequency, the more outlying that value is. This results in a mode-based normalization, making the value outlierness comparable across features. For example, in Table \ref{tab:toyexample}, the feature values `\textit{bachelor}' and `\textit{divorced}' have the same frequency $\frac{1}{6}$, leading to the same outlierness in traditional pattern-based methods. However, the frequency distributions are different in the features `\textit{Education}' and `\textit{Marriage}', i.e., $\{\frac{1}{2},\frac{1}{3},\frac{1}{6}\}$ and $\{\frac{5}{12},\frac{5}{12},\frac{1}{6}\}$, respectively. Therefore, the two frequencies indicate different outlierness even if they are equal to each other. In our method, by applying $\delta$, the resulting initial outlierness of the value `\textit{bachelor}' is $0.58$ while that of the value `\textit{divorced}' is $0.59$.

Note that the number of values may differ largely in different features, and they may exist more than one mode in single features. $\delta$ is built on the frequency of modes, so its output outlierness is not affected by the multiple modes case since the modes in the same feature always have the same frequency. Also, the frequency of modes represents the central tendency of a given variable and thus indicates the regularity underlying the variable. This is particularly true when the sample size is sufficiently large. Therefore, deviating from the mode frequency leads to irregularity, indicating large outlierness. Additionally, mode is also a robust statistic as it is insensitive to \textit{outlying values} (i.e., infrequent values contained by outliers), making $\delta$ an effective outlierness measure for features with different number of outlying values. Other statistics like skewness or kurtosis may be adapted and used together with modes to consider the skewness or tailedness of categorical distributions in $\delta$. These high-order measures can improve $\delta$ when there are sufficiently large number of values in each feature, but they fail to work in features with two or only a few values and largely bias $\delta$. We therefore consider the mode-based $\delta$ only for the applicability in broad real-life domains where the number of values in each feature is very diverse and many of them may have only two values.

\subsection{Conditional Probability-based Outlierness Influence}\label{subsec:cp}

There is one critical condition for specifying function $\eta$ in the outlierness influence vector to capture the homophily outlying couplings: $\eta$ should be capable of contrasting the strong couplings between outlying values from the couplings between other values. Below, we discuss how the conditional probability-based $\eta$ satisfies this condition.

\begin{definition}[Conditional Probability-based Outlierness Influence Vector]\label{def:weight}
The outlierness influence vector of a value $v$ due to the other values is defined as
\begin{equation}\label{eqn:weight}
\mathbf{q}_{v}= [\eta(u_1,v), \cdots, \eta(u_{|\mathcal{V}|},v)]^\intercal=[\frac{\mathit{freq}(u_1,v)}{\mathit{freq}(v)},\cdots, \frac{\mathit{freq}(u_{|\mathcal{V}|},v)}{\mathit{freq}(v)}]^\intercal,
\end{equation}
where $\mathit{freq}(u_l,v)=\frac{\mathit{supp(u_l,v)}}{N}$ with $\mathit{supp}(u_l, v)=\left | \{ \mathbf{x}_i \in \mathcal{X}| x_{ij} = u_l \; \& \; x_{ik} = v \}\right|$, $u_l \in \mathit{dom}(\mathsf{F}_{j})$ and $v \in \mathit{dom}(\mathsf{F}_{k})$.
\end{definition}

CBRW considers the interactions of the value $v$ with all the values. Recall that $\eta(u,v)=0$ if $u$ and $v$ are from the same features, so the vector $\mathbf{q}$ captures the outlierness influence based on inter-feature value couplings. Its entry, $\eta(u,v)$, is essentially the conditional probability of $u$ given $v$, and it has three key properties.
\begin{enumerate}
\item $\eta(u,v)\in [0,1]$.
\item $\eta(u,v) \neq \eta(v,u)$ if $\mathit{freq}(u) \neq \mathit{freq}(v)$.
\item $\eta(u,v) > 0$ if $\eta(v,u) > 0$; and $\eta(u,v) = 0$ if $\eta(v,u) = 0$.
\end{enumerate}

Since $0 \leq \mathit{freq}(u,v) \leq \mathit{freq}(v) < 1$, we have $\eta(u,v)\in [0,1]$. The second and third properties follow directly from the property statements. Considering this outlierness influence helps CBRW to distinct outlying values from \textit{noisy values} (i.e., infrequent values contained by normal objects). This is because although noisy values may also have low individual frequencies, they rarely co-occur together if they are randomly distributed, leading to smaller $\eta$ than outlying values.  For example, compared to the noisy value `\textit{bachelor}', although the outlying value `\textit{low}' has lower outlierness by only considering the intra-feature outlier factor, it can have much higher final outlierness when adding inter-feature outlierness influence due to its strong couplings with the value `\textit{divorced}'. Note that \textit{normal values} (i.e., frequent values contained by normal objects), such as `\textit{married}' and `\textit{medium}', may have strong interdependency if they co-occur frequently, so they have strong outlierness influence (e.g., large $\eta$) on each other. This shows that the conditional probability-based $\eta$ sometimes may not be able to distinguish the couplings between outlying values from the couplings between frequently co-occurred normal values. We will address this issue in the second instance of CUOT in Section \ref{subsec:pmi}. However, in CBRW, large $\eta$ does not necessarily lead to high outlierness because the final outlierness is determined by both $\delta$ and $\eta$ associated with a feature value (see Section \ref{subsec:brw1} for detailed discussions). Thus, CBRW can still work well in most cases.

\subsection{Directed and Attributed Value Graph}\label{subsec:directedgraph}

It is challenging to properly integrate $\delta$ and $\mathbf{q}$ since they are of different lengths, i.e., $\delta(v)$ is a scalar while $\mathbf{q}_{v}$ is a $|\mathcal{V}|$-dimensional vector. CBRW tackles this challenge by mapping these two components onto an attributed value-value graph as follows.

\begin{definition}[Attributed Value-value Graph]\label{def:valuegraph}
The attributed value-value graph $\mathsf{G}$ is described by $\mathsf{G}=<\mathcal{V},\mathcal{E},\omega_{\delta,\eta}>$, where
\begin{itemize}
\item $\mathcal{V}$ represents the node set and each node $v\in \mathcal{V}$ represents a feature value.
\item $\mathcal{E}$ denotes a set of edges connecting the nodes, i.e., $\mathcal{E} \subseteq \mathcal{V}\times\mathcal{V}$.
\item $\delta(\cdot):\mathcal{V} \mapsto (0,1)$ is a node property mapping function using the intra-feature outlier factor in Equation (\ref{eqn:mad})\footnote{$\delta$ is normalized into the range in (0,1) to work well with $\eta$.}. 
\item $\eta(\cdot,\cdot):\mathcal{V}\times\mathcal{V}  \mapsto [0,1]$ is an initial edge weighting function using the outlierness influence vector in Equation (\ref{eqn:weight}).
\end{itemize}
\end{definition}

It is easy to see that the graph $\mathsf{G}$ is a \textit{directed} and \textit{weighted} graph without self loops according to the properties of its edge weighting function $\eta$ in Section \ref{subsec:cp}. The edge weighting function $\eta$ is different from the conventional methods that are built on similarities between the nodes. It is used because the conditional probabilities are simple and they well capture the desired homophily relationships between outlying behaviors. Note that although we do not explicitly specify the edge weighting function $\omega$, but we show in the next section that $\omega$ is equivalent to a linear combination of the $\delta$ and $\eta$ functions in a special type of random walks over the graph.

\subsection{Biased Random Walks for Learning Value Outlierness}\label{subsec:brw1}

CBRW then builds \textit{biased} random walks (BRWs) on the value graph $\mathsf{G}$ to learn the value outlierness. Let $\mathbf{A}$ be an adjacency matrix of $\mathsf{G}$, where $\mathbf{A}(u,v)$ denotes the outgoing edge weight from node $u$ to node $v$. Then we have 

\begin{equation}\label{eqn:adja}
\mathbf{A}(u,v)=\mathbf{M}_{\eta}(u,v).
\end{equation}

In building \textit{unbiased} random walks (URWs), we can obtain a \textit{walking (or transition) matrix} $\mathbf{W}$ by

\begin{equation}\label{eqn:basicform}
\mathbf{W}=\mathbf{A}\mathit{diag}(\mathbf{A})^{-1},
\end{equation}
\noindent where $\mathit{diag}(\mathbf{A})$ denotes the diagonal matrix of $\mathbf{A}$ with its $u$-th diagonal entry $d(u)=\sum_{v\in V}\mathbf{A}(u,v)$. The entry $\mathbf{W}(u,v)= \frac{\mathbf{A}(u,v)}{d(u)}$ represents the probability of the transition from node $u$ to node $v$, which satisfies $\sum_{v\in V}\mathbf{W}(u,v)=1$. 

However, URWs omit the $\delta$-based intra-feature outlierness and only consider the $\eta$-based inter-feature outlierness influence. Here, CBRW uses BRWs to introduce the intra-feature outlierness as a bias into the random walk process. This helps capture the intra-feature value couplings and the joint effects they may have with the inter-feature value couplings on the subsequent outlierness estimation. The entry of the corresponding transition matrix is defined as

\begin{equation}\label{eqn:brws1}
\mathbf{W}^{b}(u,v)=\frac{\delta(v)\mathbf{A}(u,v)}{\sum_{v\in V}\delta(v)\mathbf{A}(u,v)}.
\end{equation}

$\mathbf{W}^{b}(u,v)$ can be interpreted as that the transition from node $u$ to node $v$ has a probability proportional to $\delta(v)\mathbf{A}(u,v)$. Thus, every random move is jointly determined by the intra-feature outlierness and inter-feature outlierness influence. 

CBRW essentially simulates an outlierness propagation process over the value graph to model the homophily couplings between outlying values. The inter-feature influence vector maintains the strength of homophily couplings between outlying values during the outlierness propagation process, while the intra-feature initial outlierness enables outlying values to attract more outlierness. As a result, if $u$ and $v$ are strongly coupled and they have large outlierness, the outlierness propagation from $u$ to $v$ would be large. $v$ has large outlierness if there are many nodes having a similar relationship as $u$ to $v$. Similarly, $u$ has large outlierness if it is coupled with many outlying values. Such a cascade outlierness of each node can be effectively captured by the probability of the random walker visiting the node.

Let the vector $\boldsymbol\pi_{t}\in \mathbb{R}^{|\mathcal{V}|}$ denotes the \textit{probability distribution} of the biased random walk at time step $t$, i.e., the probability of a random walker visiting any given node at the $t$-th step. Then we have

\begin{equation}\label{eqn:updateform}
\boldsymbol\pi^{t+1} = \mathbf{W}^{b}\boldsymbol\pi^{t}.
\end{equation}

We can accordingly obtain the following theorem.

\begin{theorem}[Convergence of CBRW]\label{thm:convg}
If $\mathsf{G}$ is irreducible and aperiodic, CBRW will converge, i.e., $\boldsymbol\pi$ converges to a unique \textit{stationary probability} vector $\boldsymbol\pi^{*}$ such that $\boldsymbol\pi^{*}=\mathbf{W}^{b}\boldsymbol\pi^{*}$. 
\end{theorem}

The theorem states that $\boldsymbol\pi$ will converge to a stationary probability distribution $\boldsymbol\pi^{*}$ if the graph $\mathsf{G}$ is irreducible and aperiodic, i.e., $\boldsymbol\pi^{*}=\mathbf{W}^{b}\boldsymbol\pi^{*}$. This means that the stationary probabilities of the nodes are independent of the initialization of $\boldsymbol\pi$, and they are positively correlated to the incoming weights of the nodes. Motivated by this, we define the final value outlierness as follows.

\begin{definition}[CBRW-based Value Outlierness]\label{def:score}
The outlierness of node $v$ is defined by its stationary probability
\begin{equation}\label{eqn:score}
\phi(v)=\boldsymbol\pi^{*}(v),
\end{equation}
where $\boldsymbol\pi^{*}(v)$ is the entry w.r.t. the value $v$ in the stationary probability vector, $0< \boldsymbol\pi^{*}(v) <1 $ and $\sum_{v \in \mathcal{V}}\boldsymbol\pi^{*}(v) =1$.
\end{definition}

The value $v$ has large outlierness iff it demonstrates outlying behaviors within the feature and co-occurs with many other outlying values. This is because $\boldsymbol\pi^{*}(v)$ is proportional to $\mathbf{W}^{b}(u,v)$, which is determined by $\delta(v)$ and $\eta(u,v)$.

\subsection{The Algorithm of CBRW and Its Time Complexity}\label{subsec:alg}

The steps in CBRW are outlined in Algorithm \ref{alg:cbrw}. Steps 1-8 obtain the intra- and inter-feature value couplings. The matrix $\mathbf{W}^{b}$ is then generated based on Equations (\ref{eqn:mad}), (\ref{eqn:weight}), and (\ref{eqn:brws1}). 

\renewcommand{\algorithmicrequire}{\textbf{Input:}}
\renewcommand{\algorithmicensure}{\textbf{Output:}}
\begin{algorithm}
\caption{Coupled Biased Random Walk}
\label{alg:cbrw}
\begin{algorithmic}[1]
\REQUIRE $\mathcal{X}$ - data objects,  $\alpha$ - damping factor
\ENSURE $\boldsymbol\pi^{*}$ - the stationary probability distribution
\FOR{ $i = 1$ to $D$}
\STATE Compute $\mathit{freq}(v)$ for each $v\in dom(\mathsf{F}_i)$
\STATE Find the mode of $\mathsf{F}_i$ 
\STATE Compute $\delta(v)$
\FOR{ $j = i+1$ to $D$}
\STATE Compute $\mathit{freq}(u,v)$, $\forall u\in dom(\mathsf{F}_j)$
\ENDFOR
\ENDFOR
\STATE Generate the matrix $\mathbf{W}^{b}$
\STATE Initialize $\boldsymbol\pi^{*}$ as a uniform distribution
\REPEAT
\STATE $\boldsymbol\pi^{*} \leftarrow (1-\alpha)\frac{1}{|\mathcal{V}|}\mathbf{1}+\alpha\mathbf{W}^{b}\boldsymbol\pi^{*}$
\UNTIL Convergence, i.e., $|\Delta \boldsymbol\pi^{*}| \leq 0.001$ or reach the maximum iteration $I_{max}=100$ 
\RETURN $\boldsymbol\pi^{*}$
\end{algorithmic}
\end{algorithm}

Following \citep{page1999pagerank}, Step 12 introduces the damping factor $\alpha$ into Equation (\ref{eqn:updateform}) to guarantee the convergence of the random walks
\begin{equation}\label{eqn:updateform2}
\boldsymbol\pi^{t+1} = (1-\alpha)\frac{1}{|\mathcal{V}|}\mathbf{1} + \alpha\mathbf{W}^{b}\boldsymbol\pi^{t}.
\end{equation}

This is because different data sets may contain very different value couplings; therefore, the assumption that the graph $\mathsf{G}$ is irreducible and aperiodic may not always hold in practice. The damping factor in Equation (\ref{eqn:updateform2}) is commonly used to remedy this problem, which is justified as follows.

\begin{corollary}[Teleporting Random Walks Guaranteeing Convergence]\label{thm:convg_trw}
By setting $\mathbf{W}^{b}=(1-\alpha)\frac{1}{|\mathcal{V}|}\mathbf{1}+\alpha\mathbf{W}^{b}$, where $\alpha \in [0,1)$, $\mathbf{W}^{b}\boldsymbol\pi$ will always converge to a unique probability vector $\boldsymbol\pi^{*}$, i.e., $\boldsymbol\pi^{*}=\mathbf{W}^{b}\boldsymbol\pi^{*}$.
\end{corollary}
\begin{proof}
It is obvious that $\mathbf{W}^{b}$ becomes a real positive square matrix by the addition of $(1-\alpha)\frac{1}{|\mathcal{V}|}\mathbf{1}$. This guarantees that $\mathbf{W}^{b}$ is irreducible and aperiodic. We therefore will always have $\boldsymbol\pi^{*}=\mathbf{W}^{b}\boldsymbol\pi^{*}$.
\end{proof}

In our experiments, we set $\alpha=0.95$ directly rather than learning the parameter. There are two main reasons to do this. (i) The parameter $\alpha$ has an explicit meaning, so users can easily determine their own setting based on the application contexts. A discussion on this issue is presented in Appendix \ref{subsec:alpha}. (ii) Our empirical results show that CBRW performs very stably with a wide range of values for $\alpha$, i.e., $\alpha \in [0.85,1)$. Therefore, employing advanced procedures to learn the parameter may not have an obvious benefit in terms of detection performance.

CBRW requires $O(ND^2)$ to obtain the value couplings information in Steps 1-8. The generation of $\mathbf{W}^{b}$ requires at most $O(|\mathcal{E}|)$ in Step 9. The random walks in Steps 11-13 are linear to the maximum iteration step and the number of edges in the value graph, resulting in $O(|\mathcal{E}|I_{max})$. Therefore, the overall time complexity is $O(ND^2+|\mathcal{E}|I_{max})$. $I_{max}$ is a constant. $|\mathcal{E}|$ is approximately equal to $|\mathcal{V}|^{2}$ when $\mathsf{G}$ is a highly dense graph, leading to $O(ND^2+|\mathcal{V}|^{2})$. We have $|\mathcal{E}| \approx D^{2}$ when $\mathsf{G}$ is a highly sparse graph. This case results in $O(ND^2)$.

\section{CUOT's Instance II: SDRW for Noise-tolerant Estimation of Non-IID Value Outlierness}\label{sec:sdrw}

This section introduces another instantiation of CUOT, called multiple-granularity Subgraph Densities augmented Random Walks (SDRW). SDRW is motivated by CBRW, but it is a significantly enhanced instance compared to CBRW. It not only retains the CBRW's ability to handle non-IID data, but it also considers noisy features to cope with more challenging data. 

Specifically, SDRW uses the same mode-normalized initial outlierness as CBRW, but it replaces the conditional probability-based outlierness influence with the one based on the concept of lift \citep{brin1997lift}. Although it seems to be a minor change, the resulting improvement is significant. It effectively transforms the value graph into an undirected graph, and subsequently we can derive a parameter-free closed-form solution for learning value outlierness. The closed-form solution also reduces the computational cost of SDRW compared to CBRW. Additionally, the lift-based outlierness influence enables us to better distinguish coupled outlying values from strongly coupled normal values. To enhance its tolerance to noisy features, SDRW uses a multiple-granularity dense subgraph mining to learn a more reliable bias into the biased random walks. A summary of the differences between CBRW and SDRW is provided in Table \ref{tab:cbrw_vs_dsrw}.

\begin{table}[htbp]
  \centering
  \ttabbox{\caption{Conceptual Comparison of CBRW and SDRW}\label{tab:cbrw_vs_dsrw}}
  {
  \scalebox{0.95}{
    \begin{tabular}{|c|c|c|}
    \hline
       & \textbf{CBRW} & \textbf{SDRW} \\
    \hline
    Intra-feature Outlier Factor & \multicolumn{2}{c|}{Mode-based Normalization} \\\hline
    Inter-feature Outlier Factor & Conditional Probability & Lift \\
    Value Graph & Directed & Undirected \\
    Value Outlierness Learning & BRWs & Noise-tolerant BRWs \\
    Closed-form Solution & No & Yes \\
    Parameters & $\alpha$ & None \\
    Time Complexity & $O(ND^2)+O(|\mathcal{E}|I_{\mathit{max}})$ & $O(ND^2)+O(|\mathcal{E}|)$\\\hline
    \end{tabular}%
    }
    }
\end{table}

To differentiate between the specifications for CBRW and SDRW, a superscript `$\prime$' is added to the notations in SDRW if the same notation is used in CBRW.

\subsection{Lift-based Outlierness Influence}\label{subsec:pmi}

Lift, which is a non-logarithmic pointwise mutual information, is a widely-used measure to define the correlation between two values. Lift replaces the conditional probabilities in the outlierness influence vector as follows. 

\begin{definition}[Lift-based Outlierness Influence Vector]\label{def:pmi}
The lift-based outlierness influence vector of a value $v$ due to all the other values is defined as
\begin{equation}\label{eqn:pmi}
\begin{split}
\mathbf{q}_{v}^{\prime} &= [\eta^{\prime}(u_1,v), \cdots, \eta^{\prime}(u_{|\mathcal{V}|},v)]^\intercal \\
& =[\frac{\mathit{freq}(u_1,v)}{\mathit{freq}(u_1)\mathit{freq}(v)},\cdots, \frac{\mathit{freq}(u_{|\mathcal{V}|},v)}{\mathit{freq}(u_{|\mathcal{V}|})\mathit{freq}(v)}]^\intercal,
\end{split}
\end{equation}
\end{definition}

The resulting inter-feature outlier factor has the following two key properties. 
\begin{enumerate}
    \item $\eta^{\prime}(u,v)\geq 0$.
    \item $\eta^{\prime}(u,v)=\eta^{\prime}(v,u)$.
\end{enumerate}
The values $u$ and $v$ are positively correlated if $\eta^{\prime}(u,v) > 1$ and are independent or negatively correlated otherwise. In our approach, larger lift is more desired. This is because we focus on capturing the homophily couplings between values, i.e., how their outlierness is \textit{positively} influenced on each other. More importantly, although they may exist interdependence between either frequent values or rare values, the lift of highly correlated rare values can be far larger than that of correlated frequent values because of the difference in the product of the marginals. For example, in a data size of 100, we have two rare values that have five co-occurrences, with each having 10 individual appearances, and we also have two frequent values that have 30 co-occurrences, with each having 60 individual appearance; in this case, although their conditional probabilities are equal, the lift of the rare values is four, while the lift of the frequent values is less than one. This strengthens the propagation of outlierness between outlying values in the subsequent random walk process. Therefore, compared to conditional probabilities, lift captures more accurate homophily couplings. Particularly, the difference obtained by $\eta^{\prime}(u,v)-\eta^{\prime}(w,z)$ is much larger than $\eta(u,v)-\eta(w,z)$ in CBRW, when both values $u$ and $v$ are outlying values and at least one of the values $w$ and $z$ is not an outlying value. 

\subsection{Refining the Value Graph with Subgraph Densities}\label{subsec:undirectedgraph}

SDRW then constructs an attributed value graph $\mathsf{G}^{\prime}=<\mathcal{V},\mathcal{E},\omega_{\delta^{\prime},\eta^{\prime}}>$ in the same way as CBRW built the graph $\mathsf{G}$. Since SDRW and CBRW use the same mode-based initial outlierness, we have $\delta^{\prime}=\delta$. Here the key difference between SDRW and CBRW is that $\mathsf{G}^{\prime}$ is an undirected graph as $\eta^{\prime}(u,v)=\eta^{\prime}(v,u)$, while $\mathsf{G}$ is a directed graph.

Let $\mathbf{A}^{\prime}$ be the adjacency matrix of $\mathsf{G}^{\prime}$ with its entry $\mathbf{A}^{\prime}(u,v)=\eta^{\prime}(u,v)$. According to Lemma \ref{thm:lem}, the attributed value graph can be equivalently transformed to a plain graph with an adjacency matrix $\mathbf{C}$, in which its entry is
\begin{equation}\label{eqn:valuegraph2}
    \mathbf{C}(u,v)=\delta^{\prime}(u)\eta^{\prime}(u,v)\delta^{\prime}(v), \; \forall u,v \in \mathcal{V}.
\end{equation}

One major problem with $\mathbf{C}(u,v)$ (or $\mathbf{B}(u,v)$ in CBRW) is that $\delta^{\prime}$ (or $\delta$) may mislead the subsequent value outlierness learning when $u$ or $v$ is a noisy value. This is because noisy values may have a lower frequency than outlying values. Consequently, noisy values have larger intra-feature outlierness $\delta$ than outlying values. When there are many such noisy values, this can downgrade the quality of the outlierness learning. One simple solution is to remove the term $\delta^{\prime}$, but that would also remove important intra-feature value coupling information, making the solution less effective when outliers demonstrate obvious outlying behaviors in individual features (See the empirical results in Section \ref{subsec:component}).

Instead, SDRW learns a noise-tolerant term to replace $\delta^{\prime}$ by aggregating the density of a collection of multiple-granularity dense subgraphs associated with a specific value. Our intuition is as follows. Due to the homophily couplings between outlying values, the neighbors of outlying values in the value graph are much more likely to be outlying values than noisy values. Since the edge weights convey the value outlierness, the outlying values are located in denser subgraphs than noisy values. We therefore define the following subgraph density-based outlier factor:

\begin{definition}[Subgraph Density-based Outlier Factor]
Let $\mathsf{H}_{k}$ be the densest $k$ subgraph, i.e., the densest subgraph of exactly $k$ nodes in graph $\mathsf{G}^{\prime}$, and let $\mathcal{G}=\{\mathsf{H}_{2},\mathsf{H}_{3}, \cdots, \mathsf{H}_{|\mathcal{V}|-1}\}$ be the complete set of the densest $k$ subgraphs. The subgraph density-based outlier factor of a value $v$ is defined as the average density of all the densest $k$ subgraphs that contain $v$, i.e.,
\begin{equation}\label{eqn:ad}
    \gamma(v) = \frac{1}{|\mathcal{G}_{v}|} \sum_{\mathsf{H}^{v} \in \mathcal{G}_{v}} \mathit{den}(\mathsf{H}^{v}),
\end{equation}
where $\mathcal{G}_{v}$ is the set of the densest $k$ subgraphs that contain $v$ and the subgraph density is computed by 
\begin{equation}\label{eqn:subgraphdensity}
    \mathit{den}(\mathsf{H}^{v})=\frac{\sum_{u \in \mathcal{V}_{v}}\sum_{v \in \mathcal{V}_{v}}\delta^{\prime}(u)\eta^{\prime}(u,v)\delta^{\prime}(v)}{2|\mathcal{V}_{v}|},
\end{equation}
where $\mathcal{V}_{v}$ denotes the set of nodes contained in $\mathsf{H}^{v}$.
\end{definition}

$\gamma$ is built on the homophily couplings and is designed to capture the possible cascade relations of outlying values. This helps increase the outlierness of the outlying values that are surrounded by only a few direct outlying nodes, but their outlying neighbor nodes (or the neighbors of them, and so on) are coupled with many outlying values. There exist $2^{|\mathcal{V}|}-1$ subgraphs for $\mathsf{G}^{\prime}$ in total, and we have $|\mathcal{V}| \choose k$ subgraphs at a specific resolution, i.e., subgraphs with $k$ nodes only. Using all these subgraphs is computationally intractable. We therefore use the single densest $k$ subgraph as it contains the most important outlierness information at each resolution. However, finding single densest $k$ subgraph has been proven to be an NP-hard problem \citep{khuller2009subgraph}. We resort to a greedy method in Algorithm \ref{alg:densesubgraph} to produce a set of dense subgraphs, $\mathcal{G}^{+}$, to approximate $\mathcal{G}$. Note that $\mathsf{H}_{1}$ and $\mathsf{H}_{|\mathcal{V}|}$ are excluded from $\mathcal{G}$, since they provide no distinguishing information for computing the $\gamma$ of each value. Also, we set $\gamma(v)=0$ if $|\mathcal{G}_v| = 0$.

\renewcommand{\algorithmicrequire}{\textbf{Input:}}
\renewcommand{\algorithmicensure}{\textbf{Output:}}
\begin{algorithm}
\caption{Dense Subgraph Discovery}
\label{alg:densesubgraph}
\begin{algorithmic}[1]
\REQUIRE $\mathcal{X}$ - data objects
\ENSURE $\mathcal{G}^{+}$ - a set of dense subgraphs
\STATE Generate $\mathbf{C}$ using Equation (\ref{eqn:valuegraph2})
\STATE Compute the weighted degree of each node $v\in \mathcal{V}$
\STATE Initialize $\mathcal{G}^{+}$ as an empty set
\REPEAT
\STATE Let $v\in \mathcal{V}$ be the node having the minimal weighted degree in $\mathsf{G}^{\prime}$
\STATE $\mathsf{G}^{\prime} \leftarrow \mathsf{G}^{\prime} \setminus v $
\STATE $\mathcal{G}^{+} \leftarrow  \mathcal{G}^{+} \cup \{\mathsf{G}^{\prime}\} $
\UNTIL Only one node left in $\mathsf{G}^{\prime}$
\RETURN $\mathcal{G}^{+}$
\end{algorithmic}
\end{algorithm}

 Although Algorithm \ref{alg:densesubgraph} cannot find the exact set of the densest $k$ subgraphs, using $\mathcal{G}^{+}$ to compute $\gamma$ guarantees that values with more outlying neighbors (i.e., with a  larger weighted degree) obtain a larger $\gamma$. Moreover, it has linear time complexity w.r.t. $|\mathcal{V}|$, which enables SDRW to compute $\gamma$ very efficiently. Other desirable properties of this algorithm include: (i) the densest subgraph in the subgraphs it produces has a $\frac{1}{2}$-approximation to the optimal densest subgraph without size constraints; and (ii) it is able to produce the densest subgraph with at least $k$ nodes (a relaxed problem to the problem of finding the densest $k$ subgraph) having $\frac{1}{3}$-approximation to the optimal solution \citep{andersen2009subgraph,khuller2009subgraph}. These two properties make the use of $\mathcal{G}^{+}$ obtain a good approximation to the exact $\gamma$.
 
 We further replace $\delta$ with $\gamma$ in Equation (\ref{eqn:valuegraph2}) and obtain
 \begin{equation}\label{eqn:valuegraph3}
    \mathbf{B}^{\prime}(u,v)=\gamma(u)\eta^{\prime}(u,v)\gamma(v),\; \forall u,v \in \mathcal{V},
\end{equation}

\noindent where $\gamma$ is an enhanced $\delta$ for better tolerance to noisy values. Unlike $\eta$ that captures low-order pairwise value couplings, $\gamma^{\prime}$ captures high-order arbitrary-length homophily couplings.

\subsection{Noise-tolerant Biased Random Walks for Learning Value Outlierness}\label{subsec:brw2}

SDRW then performs random walks with the adjacency matrix $\mathbf{B}^{\prime}$. The transition matrix is as follows:

\begin{equation}
    \mathbf{T}^{\prime}(u,v) = \frac{\mathbf{B}^{\prime}(u,v)}{\sum_{v \in \mathcal{V}}\mathbf{B}^{\prime}(u,v)}=\frac{\gamma(u)\mathbf{A}^{\prime}(u,v)\gamma(v)}{\sum_{v \in \mathcal{V}}\gamma(u)\mathbf{A}^{\prime}(u,v)\gamma(v)}.
\end{equation}

This is equivalent to biased random walks with the following transition matrix
\begin{equation}\label{eqn:brw2}
    \mathbf{W}^{b\prime}(u,v) =\frac{\gamma(v)\mathbf{A}^{\prime}(u,v)}{\sum_{v \in \mathcal{V}}\gamma(v)\mathbf{A}^{\prime}(u,v)},
\end{equation}

\noindent in which the terms $\gamma$ and $\mathbf{A}^{\prime}(u,v)$ replace $\delta$ and $\mathbf{A}(u,v)$ in Equation (\ref{eqn:brws1}), respectively. $\gamma$ improves the tolerance to noisy values over $\delta$, while $\eta^{\prime}(u,v)$ improves the homophily coupling modeling over $\eta(u,v)$.

We can accordingly define the value outlierness as follows.

\begin{definition}[SDRW-based Value Outlierness]
The outlierness of node $v$ is defined as  
\begin{equation}\label{eqn:score2}
\phi^{\prime}(v)=\boldsymbol\pi^{*\prime}(v),
\end{equation}
where $\boldsymbol\pi^{*\prime}=\mathbf{W}^{b\prime}\boldsymbol\pi^{*\prime}$ denotes the stationary probabilities of biased random walks with the transition matrix $\mathbf{W}^{b\prime}$.
\end{definition}

This value outlierness estimation function $\phi^{\prime}(v)$ can be computed using its closed-form as shown below.

\begin{theorem}[Closed-form Outlierness Estimation]\label{thm:closedform}
Let $\mathsf{G}^{\prime}$ be the graph with its adjacency matrix $\mathbf{B}^{\prime}$ such that $\mathbf{B}^{\prime}(u,v)=\gamma(u)\eta^{\prime}(u,v)\gamma(v),\; \forall u,v \in \mathcal{V}$. Then we have 
\begin{equation}\label{eqn:closedform}
    \boldsymbol\pi^{*\prime}(v) =\frac{d^{\prime}(v)}{\mathit{vol}(\mathsf{G}^{\prime})}, \;\; \forall v \in \mathcal{V},
\end{equation} 
where $d^{\prime}(v) = \sum_{u \in \mathcal{V}}\mathbf{B}^{\prime}(u,v) = \sum_{u \in \mathcal{V}}\gamma(v)\eta^{\prime}(u,v)\gamma(u)$ denotes the weighted degree of node $v$ and $\mathit{vol}(\mathsf{G}^{\prime})=\sum_{v\in \mathcal{V}}d^{\prime}(v)$ is the volume of $\mathsf{G}^{\prime}$.
\end{theorem}

We can derive the following closed-form of $\phi^{\prime}(v)$ after its rollout.

\begin{equation}\label{eqn:closedform_sdrw}
  \phi^{\prime}(v)= \frac{\sum_{u \in \mathcal{V}}\gamma(v)\eta^{\prime}(u,v)\gamma(u)}{\sum_{v \in \mathcal{V}}\sum_{u \in \mathcal{V}}\gamma(v)\eta^{\prime}(u,v)\gamma(u)},
\end{equation}
\noindent where the nominator is the weighted degree of node $v$ and the denominator is the volume (i.e., total weighted degree) of the graph $\mathsf{G}^{\prime}$.

Similar to $\phi$, we have $0< \phi^{\prime} (v) <1 $ and $\sum_{v \in \mathcal{V}}\phi^{\prime}(v) =1$. As indicated in Equation (\ref{eqn:closedform}), $\phi^{\prime}(v)$ is positively proportional to the weighted degree of node $v$. Therefore, outlying values are expected to obtain much larger outlierness than normal values. Also, the number of outlying values is often very small, in an ideal case, the outlierness of all values follows a long-tail distribution with the outlying values having dominant outlierness.

Note that the above form is not necessarily the unique convergence form of the random walks. It becomes a unique convergence if, and only if, $\mathsf{G}^{\prime}$ is irreducible and aperiodic \citep{meyer2000matrix}. However, this closed-form well captures the homophily outlying couplings. Hence, it is used to compute the outlierness of all values.

\subsection{The Algorithm of SDRW and Its Time Complexity}

Algorithm \ref{alg:dsrw} presents the procedures of SDRW. Step 1 computes the dense subgraph-based outlier factor $\gamma$, followed by the generation of $\mathbf{B}^{\prime}$ in Step 2. Steps 3-5 further estimate the outlierness of each value using the closed-form of the stationary probability distribution of biased random walks.

\renewcommand{\algorithmicrequire}{\textbf{Input:}}
\renewcommand{\algorithmicensure}{\textbf{Output:}}
\begin{algorithm}
\caption{Subgraph Density-augmented Random Walks}
\label{alg:dsrw}
\begin{algorithmic}[1]
\REQUIRE $\mathcal{X}$ - data objects
\ENSURE $\boldsymbol\pi^{*\prime}$ - the stationary probability distribution
\STATE Compute $\gamma$ using Equation (\ref{eqn:ad}) with $\mathcal{G}^{+}$ returned by Algorithm \ref{alg:densesubgraph}
\STATE Obtain $\mathbf{B}^{\prime}$ using Equation (\ref{eqn:valuegraph3})
\FOR{ $v \in \mathcal{V}$}
\STATE $\boldsymbol\pi^{*\prime}(v) \leftarrow \frac{\sum_{u \in \mathcal{V}}\gamma(v)\eta^{\prime}(u,v)\gamma(u)}{\sum_{v \in \mathcal{V}}\sum_{u \in \mathcal{V}}\gamma(v)\eta^{\prime}(u,v)\gamma(u)}$
\ENDFOR
\RETURN $\boldsymbol\pi^{*\prime}$
\end{algorithmic}
\end{algorithm}
Both the iterative removal of nodes in Algorithm \ref{alg:densesubgraph} and the calculation of $\gamma$ in Step 2 have  a time complexity of $O(|\mathcal{E}|)$. 
Similar to CBRW, SDRW requires $O(ND^2)$ to obtain $\mathbf{B}^{\prime}$ using Equation (\ref{eqn:valuegraph3}). The subsequent estimation of value outlierness has a time complexity of $O(|\mathcal{E}|)$. Therefore, SDRW has an overall time complexity of $O(ND^2+|\mathcal{E}|)$.

\section{Theoretical Comparisons of IID and Non-IID Outlier Factors}\label{sec:thm}

Modeling the intrinsic non-IID outlying behaviors enables CBRW and SDRW to build detection models that are more faithful to real-world data, since the IID assumption rarely holds in practice. This section analyzes how CBRW and SDRW model such behavior characteristics. As shown below, properly capturing such characteristics also lead to a better ability to tackle other complex problems than IID methods, such as handling data with low outlier separability or noisy features.


\subsection{On Heterogeneity}

We first analyze how the intra-feature outlier factor $\delta$ (or $\delta^{\prime}$) in CBRW and SDRW capture the heterogeneity between values, and then compare that to IID methods in handling data with low outlier separability.

\subsubsection{Capturing Heterogeneous Categorical Distributions}
Traditional methods focus on the identification of outlying/normal patterns. They ignore the heterogeneous categorical distributions taken by different feature subspaces and assign the same outlierness to the values/patterns of the same frequency. Two typical ways the IID methods use to compute the outlierness are: $1-\mathit{freq}(\cdot)$ or $\frac{1}{\mathit{freq}(\cdot)}$. However, the categorical distributions (e.g., the location of the distribution) taken by different features/subspaces may differ significantly from each other, and as a result, the frequency drawn from one distribution has a different meaning for another. The following proposition is obtained based on this observation.

\begin{proposition}[Semantic Heterogeneity]
Given any two modes $m_i$ and $m_j$ of features $\mathsf{F}_i$ and $\mathsf{F}_j$ respectively, and their frequencies $\mathit{freq}(m_i)$ and $\mathit{freq}(m_j)$. If $\mathit{freq}(m_i) \neq \mathit{freq}(m_j)$, then $\mathit{freq}(u)$ is not directly comparable to $\mathit{freq}(v)$ in terms of their outlierness, $\forall u\in dom(\mathsf{F}_i),v\in dom(\mathsf{F}_j)$.

\end{proposition}
In the proposition, $\mathit{freq}(m_i) \neq \mathit{freq}(m_j)$ indicates that $\mathsf{F}_i$ and $\mathsf{F}_j$ have different rarity evaluation benchmarks. Thus, $\mathit{freq}(u)$ and $\mathit{freq}(v)$ are not directly comparable. The $\delta$ function in Equation (\ref{eqn:mad}) takes the location parameter of categorical distribution into account and works as a mode-oriented normalization. This way facilitates a more meaningful outlierness comparison between values from different features, compared to the IID methods.


In addition to the normalization, the next section shows that the $\delta$ function also has better ability than the IID methods in handling data with low outlier separability.

\subsubsection{Handling Data with Low Outlier Separability}\label{subsubsec:sep}

Data with low outlier separability is generally referred to as data without strongly relevant features. To handle such data, one potential solution is to enlarge the gaps between the outlierness of outlying values and normal values. Let $\delta$ denotes the intra-feature value outlierness estimation function used in CBRW and SDRW, and $\phi_{\mathit{freq}}$ denotes a baseline value outlierness estimation function used in IID methods, which is defined as $\phi_{\mathit{freq}}(\cdot) = 1 - \mathit{freq}(\cdot)$; let $u^{\prime}$ and $v^{\prime}$ be outlying and normal values, respectively, then the following theorem helps justify that CBRW and SDRW have better capability than IID methods in handling low outlier-separable data .

\begin{theorem}[Outlierness Contrast]\label{thm:gapdelta}
Let $m$ be the mode of the feature $\mathsf{F}$, and an outlying value $u^{\prime}\in \mathit{dom}(\mathsf{F})$ and a normal value $v^{\prime} \in \mathit{dom}(\mathsf{F})$. Then 
\begin{equation}\label{eqn:oc}
    \delta(u^{\prime}) - \delta(v^{\prime}) > \phi_{\mathit{freq}}(u^{\prime}) - \phi_{\mathit{freq}}(v^{\prime}).
\end{equation}
\end{theorem}

After some transformations in Equation (\ref{eqn:oc}), We can obtain that the difference between the left and the right equations is $\frac{(N-\mathit{supp}(m))\beta}{\mathit{supp}(m)N}$ where $\beta$ is the difference between the supports of the outlying and normal values. This indicates that $\delta$ is able to obtain much higher outlierness contrast between outlying values and normal values when: (i) the support of the feature mode $m$ is smaller; and/or (ii) the difference between the supports of the outlying and normal values is larger.

\subsection{On Homophily Couplings}

We first analyze how the combination of the outlierness influence factor $\eta$ (or $\eta^{\prime}$) and the random walks captures the homophily couplings between outlying values. We then discuss their ability to handle noisy features.

\subsubsection{Modeling Homophily Outlying Behaviors}
Random walks are one of the most popular and efficient methods for modeling homophily couplings \citep{koutra2011homophily}. Basically, as CBRW and SDRW conduct random walks on the value graph, in which large edge weights indicate large outlierness of the corresponding nodes of the edge, the final outlierness of a node is determined by the outlierness associated with its direct neighbor nodes, and the outlierness of these neighbor nodes is governed by the neighbors of these nodes, and so on. This process models an iterative effect of outlierness propagation.

\begin{definition}[Direct Neighbor]
The direct neighbors of a node $u$ are defined as 
\begin{equation}
    \mathcal{N}(u)=\{\forall v\in \mathcal{V} | \text{dist}(v,u)=1\},
\end{equation}
where $\text{dist}(v,u)$ returns the shortest path from node $v$ to node $u$.
\end{definition}

\begin{proposition}[Homophily Coupling Modeling]\label{thm:conedform}
Let $\mathcal{N}(v)$ be the direct neighbors of a node $v \in \mathcal{V}$. Then the outlierness of value $v$ is linearly proportional to the outlierness of its direct neighbors and its coupling strength with these neighbors, i.e., in CBRW, we have
\begin{equation}\label{eqn:cbrw_homophily}
    \phi(v)\propto {\sum_{u \in \mathcal{N}(v)}\phi(u)\delta(u)\eta(u,v)\delta(v)},
\end{equation} 
while in SDRW, we have
\begin{equation}\label{eqn:sdrw_homophily}
    \phi^{\prime}(v)\propto {\sum_{u \in \mathcal{N}(v)}\phi^{\prime}(u)\gamma(u)\eta^{\prime}(u,v)\gamma(v)}.
\end{equation} 
\end{proposition}

\begin{proof}
Following Lemma \ref{thm:lem}, the random walks in CBRW can be represented by $\boldsymbol\pi^{t+1}(v) = \sum_{u \in \mathcal{N}(v)}\boldsymbol\pi^{t}(u)\frac{\delta(u)\mathbf{A}_{u,v}\delta(v)}{\sum_{v\in V}\delta(u)\mathbf{A}_{u,v}\delta(v)}$. Since the denominator is a constant for all the neighbors of node $u$ under the same context, we can omit it and obtain $\boldsymbol\pi(v)\propto {\sum_{u \in \mathcal{N}(v)}\boldsymbol\pi(u)\delta(u)\eta(u,v)\delta(v)}$. Since $\phi(v)=\boldsymbol\pi(v)$, we achieve Equation (\ref{eqn:cbrw_homophily}). We can also obtain Equation (\ref{eqn:sdrw_homophily}) through the same process. 
\end{proof}

Proposition \ref{thm:conedform} states that the outlierness of a value is mainly determined by its coupling strength and the outlierness of its direct neighbors, in addition to its intra-feature outlierness. That is, a value has large outlierness if it is centered around outlying values.  This captures exactly the homophily phenomenon of outlying behaviors - the tendency of a set of outlying behaviors to join together. Compared to IID methods that treat the outlierness scoring of outlying behaviors independently, i.e., $\phi(v)$ is independent of $\phi(u)$, our models can achieve a more effective outlierness estimation on data with homophily outlying behaviors.

\subsubsection{Handling Noisy Features}\label{subsubsec:noise}

To handle data with noisy features, a fundamental requirement for value outlierness-based outlier detectors is to assign larger outlierness to outlying values (i.e., infrequent values contained by outliers) than noisy values (i.e., infrequent values contained by normal objects). The outlierness scoring methods in IID methods, such as $1-\mathit{freq}(\cdot)$ or $\frac{1}{\mathit{freq}(\cdot)}$, assign similar outlierness to outlying and noisy values, since these two types of values often have a similarly low frequency. These methods therefore fail to distinguish outlying values/patterns from noisy ones and become ineffective in handling data with many noisy features. By contrast, modeling homophily couplings enables CBRW and SDRW to contrast the outlierness of outlying values from the noisy ones. We demonstrate this intuition with a straightforward example, as follows.

Let $u^{\prime}$ and $w^{\prime}$ be outlying and noisy values of the same frequency, respectively. Assume both $u^{\prime}$ and $w^{\prime}$ share the same direct neighbor set $\mathcal{N}$, in which $\mathcal{N}^{o} \subset \mathcal{N}$ is a set of outlying values and $\mathcal{N}\setminus \mathcal{N}^{o}$ is the set of normal values. Then, according to Proposition \ref{thm:conedform},  we have 
\begin{equation}
    \phi(u^{\prime})\propto {\sum_{u \in \mathcal{N}^{o}}\phi(u)\delta(u)\eta(u,u^{\prime})\delta(u^{\prime})} + {\sum_{w \in \mathcal{N}\setminus\mathcal{N}^{o}}\phi(w)\delta(w)\eta(w,u^{\prime})\delta(u^{\prime})}.
\end{equation}

We can obtain a similar proportional form for $\phi(w^{\prime})$ by replacing $u^{\prime}$ with $w^{\prime}$ in the above equation. Since $u^{\prime}$ and $w^{\prime}$ share the same direct neighbor set, the outcomes of the terms $\phi$ and $\delta$ on the right-hand side are the same as that in $\phi(w^{\prime})$. We can therefore simplify them as $\phi(u^{\prime})\propto {\sum_{u \in \mathcal{N}^{o}}\eta(u,u^{\prime})} + {\sum_{w \in \mathcal{N}\setminus\mathcal{N}^{o}}\eta(w,u^{\prime})}$ and $\phi^{\prime}(w^{\prime})\propto {\sum_{u \in \mathcal{N}^{o}}\eta(u,w^{\prime})} + {\sum_{w \in \mathcal{N}\setminus\mathcal{N}^{o}}\eta(w,w^{\prime})}$. When there exist homophily couplings among the outlying values while noisy values are randomly coupled with the outlying values, we have $\sum_{u \in \mathcal{N}^{o}}\eta(u,u^{\prime}) > \sum_{u \in \mathcal{N}^{o}}\eta(u,w^{\prime})$. If $u^{\prime}$ and $w^{\prime}$ have similar co-occurrence patterns with the set of normal values $\mathcal{N}\setminus \mathcal{N}^{o}$, then $\sum_{w \in \mathcal{N}\setminus\mathcal{N}^{o}}\eta(w,u^{\prime}) \approx \sum_{w \in \mathcal{N}\setminus\mathcal{N}^{o}}\eta(w,w^{\prime})$. Hence, $\phi(u^{\prime}) > \phi(w^{\prime})$ holds in CBRW.

For SDRW, we can obtain the following equivalence due to Theorem \ref{thm:closedform}

\begin{equation}\label{eqn:sdrw_noise}
    \phi^{\prime}(u^{\prime}) = {\sum_{u \in \mathcal{N}^{o}}\delta(u)\eta^{\prime}(u,u^{\prime})\delta(u^{\prime})} + {\sum_{w \in \mathcal{N}\setminus\mathcal{N}^{o}}\delta(w)\eta^{\prime}(w,u^{\prime})\delta(u^{\prime})},
\end{equation}
where SDRW retains the $\delta$ function but changes the $\eta$ function from conditional probabilities to lift. Since the normal value $w$ generally has a high frequency, $\eta^{\prime}(w,u^{\prime})$ is marginalized by $\mathit{freq}(w)$ and $\delta(w)$ is very small. Hence, the second term in Equation (\ref{eqn:sdrw_noise}) can be generally left out. Similar to the cases in CBRW, we can omit $\delta$. We therefore obtain $\phi^{\prime}(u^{\prime}) = {\sum_{u \in \mathcal{N}^{o}}\eta^{\prime}(u,u^{\prime})}$ and $\phi^{\prime}(w^{\prime}) = {\sum_{u \in \mathcal{N}^{o}}\eta^{\prime}(u,w^{\prime})}$. In such cases, we achieve $\phi^{\prime}(u^{\prime}) > \phi^{\prime}(w^{\prime})$ in SDRW even when not using $\gamma$. Compared to CBRW that obtains $\phi(u^{\prime}) > \phi(w^{\prime})$ under certain conditions, SDRW can achieve the same result without such constraints. This demonstrates the benefit of replacing $\eta$ with $\eta^{\prime}$.

However, many real-world data sets may demonstrate much tougher cases than the example as above. For example, the frequency of noisy values can be much lower than that of outlying values, and the values $u^{\prime}$ and $w^{\prime}$ have different direct neighbor sets. In such cases, if $\delta$ is still used in SDRW, we have $\delta(w^{\prime})>\delta(u^{\prime})$, and consequently $\phi^{\prime}(w^{\prime})$ can obtain larger outlierness from its normal value neighbors, compared to $\phi^{\prime}(u^{\prime})$. If $w^{\prime}$ randomly occurs with some outlying values while $u^{\prime}$ has limited outlying values in its direct neighbors, $\phi^{\prime}(w^{\prime})$ can also obtain larger outlierness from its outlying value neighbors, leading to the undesired result $\phi(u^{\prime}) < \phi(w^{\prime})$. To tackle this problem, $\gamma$ is introduced into SDRW to consider both direct neighbors and indirect neighbors. It is assured that $\gamma(u^{\prime})$ is much larger than $\gamma(w^{\prime})$ when the outlying values bond together, e.g., in the form of cascade. This enables SDRW to obtain much larger outlierness from the direct and indirect outlying value neighbors for the outlying value $u^{\prime}$, compared to the noisy value $w^{\prime}$. As a result, replacing $\delta$ with $\gamma$ in Equation (\ref{eqn:sdrw_noise}) largely increases the ability of SDRW to assign larger outlierness to outlying values than noisy ones.

These data characterizations will be used to guide the quantization of the data complexity of real-world datasets in Section \ref{subsubsec:complexity}.

\section{Experimental Design}\label{sec:exp}

This section first provides two applications of the value outlierness obtained by CBRW and SDRW, and then gives the performance evaluation approach. It is followed by a subsection on creating benchmark data sets.

\subsection{Applications of CBRW and SDRW}

\subsubsection{Outlying Feature Weighting and Selection}

In outlier detection, relevant features are the features where the outliers demonstrate outlying behaviors and are distinguishable from normal objects. Thus, the relevance of a feature can be measured by consolidating the outlierness of each value of the feature. Below we use $\phi$, the value outlierness estimation function used in CBRW, to demonstrate this application. The same procedure is applied to compute feature outlierness in SDRW by replacing $\phi$ with $\phi^{\prime}$.

\begin{definition}[Feature Relevance]\label{def:frel}
The relevance of a feature $\mathsf{F}$ is defined as
\begin{equation}\label{eqn:frel}
\mathit{rel}(\mathsf{F})= 1-\prod_{v \in \mathit{dom}(\mathsf{F})}[1-\phi(v)].
\end{equation}
where $\phi(v) \in (0,1)$ according to Equation (\ref{eqn:score}).
\end{definition}

Since $\phi(v)$ is taken from the stationary probability, it can be interpreted as the probability of the value $v$ being outlying. Therefore, we define $\mathit{rel}(\mathsf{F})$ in Equation (\ref{eqn:frel}) to estimate the outlying likelihood of the feature $\mathsf{F}$. A large $\mathit{rel}(\cdot)$ indicates high relevance of the feature to outlier detection. The top-ranked features are the most relevant features, while the bottom-ranked features are noisy or irrelevant. In addition to being used as a feature filter, these relevance weights can also be embedded in the outlier scoring function of an outlier detector as a feature weighting. One such example is shown in the next subsection. 

Note that $\mathit{rel}(\mathsf{F})$ may bias towards features with more values. Alternatively, considering the feature relevance using selected large-outlierness values may help reduce this bias, but we found empirically that it is very difficult to determine this parameter in cases where the number of values in different features differs largely. One other option is to only use the value with largest outlierness in each feature, but it results in severe loss of feature relevance information. $\mathit{rel}(\mathsf{F})$ is used due to its good interpretability and empirical effectiveness in real-world data.

As shown in Steps 1-3 in Algorithm \ref{alg:fs}, our feature selection method, denoted as $\text{CBRW}_{\text{fs}}$/$\text{SDRW}_{\text{fs}}$, computes the weight of each feature using Equation (\ref{eqn:frel}). The top-ranked features for each data set are selected in Step 4. Outlier detectors can then work on the newly obtained data sets with the selected features.

\begin{algorithm}
\caption{Feature Selection}
\label{alg:fs}
\begin{algorithmic}[1]
\REQUIRE $\mathcal{F}$ - feature set,  $\phi$ - value outlierness estimation function, $\theta$ - a decision threshold (i.e., the number of features to be selected or a relevance threshold)
\ENSURE $Subset_\mathcal{F}$ - a subset of features in $\mathcal{F}$
\FOR{ $i = 1$ to $D$}
\STATE $\mathit{rel}(\mathsf{F}_i) \leftarrow 1-\prod_{v \in \mathit{dom}(\mathsf{F}_i)}[1-\phi(v)]$ 
\ENDFOR
\STATE $Subset_\mathcal{F} \leftarrow filter(\mathcal{F})$: Select the features that meet the threshold $\theta$
\RETURN $Subset_\mathcal{F}$
\end{algorithmic}
\end{algorithm}

\subsubsection{Direct Outlier Detection}
The value outlierness can measure the outlierness of data objects by consolidating the outlierness of values contained by the objects. Similarly, the value outlierness estimation function $\phi$ in CBRW is used as an exemplar below. By replacing $\phi$ with $\phi^{\prime}$, we can obtain the SDRW-based method.

\begin{definition}[Object Outlierness]\label{def:od}
The outlierness of an object $\mathbf{x}_i \in \mathcal{X}$ is defined as
\begin{equation}\label{eqn:od}
object\_score(\mathbf{x}_i) = 1 - \prod_{j=1}^{D}[1-\phi(x_{ij})]^{\tau(\mathsf{F}_j)},
\end{equation}
where $\phi(x_{ij}) \in (0,1)$ and $\tau(\mathsf{F}_j)=\frac{\mathit{rel}(\mathsf{F}_j)}{\sum_{j=1}^{D}\mathit{rel}(\mathsf{F}_j)}$ is a feature weighting component.
\end{definition}

$object\_score(\cdot)$ is used to evaluate the outlying likelihood of an object, with a relevance weighting factor to highlight the importance of highly relevant features. As shown in Steps 1-3 of Algorithm \ref{alg:od}, our CBRW/SDRW-based outlier detection method (denoted as $\text{CBRW}_{\text{od}}$/$\text{SDRW}_{\text{od}}$) employs Equation (\ref{eqn:od}) to compute the outlying likelihood of each data object. Objects are then sorted by their outlierness. Outliers are data objects having large outlier scores.

\begin{algorithm}
\caption{Outlier Detection}
\label{alg:od}
\begin{algorithmic}[1]
\REQUIRE $\mathcal{X}$ - data objects,  $\phi$ - value outlierness estimation function, $\boldsymbol\tau$ - feature weights
\ENSURE $rank_{\mathcal{X}}$ - an outlier ranking of objects in $\mathcal{X}$
\FOR{ $i = 1$ to $N$}
\STATE $object\_score(\mathbf{x}_i) \leftarrow 1 - \prod_{j=1}^{D}[1-\phi(x_{ij})]^{\tau(\mathsf{F}_j)}$
\ENDFOR
\STATE $rank_{\mathcal{X}} \leftarrow$ Sort the objects in $\mathcal{X}$ in descending order
\RETURN $rank_{\mathcal{X}} $
\end{algorithmic}
\end{algorithm}

$\text{CBRW}_{\text{od}}$ is used with the default setting $\alpha=0.95$, $\text{SDRW}_{\text{od}}$ is parameter-free, and they are implemented in Java in WEKA \citep{hall2009weka}\footnote{The source codes of CBRW/SDRW-based outlier detection (or feature selection) algorithms are made publicly available at https://sites.google.com/site/gspangsite/sourcecode.}.

\subsection{Performance Evaluation}

\subsubsection{Competing Outlier Detection Methods}
We select five outlier detection methods as our competing detectors. These include: three  categorical data-oriented outlier detectors, FPOF \citep{he2005fp}, CompreX \citep{akoglu2012comprex}, and MarP \citep{das2007marp}; and two numerical data-oriented methods, iForest \citep{liu2012isolation} and Sp \citep{sugiyama2013knn,pang2015lesinn,ting2017defying}. They are used as our competitors because they have shown better or very comparable performance compared to other well-known methods. Thus, they represent the state-of-the-art in outlier detection.

\begin{itemize} 
    \item FPOF employs frequent patterns as normal patterns to detect outliers, which is one of the most popular and effective outlier detectors for categorical data. According to \cite{wu2013information}, FPOF performs better than popular infrequent pattern-based methods, the distance-based method $k$NN and comparably well to the information-theoretic-based method ITB therein. As recommended in \citep{he2005fp}, FPOF is used with the minimum support set to 0.1 and the maximum pattern length set to 5.
    \item CompreX uses data compression costs in an MDL-based feature partition space as outlier scores. It shows powerful performance in data sets with different structures, demonstrating better performance than the information-theoretic-based method KRIMP \citep{smets2011krimp} and the state-of-the-art density-based method LOF \citep{breunig2000lof}. CompreX is parameter-free.
    \item MarP employs the inverse of marginal probabilities to define outliers. Although it is a simple method, it performs better than, or comparably to, conditional probability-based detectors and Bayes Net methods \citep{das2007marp}.  MarP is an efficient and parameter-free detector, with impressive performance on data sets where features are independent. 
    \item iForest uses the number of partitions to isolate an object from other objects to define outliers. It is a top performer among a range of state-of-the-art detectors, including LOF, one-class SVM, Support Vector Data Description and Ensemble Gaussian Mixture Model, on a large number of data sets \citep{emmott2013systematic,liu2012isolation}. Following \citep{liu2012isolation}, iForest is used with the ensemble size set to 100 and the subsampling size set to 256. Note that iForest only works on numeric data, since it requires ordered features to perform data space partition and build isolation trees \citep{liu2012isolation}. Therefore, categorical features are required to be converted into numeric features for iForest to work on categorical data. Two commonly used conversion methods are 1-of-$\ell$ encoding and inverse document frequency (IDF) encoding \citep{campos2016evaluation}. 1-of-$\ell$ encoding converts a categorical feature with $\ell$ values into $\ell$ binary features, in which the values `1' and `0' indicate the presence and absence of a categorical value; whereas IDF encodes a feature value as $\textit{IDF}(v)=\textit{ln}(N/\textit{freq}(v))$. These two methods were examined by a range of $k$ nearest neighbours based outlier detectors and they perform comparably well, as reported in \citep{campos2016evaluation}. In our experiments, iForest using 1-of-$\ell$ encoding performed better than the IDF method. Here we report the best results based on 1-of-$\ell$ encoding.
    \item Sp uses the nearest neighbor distances in a small random subsample as outlier scores. It runs orders of magnitude faster than traditional $k$NN-based detectors that work on the full data set. However, Sp defines the outlier scores in only one random subsample and thus may perform unstably, as shown by the empirical results in \citep{pang2015lesinn}. Following \citep{pang2015lesinn}, bootstrap aggregating was incorporated into Sp to improve its stability. According to \citep{sugiyama2013knn}, the subsampling size is set to 20 in Sp, and 50 random subsamples are generated in the bootstrap aggregating. The widely-used categorical distance measure, Hamming distance, is employed to facilitate the nearest neighbor search.
\end{itemize}

FPOF, MarP, iForest and Sp are also implemented in Java in WEKA \citep{hall2009weka}. CompreX is obtained from the authors of \citep{akoglu2012comprex} in MATLAB.

\subsubsection{Competing Feature Selection Methods}

We compare $\text{SDRW}_{\text{fs}}$ and $\text{CBRW}_{\text{fs}}$ with a closely related feature-weighting method (denoted as ENTR) introduced in \citep{wu2013information} and two baselines. One baseline considers the full feature set, denoted as FULL; the other considers a random feature selector that randomly selects 50\% features as relevant features, denoted as RADM. 

Note that similar to many existing feature selection methods, $\text{SDRW}_{\text{fs}}$, $\text{CBRW}_{\text{fs}}$ and ENTR provide a feature relevance ranking only. Users are required to determine a relevance threshold or the number of selected features to filter out irrelevant features. Although users may be burdened with this parameter setting, this type of methods provides flexibility for users to determine the final selected feature subset. Our experiments show that $\text{SDRW}_{\text{fs}}$ and $\text{CBRW}_{\text{fs}}$ obtain stable performance on all 15 data sets in a wide range of relevance threshold options. We present the results on all the data sets using a consistent relevance threshold to demonstrate the general applicability of $\text{SDRW}_{\text{fs}}$ and $\text{CBRW}_{\text{fs}}$ in practice. Specifically, we observe that an outlier often demonstrates outlying behaviors in a few features only. Since the percentage of outliers is very small, it is reasonable to assume that only a small proportion of features are relevant to outlier detection. We therefore consider to use a small percentage as the relevance threshold. We found that different outlier detectors obtained substantially better performance on some data sets with only 10\% top-ranked features, and on most data sets with 30\% features, compared to their performance on original data. When 50\% features were used, they performed very stably and obtained much better detection performance on all the data sets used in our experiments. We, therefore, select the top-ranked 50\% features as relevant features.

\subsubsection{Efficacy Evaluation Measures}\label{subsec:evaluationmethod}

We evaluate the efficacy of the outlier detectors in terms of its effectiveness and efficiency. The area under ROC curve (AUC), one of the most popular performance measures in outlier detection \citep{campos2016evaluation}, is chosen to evaluate the detection effectiveness. All outlier detectors output a ranking of data objects w.r.t. their outlierness. The AUC is computed based on the ranking \citep{hand2001auc}. A larger AUC indicates better detection effectiveness. 



The efficiency of each detector is compared using runtime in our scale-up test. All the efficiency test experiments are performed at a node in a 3.4GHz Phoenix Cluster with 32GB memory to make the runtime comparable.

We evaluate the effectiveness of our feature selectors, $\text{SDRW}_{\text{fs}}$ and $\text{CBRW}_{\text{fs}}$, from two different aspects: (i) to examine whether they can reduce the complexity of data; and (ii) how they affect the detection performance of subsequent outlier detection. In terms of data complexity reduction, intuitively, a good feature selection method should be able to reduce the complexity of data sets by removing noisy/irrelevant/redundant features while retain the outlier-discriminative information. A set of four data indicators is introduced in Section \ref{subsubsec:complexity} to facilitate the quantization of such data complexities. We then test different outlier detectors on the data sets with the selected features to examine their detection performance on these reduced data. 


\subsection{Creating Benchmark Data Sets}\label{subsec:dataset}

We aim at creating a set of benchmark data sets with various quantitative characterization to enable the understanding of the performance of tested outlier detection methods. We first select a collection of data sets from a wide range of public available real-world data sets using principles like data size and relevance to outlier detection tasks, and then define a set of data indicators to characterize the complexities of each data set. The resulting benchmarks are used in our experiments to examine the performance of outlier detection and feature selection. They are also made publicly available at \url{https://sites.google.com/site/gspangsite} to promote the development and evaluation of new outlier detection (or outlying feature selection) methods. The details of these two steps are presented in the following subsections.

\subsubsection{Selection of the Data Sets}
We use the following two principles to filter out high-quality data sets: (i) the number of objects must be at least 1,000 to avoid potential bias caused by small data sets, and (ii) the data sets must contain semantically meaningful outliers or have an extremely imbalanced class distribution to facilitate direct conversion from classification data into outlier detection data. The latter one helps avoid the reproducibllity and randomization problems caused by downsampling in the data conversion. Fifteen publicly available real-world data sets are finally adopted, which cover diverse domains, e.g., intrusion detection, image recognition, cheminformatics and ecology, as shown in Table \ref{tab:datasummary}.

\textit{Probe} and \textit{U2R} are derived from KDDCUP99 data sets using probe and user-to-root attacks as outliers against the normal class, respectively. The other data sets are transformed from extremely class imbalanced data, where the rare classes are treated as outliers versus the rest of classes as normal class \citep{akoglu2012comprex,liu2012isolation,wu2013information}. 
That is, the small class is used as the outlier class against the large class in the binary class data sets \textit{LINK}, \textit{Census}, \textit{CelebA}, \textit{APAS}, \textit{BM}, \textit{AID362}, \textit{AD} and \textit{w7a}; in the multi-class data sets, we employ the most under-represented class as outliers and merge the rest of classes as normal classes. As a result, the data sets converted from multi-class classification data may have more complex data distributions in the normal class than those converted from binary classification data, multi-modal distribution vs. unimodal distribution, which poses greater challenges to methods that attempt to model the normal class. 

The above conversion methods guarantee that the outlier class chosen is a class with outlying semantics. Since our purpose is to detect outliers in categorical data, data sets with both numerical and categorical features are used with categorical features only. Note that the data conversion methods produce some features that only contain one value for some data sets, i.e., \textit{SF}, \textit{Probe} and \textit{U2R}. These features were removed as they contain no useful information for outlier detection.

\begin{table}[htbp]
  \centering
  \ttabbox{\caption{A Summary of 15 Data Sets and Their Complexity Evaluation Results. The data sets are ordered by the average rank in the last column. The following acronyms are used: BM = Bank Marketing, APAS = aPascal, AD = Internet Advertisements, CMC = Contraceptive Method Choice, SF = Solar Flare, R10 = Reuters10, CT = CoverType, and LINK = Linkage.}\label{tab:datasummary}}
  {
  \scalebox{0.80}{
    \begin{tabular}{|l@{}c@{}c@{}c|p{0.5cm}c|p{0.5cm}c|p{0.5cm}c|p{0.5cm}c|c|}
    \hline
         \multicolumn{4}{|c|}{Basic Data Characteristics} & \multicolumn{2}{c|}{$\kappa_{\mathit{vcc}}$} & \multicolumn{2}{c|}{$\kappa_{\mathit{het}}$} & \multicolumn{2}{c|}{$\kappa_{\mathit{ins}}$} & \multicolumn{2}{c|}{$\kappa_{\mathit{fnl}}$} &  \\\hline
       Data   & $N$ & $D$ & Outliers & Value & Rank  & Value & Rank  & Value & Rank  & Value & Rank   & Avg. Rank \\\hline
     BM & 41,188 & 10 & yes & 21.0\% & 6  & 2.028 & 2  & 0.373 & 3  & 90.0\% & 1  & 3.0 \\
    Census & 299,285 & 33 & 50K+ & 41.9\% & 2  & 1.648 & 3  & 0.238 & 7  & 57.6\% & 4  & 4.0 \\
    AID362 & 4,279 & 114 & active & 32.4\% & 5  & 1.140 & 11 & 0.396 & 2  & 86.0\% & 2  & 5.0 \\
    w7a & 49,749 & 300 & yes & 37.2\% & 3  & 1.059 & 12 & 0.407 & 1  & 48.0\% & 6  & 5.5 \\
    CMC & 1,473 & 8   & \#child$>$10 & 3.8\% & 10 & 1.579 & 4  & 0.344 & 4  & 37.5\% & 7  & 6.3 \\
    APAS & 12,695 & 64  & train & 33.0\% & 4  & 1.192 & 10 & 0.128 & 11 & 81.3\% & 3  & 7.0 \\
    CelebA & 202,599  & 39 & bald & 12.1\% & 8  & 1.265 & 9  & 0.204 & 8  & 48.7\% & 5  & 7.5 \\
    Chess & 28,056 & 6 & zero & 0.0\% & 14 & 2.242 & 1  & 0.264 & 6  & 33.3\% & 9  & 7.5 \\
    AD & 3,279 & 1,555 & ad. & 46.4\% & 1  & 1.011 & 14 & 0.302 & 5  & 4.5\% & 12 & 8.0 \\
    SF & 1,066 & 11  & F  & 12.4\% & 7  & 1.564 & 5  & 0.178 & 9  & 9.1\% & 11 & 8.0 \\\hline
    Probe & 64,759 & 6 & attack & 1.3\% & 12 & 1.324 & 7  & 0.057 & 12 & 0.0\% & 13 & 11.0 \\
    U2R & 60,821 & 6 & attack & 1.5\% & 11 & 1.285 & 8  & 0.015 & 15 & 16.7\% & 10 & 11.0 \\
    LINK & 5,749,132 & 5   & match & 0.6\% & 13 & 1.392 & 6  & 0.021 & 14 & 0.0\% & 13 & 11.5 \\
    R10 & 12,897  & 100    & corn & 6.1\% & 9  & 1.010 & 15 & 0.132 & 10 & 0.0\% & 13 & 11.8 \\
    CT & 581,012  & 44 & cottonwood & 0.0\% & 14 & 1.102 & 13 & 0.029 & 13 & 34.1\% & 8  & 12.0 \\\hline 
    
    \end{tabular}%
    }
    }
\end{table}%

\subsubsection{Characterization of Data Complexities}\label{subsubsec:complexity}

As previously mentioned, these benchmarks are accompanied by four data indicators that describe some intrinsic characteristics of the data. Three of the indicators use class labels to better capture the underlying data complexity. These data indicators enable us to well understand the data, and analyze and evaluate the performance of outlier detection and feature selection.

\vspace{3mm}
\noindent\textit{Heterogeneity of Categorical Distribution}. Most outlier detection methods implicitly assume categorical distributions taken by different features are homogeneous. Measuring the heterogeneity between the distributions of the features therefore offers one basic way to evaluate the problem difficulty in the given data. It is well known that location parameters convey key properties of probability distributions. Many statistical tests, e.g., the t-test or  signed-rank test, are available for distribution location test, but they are ineffective in evaluating the heterogeneity of categorical distributions across the features. This is because the sample size varies significantly in different features and/or data sets, which often violates the sample size and distribution assumptions made in those tests. The problem is simplified as follows. We consider the mode as the key location parameter and use the difference of the frequencies of the modes across features to define the heterogeneity level. Specifically, the heterogeneity level $\kappa_{\mathit{het}}$ is defined as the average difference in mode frequency over all possible feature pairs:
\begin{equation}
    \kappa_{\mathit{het}}=\frac{2}{D(D-1)}\sum_{1 \leq k_i<k_j \leq D}\frac{\mathit{freq}(m_{k_i})}{\mathit{freq}(m_{k_j})},
\end{equation}
\noindent where $\{m_{k_1},\cdots,m_{k_D}\}$ are the modes of all $D$ features sorted based on their frequencies in descending order. $\kappa_{\mathit{het}} \in [1, \infty)$ and a large $\kappa_{\mathit{het}}$ indicates strong heterogeneity.

\vspace{3mm}
\noindent\textit{Value Coupling Complexity}. Existing methods generally assume that the outlierness of values/patterns is independent of each other. Such an assumption is rarely satisfied in real-world applications for many reasons, e.g., the homophily couplings between outlying behaviors. Therefore, we introduce a complexity measure $\kappa_{\mathit{vcc}}$ to show how the homophily outlying couplings affect the detection performance.

We first define two concepts: noisy and positive value couplings. Noisy value couplings are the co-occurrence of multiple infrequent values contained by normal objects. Positive value couplings refer to the co-occurring infrequent values contained by outliers. Positive value couplings are positive because they are consistent with the outlier definition. Noisy value couplings are the opposite of the positive couplings and are therefore negative. Let $N_{C_0}$ and $N_{C_1}$ denote the number of data objects in the outlier class $C_0$  and the normal class $C_1$, respectively. $N_{C_0}^{\prime}$ and $N_{C_1}^{\prime}$ denote the number of outliers and normal objects that contain at least two values with a frequency no more than a low frequency threshold $\theta$. The rate of noisy and positive value couplings can then be defined as $\mathit{nvv}=\frac{N_{C_1}^{\prime}}{N_{C_1}}$ and $\mathit{pvv}=\frac{N_{C_0}^{\prime}}{N_{C_0}}$, respectively. We then define $\kappa_{\mathit{vcc}}$ as 

\begin{equation}
    \kappa_{\mathit{vcc}}=\frac{\mathit{nvv}}{\mathit{pvv}+\mathit{nvv}+\epsilon},
\end{equation}
\noindent where $\epsilon$ is a small constant used to avoid the zero probability problem. 

When $\mathit{nvv} \gg \mathit{pvv}$, the noisy value couplings dominate a data set, leading to high data complexity. The outlier detection task is easy if $\mathit{nvv} \ll \mathit{pvv}$. $\kappa_{\mathit{vcc}} \in [0,1]$, and a larger $\kappa_{\mathit{vcc}}$ indicates more dominant of $\mathit{nvv}$ over $\mathit{pvv}$ and, therefore, has higher data complexity. $\theta=0.05$ and $\epsilon=0.001$ are used in our evaluation.

\vspace{3mm}
\noindent\textit{Outlier Separability}. One basic measure of the difficulty of outlier detection is the separability of outliers from normal objects. However, it is very challenging to exactly compute this difficulty, since the separability varies significantly in different subspaces and the number of possible subspaces is $2^D$. Rather than searching over such a huge space, existing studies focus on the separability in single features, as having strongly relevant features normally enables learning methods to achieve good accuracy \citep{ho2002complexity,leyva2015set}. \textit{Feature efficiency} is widely used for evaluating class separability in supervised classification problems \citep{ho2002complexity,leyva2015set}, in which the separability is defined by the range of non-overlapping values spanned by classes. However, this definition is not suitable for outlier detection, for which data has extremely skewed class distribution. We are interested in the capability of ranking outliers prior to normal objects. Based on the definition of outliers, if there exists a feature where the outliers always contain more infrequent values than normal objects, the outliers can be easily separated from the normal objects. Such features are the most efficient. When outliers and normal objects contain the same values, or outliers contain more frequent values than normal objects, that feature is considered to be inefficient or noisy. The outlier separability of a feature $\mathsf{F}_i$ is defined as
\begin{equation}
    \kappa_{\mathit{sep}}^{i} = \textit{performance}(\mathit{rank}_{i}),
\end{equation}
\noindent where $\mathit{rank}_{i}$ is the ranking list of data objects using the inverse of the frequencies of the values in feature $\mathsf{F}_{i}$ and $\textit{performance}(\cdot)$ is a performance evaluation method. We instantiate $\textit{performance}(\cdot)$ by computing the AUC based on $\mathit{rank}_{i}$. We then use the resulting AUC to denote the efficiency of the feature. Similar to \citep{ho2002complexity,leyva2015set}, the overall outlier separability $\kappa_{\mathit{sep}}$ is represented by the maximum feature efficiency:

\begin{equation}
    \kappa_{\mathit{sep}} = \argmax_{i}{\kappa_{\mathit{sep}}^{i}}.
\end{equation} 

$\kappa_{\mathit{sep}} \in [0,1]$ and a larger $\kappa_{\mathit{sep}}$ indicates better outlier separability and lower data complexity. Below we use outlier inseparability defined as $\kappa_{\mathit{ins}}=1-\kappa_{\mathit{sep}}$. This is to have larger quantization values indicating higher data complexity, making this indicator work consistently with the other indicators. $\kappa_{\mathit{sep}}$ is based on individual features only; it can be seen as a lower bound of the overall separability in all the $2^D$ feature spaces.


\vspace{3mm}
\noindent\textit{Feature Noise Level}. Due to the unsupervised nature of outlier detection, irrelevant features, in which outliers contains values with similar or higher frequencies than that of normal objects, mask normal objects as outliers. These features are therefore `noise' to outlier detection. The presence of a large proportion of these features renders outlier detectors less effective. We accordingly define the feature noise level $\kappa_{\mathit{fnl}}$ as the proportion of noisy features to characterize this difficulty. $\kappa_{\mathit{fnl}}$ is defined as 
\begin{equation}
    \kappa_{\mathit{fnl}} = \frac{\sum_{i}\mathbb{I}(\mathsf{F}_{i})}{D},
\end{equation}
\noindent where $\mathbb{I}(\mathsf{F}_{i})$ is an indicator function, which returns one if feature $\mathsf{F}_{i}$ is considered as a noisy feature and zero otherwise.

A feature is thought as noise if its AUC-based feature efficiency is smaller than 0.5, i.e., outliers are more likely to be assigned with smaller outlier scores than normal objects along that feature in a random selection of object pairs. $\kappa_{\mathit{fnl}} \in [0,1]$ and a large $\kappa_{\mathit{fnl}}$ means a high level of feature noise.

\vspace{3mm}
The quantization results of these data characteristics are reported in Table \ref{tab:datasummary}. We report the value of each data indicator per data set. Since the semantics of indicators differ significantly from each other, we rank the data sets and compute an average rank for the data to obtain an overall complexity quantization. The top-ranked data indicates the highest data complexity. Here we compute the unweighted average rank. If the significance of each indicator is known, a weighted average rank would be more preferable. Note that it is very challenging in understanding, quantifying and balancing the significance of each indicator's contribution to the majority of data sets without prior knowledge and ground truth, given the unknown outlying behaviors and their potential differences. Average ranking thus provides a neutralized and general way before we could explore more sophisticated relations between the indicators.

The top-10 ranked data sets are \textit{BM}, \textit{Census}, \textit{AID362}, \textit{w7a}, \textit{CMC}, \textit{APAS}, \textit{CelebA}, \textit{Chess} , \textit{AD}and \textit{SF}. They are often the top-ranked data sets in terms of individual indicator-based rankings. The complexity of these data sets are, to some extent, verified by the AUC results in Table \ref{tab:auc} in Section \ref{subsec:auc}, in which all outlier detectors obtain substantially smaller AUC results on the top-ranked 10 data sets than on the bottom five data sets.



\section{Evaluation Results}\label{sec:results}

The first three subsections provide the experimental evaluation results of outlier detection, algorithmic component justification, and outlying feature selection on the 15 benchmark data sets, respectively. The fourth subsection reports the scale-up test results. A summary over the evaluation results is provided in the last subsection. 

The above evaluations are conducted for both SDRW and CBRW. In addition, we also tested the convergence and sensitivity test of CBRW. The results show that CBRW converges quickly and performs very stably w.r.t.  parameter $\alpha$ (the only parameter) on all 15 data sets. More details about these results can be found in Appendix \ref{sec:selfcbrw}. SDRW is parameter-free and is implemented using its closed form, so the convergence and sensitivity tests are not applicable. 


\subsection{Detection Performance of $\text{SDRW}_{\text{od}}$ and $\text{CBRW}_{\text{od}}$}\label{subsec:auc}

We first present a summary of detection performance on all data sets, and then analyze the detection performance on complex and simple data sets separately. 

\subsubsection{Overall Performance}
The AUC results of $\text{SDRW}_{\text{od}}$, $\text{CBRW}_{\text{od}}$, MarP, FPOF, CompreX, iForest and Sp on the 15 data sets are presented in Table \ref{tab:auc}. The p-value results are based on a two-sided \textit{Wilcoxon} signed rank test using the null hypothesis that the AUC results of $\text{SDRW}_{\text{od}}$/$\text{CBRW}_{\text{od}}$ and a competing detector come from distributions with equal mean ranks. 

$\text{SDRW}_{\text{od}}$ achieves the best detection performance on 10 data sets, with four close to the best (having the difference in AUC no more than 0.01). $\text{SDRW}_{\text{od}}$ obtains more than 5\% improvement over $\text{CBRW}_{\text{od}}$ and 16\%-28\% improvement over the five competitors in terms of average AUC performance on the top 10 data sets . The significance test results show that $\text{SDRW}_{\text{od}}$ significantly outperforms $\text{CBRW}_{\text{od}}$ and FPOF at the 95\% confidence level and the other four contenders MarP, CompreX, iForest and Sp at the 99\% confidence level. 

$\text{CBRW}_{\text{od}}$ is the best performer on three data sets. Although $\text{CBRW}_{\text{od}}$ underperforms $\text{SDRW}_{\text{od}}$, it performs substantially better than all the competitors in most data sets, achieving 10\%-21\% AUC improvement on the top 10 data sets. $\text{CBRW}_{\text{od}}$ performs significantly better than all the five competitors at the 95\%/99\% confidence level.


It is clear that there exists a large gap of the AUC results between the top 10 data sets and the last five data sets. It is difficult for detectors to obtain a very good performance on the top 10 data sets, whereas all the detectors achieve consistently large AUCs on the last five data sets. We therefore break these data sets into two categories - \textit{complex and simple data}, and discuss each group separately and in detail in the next two subsections.

\begin{table}[htbp]
  \centering
  \ttabbox{\caption{AUC Performance of $\text{SDRW}_{\text{od}}$, $\text{CBRW}_{\text{od}}$ and Their Five Contenders on the 15 Data Sets. `$\circ$' indicates out-of-memory exceptions, while `$\bullet$' indicates that we cannot obtain the results within two months. The middle horizontal line roughly separates \textit{complex} data from \textit{simple} data sets based on average rank in Table \ref{tab:datasummary}. The AUC results of the ensemble methods iForest and Sp are the average over 10 runs. The best performance for each data set is boldfaced. } \label{tab:auc}}
  {
  \scalebox{0.80}{
    \begin{tabular}{|l|cc|ccccc|}
    \hline
     & \multicolumn{2}{|c|}{\textbf{Our Methods}} & \multicolumn{5}{c}{\textbf{Comparison Methods}} \\    
    \hline
    \textbf{Data} & $\textbf{SDRW}_{\text{od}}$ & $\textbf{CBRW}_{\text{od}}$ & \textbf{MarP} & \textbf{FPOF} & \textbf{CompreX} & \textbf{iForest} & \textbf{Sp} \\    
    \hline
    BM & \textbf{0.6511} & 0.6287 & 0.5584 & 0.5466 & 0.6267 & 0.5762 & 0.6006 \\
    Census & 0.6371 & \textbf{0.6678} & 0.5899 & 0.6148 & 0.6352 & 0.5378 & 0.6175 \\
    AID362 & 0.6665 & 0.6640 & 0.6270 & $\circ$ & 0.6480 & 0.6485 & \textbf{0.6678} \\
    w7a & \textbf{0.8059} & 0.6484 & 0.4723 & $\circ$ & 0.5683 & 0.4053 & 0.4517 \\
    CMC & \textbf{0.6415} & 0.6339 & 0.5417 & 0.5614 & 0.5669 & 0.5746 & 0.5901 \\
    APAS & \textbf{0.8544} & 0.8190 & 0.6193 & $\circ$ & 0.6554 & 0.4792 & 0.7401 \\
    CelebA & \textbf{0.8845} & 0.8462 & 0.7358 & 0.7380 & 0.7572 & 0.6797 & 0.7132 \\
    Chess & \textbf{0.8387} & 0.7897 & 0.6447 & 0.6160 & 0.6387 & 0.6124 & 0.6410 \\
    AD & \textbf{0.8482} & 0.7348 & 0.7033 & $\circ$ & $\bullet$  & 0.7084 & 0.7183 \\
    SF & \textbf{0.8817} & 0.8812 & 0.8446 & 0.8556 & 0.8526 & 0.7865 & 0.8434 \\\hline
    Probe & 0.9891 & \textbf{0.9906} & 0.9800 & 0.9867 & 0.9790 & 0.9762 & 0.9654 \\
    U2R & \textbf{0.9941} & 0.9651 & 0.8848 & 0.9156 & 0.9893 & 0.9781 & 0.9886 \\
    LINK & \textbf{0.9978} & 0.9976 & 0.9977 & \textbf{0.9978} & 0.9973 & 0.9917 & 0.9952 \\
    R10 & 0.9837 & \textbf{0.9905} & 0.9866 & $\circ$ & 0.9866 & 0.9796 & 0.9870 \\
    CT & 0.9703 & 0.9703 & \textbf{0.9773} & 0.9772 & 0.9772 & 0.9364 & 0.9601 \\\hline
    Avg. (Top-10) & \textbf{0.7710} & 0.7314 & 0.6337 & 0.6554 & 0.6610 & 0.6009 & 0.6584 \\
    Avg. (All) & \textbf{0.8430} & 0.8152 & 0.7442 & 0.7810 & 0.7770 & 0.7247 & 0.7653 \\\hline
     \multirow{2}{4em}{p-value} &  $\textbf{SDRW}_{\text{od}}$ vs.   & 0.0245 & 0.0006 & 0.0117 & 0.0031 & 0.0001 & 0.0004 \\
         &  &  $\textbf{CBRW}_{\text{od}}$ vs.  & 0.0004 & 0.0137 & 0.0067 & 0.0003 & 0.0020 \\\hline
    \end{tabular}%
    }
    }
\end{table}

\subsubsection{Handling Complex Data}

This analysis is separated into four parts, corresponding to the four indicators in Table \ref{tab:datasummary}.

\vspace{3mm}
\noindent\textit{Results on Data Sets with Highly Complex Value Couplings}. The 10 data sets with the largest proportions of negative value couplings are \textit{AD}, \textit{Census}, \textit{w7a}, \textit{APAS}, \textit{AID362}, \textit{BM}, \textit{SF}, \textit{CelebA}, \textit{R10} and \textit{CMC} according to $\kappa_{vcc}$ in Table \ref{tab:datasummary}. All these data sets fall in the category of complex data except \textit{R10}. Although \textit{R10} has over 6.1\% negative couplings and high dimensionality, it has very simple data distributions, good outlier separability and contains no noisy features, and as a result, even simple outlier detectors like MarP can achieve very good performance.

On these 10 data sets, $\text{SDRW}_{\text{od}}$ achieves an average AUC improvement over MarP (17\%), FPOF (18\%), CompreX (12\%), iForest (23\%) and Sp (13\%), while $\text{CBRW}_{\text{od}}$'s improvement is MarP (12\%), FPOF (13\%), CompreX (7\%), iForest (17\%) and Sp (8\%). $\text{SDRW}_{\text{od}}$ obtains more than 4\% improvement over $\text{CBRW}_{\text{od}}$. 

Most of these data sets contain more than 10\% negative value couplings. This can result in many misleading patterns and, consequently, substantially degrade the performance of traditional outlier detection methods (i.e., MarP, FPOF, CompreX, iForest and Sp). In contrast, the positive homophily couplings captured by $\text{SDRW}_{\text{od}}$ and $\text{CBRW}_{\text{od}}$ enable them to identify outliers more effectively in such adverse environments. Additionally, $\text{SDRW}_{\text{od}}$ incorporates the subgraph density outlier factor $\gamma$, which further enhances its ability to tackle the negative couplings over $\text{CBRW}_{\text{od}}$. For example, \textit{AD} has the most complex value coupling information. By utilizing the pairwise couplings, $\text{CBRW}_{\text{od}}$ obtains about 4\% AUC improvement over the traditional methods that fail to capture those coupling information; $\text{SDRW}_{\text{od}}$ achieves about 15\% improvement over $\text{CBRW}_{\text{od}}$ and the traditional methods by learning both low- and high-order value couplings.

\vspace{3mm}
\noindent\textit{Results on Data Sets with Strong Heterogeneity}. The 10 data sets with the strongest heterogeneity are \textit{Chess}, \textit{BM}, \textit{Census}, \textit{CMC}, \textit{SF}, \textit{LINK}, \textit{Probe}, \textit{U2R}, \textit{CelebA} and \textit{APAS} according to $\kappa_{het}$ in Table \ref{tab:datasummary}. Seven of the ten data sets are categorized as complex, but \textit{LINK}, \textit{Probe}, and \textit{U2R} are actually simple as they have high outlier separability and simple value couplings.

$\text{SDRW}_{\text{od}}$ achieves an average AUC improvement over MarP (13\%), FPOF (10\%), CompreX (8\%), iForest (16\%) and Sp (8\%), while $\text{CBRW}_{\text{od}}$'s improvement is MarP (11\%), FPOF (8\%), CompreX (6\%), iForest (14\%) and Sp (6\%). $\text{SDRW}_{\text{od}}$ and $\text{CBRW}_{\text{od}}$ perform comparably well, with only 1\% difference in the average AUC performance.

Data sets with a large $\kappa_{het}$ indicate diversified frequency distributions across their features, resulting in different semantics of the same frequency in the features. However, the five competitors ignore this characteristic, treating the same frequencies of values/patterns from different features/subspaces equally. This leads to inaccurate outlier scoring of objects. $\text{SDRW}_{\text{od}}$ and $\text{CBRW}_{\text{od}}$ address this issue using distribution center-normalized intra-feature outlier factor. Thus, they perform substantially better than their competitors. For example, on the most heterogeneous data set \textit{Chess}, the incorporation of the intra-feature outlier factor enables $\text{SDRW}_{\text{od}}$ and $\text{CBRW}_{\text{od}}$ to achieve 20\%-30\% improvement over the traditional methods. The contribution of this outlier factor can also be observed in our ablation study in Table \ref{tab:component} in Section \ref{subsec:component}.

\vspace{3mm}
\noindent\textit{Results on Data Sets with Low Outlier Separability}. According to $\kappa_{\mathit{ins}}$ in Table \ref{tab:datasummary}, the 10 data sets with the lowest outlier separability are \textit{w7a}, \textit{AID362}, \textit{BM}, \textit{CMC}, \textit{AD}, \textit{Chess}, \textit{Census}, \textit{CelebA}, \textit{SF} and \textit{R10}. All these data sets are complex data except \textit{R10}.

Compared to MarP, FPOF, CompreX, iForest and Sp, $\text{SDRW}_{\text{od}}$ achieves 16\%, 19\%, 12\%, 20\% and 14\% average improvements, while $\text{CBRW}_{\text{od}}$ obtains over 11\%, 14\%, 7\%, 14\% and 9\% improvements, respectively. $\text{SDRW}_{\text{od}}$ obtains more than 4\% improvement over $\text{CBRW}_{\text{od}}$ on these low separable data sets.

As discussed in Section \ref{subsubsec:sep}, the intra-feature outlier factors in $\text{SDRW}_{\text{od}}$ and $\text{CBRW}_{\text{od}}$ help to increase the outlierness contrast between outlying values and normal values. This may be one main reason for their better performance in this setting. Also, it is interesting to note that the top-ranked data sets in terms of $\kappa_{\mathit{ins}}$ are also top-ranked in terms of either $\kappa_{\mathit{vcc}}$ or $\kappa_{\mathit{het}}$. In other words, the low outlier separability in these data sets is in part due to their underlying non-IID characteristics. Our non-IID methods may therefore perform better than the other five detectors. Also, learning complex patterns, such as high-order patterns or patterns embedded in objects' neighborhoods, is very important to achieve desired performance on this type of data sets, since outliers are not easily separable from normal objects. This might be the main reason for the impressive performance of $\text{SDRW}_{\text{od}}$ on \textit{w7a} and Sp on \textit{AID362}.

\vspace{3mm}
\noindent\textit{Results on Data Sets with High Feature Noise Level}. The 10 data sets with the highest level of feature noise are \textit{BM}, \textit{AID362}, \textit{APAS}, \textit{Census}, \textit{CelebA}, \textit{w7a}, \textit{CMC}, \textit{CT}, \textit{Chess} and \textit{U2R}. All of them are complex data except \textit{U2R}, which has very high outlier separability and simple value couplings.

$\text{SDRW}_{\text{od}}$ achieves an average AUC improvement over MarP (19\%), FPOF (11\%), CompreX (12\%), iForest (23\%) and Sp (13\%), while $\text{CBRW}_{\text{od}}$'s improvement is MarP (14\%), FPOF (7\%), CompreX (8\%), iForest (18\%) and Sp (9\%). $\text{SDRW}_{\text{od}}$ also gains more than 4\% improvement over $\text{CBRW}_{\text{od}}$.

The tolerance of $\text{SDRW}_{\text{od}}$ and $\text{CBRW}_{\text{od}}$ to noisy features is also exemplified by their consistently better performance than the other methods on the highly noisy data sets except \textit{AID362}. One major reason for this noise tolerance is due to their homophily coupling modeling, as discussed in Section \ref{subsubsec:noise}. Sp, which may benefit from the neighborhood information, performs slightly better than $\text{SDRW}_{\text{od}}$ and $\text{CBRW}_{\text{od}}$ on \textit{AID362}. The use of the intra-feature outlier factor $\delta$ in $\text{CBRW}_{\text{od}}$ is less effective on data sets, where outlying values are difficult to distinguish from noisy values. $\text{SDRW}_{\text{od}}$ replaces $\delta$ with the factor $\gamma$ and is, therefore, more tolerant to such data sets. This is justified by the further 4\% improvement obtained by $\text{SDRW}_{\text{od}}$ over $\text{CBRW}_{\text{od}}$. In general, it is far more difficult to achieve desired performance on highly noisy data sets (e.g., \textit{BM} and \textit{AID362}) than other types of complex data discussed above due to the misleading of the noisy features.

\subsubsection{Handling Simple Data}
All seven detectors perform very well on the five simple data sets in Table \ref{tab:auc}. This is particularly true for \textit{R10}, \textit{Probe} and \textit{LINK}, on which all the detectors obtain the AUC of (or nearly) one. The simplicity of these data sets is also exemplified by the performance of the most simple detector MarP which gains nearly perfect AUC performance on these data sets. For such data sets, complex methods such as our methods may unnecessarily complicate the tasks and perform less effectively, such as the results on \textit{CT}. Although some of these data sets (e.g., \textit{R10}) are ranked slightly higher than some complex data sets w.r.t. one or two of the data indicators, they rank toward the bottom in most cases, resulting in an overall low data complexity.

\subsection{Justification of Algorithmic Components of $\text{SDRW}_{\text{od}}$ and $\text{CBRW}_{\text{od}}$}\label{subsec:component}

$\text{SDRW}_{\text{od}}$ and $\text{CBRW}_{\text{od}}$ consist of three major components of the CUOT framework: an intra-feature initial value outlierness, an inter-feature outlierness influence, and a graph mining method that integrates these two components to learn value outlierness. This section presents empirical results to justify the contribution of each component to the value outlierness learning. 

Specifically, we first derive a baseline method, called BASE, that assumes all features are completely independent and only uses the intra-feature outlier factor to obtain value outlierness. We then build on two additional baselines: $\text{SDRWia}_{\text{od}}$/$\text{CBRWia}_{\text{od}}$ weakens the outlierness influence factor by setting $\eta(u,v)=1$ iff $u$ and $v$ co-occur; $\text{SDRWie}_{\text{od}}$/$\text{CBRWie}_{\text{od}}$ uses the original $\eta$ in SDRW/CBRW while ignores the intra-feature outlier factor by setting all $\delta(\cdot)$ to one. 

The AUC results for $\text{SDRW}_{\text{od}}$, $\text{CBRW}_{\text{od}}$ and their variants are shown in Table \ref{tab:component}. The following four observations can be made from these results. 
\begin{enumerate}
    \item BASE substantially underperforms the other six methods on nearly all the data sets. This indicates that assuming the independence of feature-wise outlier factors is often not desirable in practice.
    \item $\text{SDRWia}_{\text{od}}$/$\text{CBRWia}_{\text{od}}$ performs comparably to $\text{SDRWie}_{\text{od}}$/$\text{CBRWie}_{\text{od}}$ in terms of overall performance. This is because they capture only partial value couplings of the data and it works well only when their assumptions fit the specific data sets well. For example, using intra-feature outlier factor with weak outlierness influences is sufficient to obtain very good performance on \textit{Census}, \textit{Chess}, \textit{LINK}, and \textit{R10}, and the performance on these data sets can be substantially downgraded when using strong inter-feature outlierness influences; likewise, inter-feature outlierness influences are much more important on data sets like \textit{BM}, \textit{w7a}, \textit{CelebA}, and \textit{U2R}. As shown in Table \ref{tab:datasummary}, each of these data sets has very different intrinsic characteristics. Significantly weakening intra-feature or inter-feature factors may therefore result in a considerable loss of the detection accuracy in the misfitted data sets.
    \item Although $\text{SDRW}_{\text{od}}$/$\text{CBRW}_{\text{od}}$ performs less effectively than its variants on a few data sets, it obtains averagely better performance and performs more stably. This indicates that the way $\text{SDRW}_{\text{od}}$ and $\text{CBRW}_{\text{od}}$ integrate the two components is generally reasonable, but a better method of integration is needed to improve the performance on the data sets like \textit{Census}, \textit{APAS}, and \textit{CelebA}.
    \item Although $\text{SDRWia}_{\text{od}}$ ($\text{SDRWie}_{\text{od}}$) and $\text{CBRWia}_{\text{od}}$ ($\text{CBRWie}_{\text{od}}$) have comparable overall performance, $\text{SDRW}_{\text{od}}$ demonstrates substantially large improvement over $\text{CBRW}_{\text{od}}$. This indicates that the way $\text{SDRW}_{\text{od}}$ integrates the two components is more reliable than $\text{CBRW}_{\text{od}}$.
\end{enumerate}

\begin{table}[htbp]
  \centering
  \ttabbox{\caption{AUC Performance of $\text{SDRW}_{\text{od}}$, $\text{CBRW}_{\text{od}}$ and Their Variants Created by Removing One or Two Components. The best performance within CBRW/SDRW is boldfaced.} \label{tab:component}}
  {
  \scalebox{0.80}{
    \begin{tabular}{|l|c|ccc|ccc|}
    \hline
    \textbf{Data} & \textbf{BASE} & $\textbf{CBRWia}_{\text{od}}$ & $\textbf{CBRWie}_{\text{od}}$ & $\textbf{CBRW}_{\text{od}}$ & $\textbf{SDRWia}_{\text{od}}$ & $\textbf{SDRWie}_{\text{od}}$ & $\textbf{SDRW}_{\text{od}}$ \\
    \hline
    BM & 0.5778 & 0.5999 & \textbf{0.6566} & 0.6287 & 0.5988 & \textbf{0.6698} & 0.6511 \\
    Census & 0.6033 & \textbf{0.6832} & 0.6579 & 0.6678 & \textbf{0.7259} & 0.6231 & 0.6371 \\
    AID362 & 0.6152 & 0.6034 & 0.6324 & \textbf{0.6640} & 0.6572 & 0.6307 & \textbf{0.6665} \\
    w7a & 0.4744 & 0.4477 & \textbf{0.7363} & 0.6484 & 0.6106 & 0.8002 & \textbf{0.8059} \\
    CMC & 0.5623 & 0.6179 & 0.6323 & \textbf{0.6339} & 0.6075 & 0.6373 & \textbf{0.6415} \\
    APAS & 0.6208 & \textbf{0.8739} & 0.8624 & 0.8190 & 0.6660 & \textbf{0.8604} & 0.8544 \\
    CelebA & 0.7352 & 0.7135 & \textbf{0.9108} & 0.8462 & 0.7367 & \textbf{0.8998} & 0.8845 \\
    Chess & 0.6854 & 0.7766 & 0.4058 & \textbf{0.7897} & 0.7692 & 0.2322 & \textbf{0.8387} \\
    AD & 0.7033 & 0.7250 & \textbf{0.8270} & 0.7348 & 0.6600 & 0.8426 & \textbf{0.8482} \\
    SF & 0.8469 & \textbf{0.8867} & 0.8833 & 0.8812 & 0.8650 & 0.8809 & \textbf{0.8817} \\
    Probe & 0.9795 & 0.9434 & \textbf{0.9907} & 0.9906 & 0.9807 & 0.9854 & \textbf{0.9891} \\
    U2R & 0.8848 & 0.8817 & 0.9640 & \textbf{0.9651} & 0.8793 & \textbf{0.9949} & 0.9941 \\
    LINK & 0.9977 & 0.9976 & 0.9976 & 0.9976 & 0.9976 & 0.9976 & \textbf{0.9978} \\
    R10 & 0.9866 & 0.9823 & 0.9903 & \textbf{0.9905} & \textbf{0.9874} & 0.9837 & 0.9837 \\
    CT & 0.9770 & 0.9388 & \textbf{0.9703} & \textbf{0.9703} & 0.9607 & 0.9581 & \textbf{0.9703} \\\hline
    Avg. & 0.7500 & 0.7781 & 0.8078 & 0.8152 & 0.7802 & 0.7998 & 0.8430 \\\hline
    \end{tabular}%
    }
  }
\end{table}%

\subsection{Outlying Feature Selection Performance of $\text{SDRW}_{\text{fs}}$ and $\text{CBRW}_{\text{fs}}$}

This section presents the results of data complexity reduction by feature selection, followed by the AUC performance of two outlier detectors on the reduced data.

\subsubsection{Data Complexity Reduction}

Table \ref{tab:fscomplexity} shows the results of the data complexity evaluation for each data indicator on the data sets with selected feature subsets as well as full feature sets. 

$\text{SDRW}_{\text{fs}}$ and $\text{CBRW}_{\text{fs}}$ considerably reduce the data complexity in most data indicators on all data sets. Specifically, $\text{SDRW}_{\text{fs}}$ reduces the complexities of $\kappa_{\mathit{vcc}}$, $\kappa_{\mathit{het}}$ and $\kappa_{\mathit{fnl}}$ by 25\%, 7\% and 19\% respectively; and $\text{CBRW}_{\text{fs}}$ achieves respective 25\%, 8\% and 10\% simplification in the indicators $\kappa_{\mathit{vcc}}$, $\kappa_{\mathit{het}}$ and $\kappa_{\mathit{fnl}}$. ENTR obtains markedly large simplification in $\kappa_{\mathit{fnl}}$ and $\kappa_{\mathit{het}}$, whereas it substantially increases the outlier inseparability according to $\kappa_{\mathit{ins}}$. This is because ENTR evaluates the relevance of features without considering their interactions. Thus, noisy features and highly relevant features may be filtered out together. In other words, ENTR reduces the data complexity in terms of $\kappa_{\mathit{fnl}}$ at the expense of increasing the data complexity in terms of $\kappa_{\mathit{ins}}$. Also, ENTR is an entropy-based feature weighting method, which retains features with similar frequency distributions. As a result, ENTR can simplify the data far more than $\text{SDRW}_{\text{fs}}$ and $\text{CBRW}_{\text{fs}}$ in terms of $\kappa_{\mathit{het}}$. However, since it builds upon the feature independence assumption, it can remove features that are very relevant when combining with other features. By contrast, both $\text{SDRW}_{\text{fs}}$ and $\text{CBRW}_{\text{fs}}$ consider the low-level intra- and inter-feature value couplings, which are sensitive to negative value couplings, value frequency distributions and noisy features, resulting in an outlier-separability-secured reduction of data complexity. Additionally, the feature selection can result in some loss of information relevant to outlier detection, hence all methods have negative gains in $\boldsymbol\kappa_{\mathit{ins}}$; compared to ENTR, $\text{SDRW}_{\text{fs}}$ and $\text{CBRW}_{\text{fs}}$ are substantially more outlier-separability-secured.

\begin{table}[htbp]
  \centering
  \ttabbox{\caption{Complexity Quantification of Data Sets with Feature Subsets Selected by $\text{SDRW}_{\text{fs}}$, $\text{CBRW}_{\text{fs}}$, ENTR and FULL. The last row shows the percentage of the average complexity reduction compared to the baseline FULL. We use SD = $\text{SDRW}_{\text{fs}}$, CB = $\text{CBRW}_{\text{fs}}$, EN = ENTR, and FU = FULL to concisely present the results.}\label{tab:fscomplexity}}
  {
  \scalebox{0.80}{
    \begin{tabular}{|l@{}|p{0.40cm}p{0.40cm}p{0.50cm}p{0.48cm}|p{0.40cm}p{0.40cm}p{0.45cm}p{0.48cm}|p{0.40cm}p{0.40cm}p{0.45cm}p{0.48cm}|p{0.40cm}p{0.40cm}p{0.45cm}c|}
    \hline
    &  \multicolumn{4}{|c|}{$\boldsymbol\kappa_{\mathit{vcc}}$} & \multicolumn{4}{c|}{$\boldsymbol\kappa_{\mathit{het}}$} & \multicolumn{4}{c|}{$\boldsymbol\kappa_{\mathit{ins}}$} & \multicolumn{4}{c|}{$\boldsymbol\kappa_{\mathit{fnl}}$}\\
    \hline
    \textbf{Data} & \textbf{SD} & \textbf{CB} & \textbf{EN} & \textbf{FU} & \textbf{SD} & \textbf{CB} & \textbf{EN} & \textbf{FU} & \textbf{SD} & \textbf{CB} & \textbf{EN} & \textbf{FU} & \textbf{SD} & \textbf{CB} & \textbf{EN} & \textbf{FU} \\\hline
    BM & 0.28 & \textbf{0.19} & 0.50 & 0.21 & 1.91 & 1.70 & \textbf{1.30} & 2.03 & \textbf{0.37} & \textbf{0.37} & 0.52 & \textbf{0.37} & 0.80 & \textbf{0.80} & 1.00 & 0.90 \\
    Census & 0.41 & \textbf{0.40} & 0.57 & 0.42 & 1.85 & 1.83 & \textbf{1.15} & 1.65 & \textbf{0.24} & \textbf{0.24} & 0.34 & \textbf{0.24} & 0.71 & 0.65 & 0.76 & \textbf{0.58} \\
    AID362 & \textbf{0.28} & 0.28 & 0.34 & 0.32 & 1.01 & 1.04 & \textbf{1.01} & 1.14 & \textbf{0.40} & \textbf{0.40} & 0.48 & \textbf{0.40} & 0.93 & 0.93 & 0.96 & \textbf{0.86} \\
    w7a & 0.13 & 0.20 & \textbf{0.10} & 0.37 & 1.01 & 1.01 & \textbf{1.00} & 1.06 & \textbf{0.41} & \textbf{0.41} & 0.44 & \textbf{0.41} & \textbf{0.01} & 0.23 & 0.03 & 0.48 \\
    CMC & 0.04 & 0.04 & \textbf{0.00} & 0.04 & 1.30 & 1.30 & \textbf{1.27} & 1.58 & \textbf{0.34} & \textbf{0.34} & 0.37 & \textbf{0.34} & \textbf{0.00} & \textbf{0.00} & 0.50 & 0.38 \\
    APAS & 0.25 & \textbf{0.22} & 0.33 & 0.33 & 1.06 & 1.06 & \textbf{1.02} & 1.19 & \textbf{0.13} & \textbf{0.13} & 0.28 & \textbf{0.13} & 0.69 & \textbf{0.66} & 0.88 & 0.81 \\
    CelebA & \textbf{0.08} & \textbf{0.08} & 0.12 & 0.12 & 1.20 & 1.16 & \textbf{1.05} & 1.26 & \textbf{0.20} & \textbf{0.20} & 0.32 & \textbf{0.20} & \textbf{0.15} & 0.20 & 0.40 & 0.49 \\
    Chess & 0.00 & 0.00 & 0.00 & 0.00 & \textbf{1.22} & \textbf{1.22} & 2.05 & 2.24 & \textbf{0.26} & 0.26 & 0.26 & 0.26 & 0.67 & 0.67 & \textbf{0.00} & 0.33 \\
    AD & \textbf{0.26} & 0.26 & 0.37 & 0.46 & 1.01 & 1.00 & \textbf{1.00} & 1.01 & \textbf{0.30} & 0.34 & 0.47 & \textbf{0.30} & 0.01 & 0.01 & \textbf{0.00} & 0.05 \\
    SF & \textbf{0.11} & 0.15 & 0.15 & 0.12 & 1.72 & 1.72 & \textbf{1.08} & 1.56 & \textbf{0.18} & \textbf{0.18} & 0.30 & \textbf{0.18} & \textbf{0.00} & \textbf{0.00} & 0.17 & 0.09 \\
    Probe & 0.00 & 0.01 & \textbf{0.00} & 0.01 & 1.42 & 1.36 & \textbf{1.04} & 1.32 & \textbf{0.06} & \textbf{0.06} & 0.07 & \textbf{0.06} & 0.00 & 0.00 & 0.00 & 0.00 \\
    U2R & 0.00 & 0.01 & \textbf{0.00} & 0.02 & 1.37 & 1.35 & \textbf{1.00} & 1.29 & \textbf{0.02} & \textbf{0.02} & 0.15 & \textbf{0.02} & \textbf{0.00} & 0.33 & \textbf{0.00} & 0.17 \\
    LINK & \textbf{0.00} & \textbf{0.00} & 0.01 & 0.01 & 1.19 & 1.19 & \textbf{1.18} & 1.39 & 0.02 & 0.02 & 0.02 & 0.02 & 0.00 & 0.00 & 0.00 & 0.00 \\
    R10 & 0.03 & 0.01 & \textbf{0.00} & 0.06 & 1.00 & 1.00 & \textbf{1.00} & 1.01 & 0.34 & \textbf{0.13} & 0.44 & \textbf{0.13} & 0.00 & 0.00 & 0.00 & 0.00 \\
    CT & 0.00 & 0.00 & 0.00 & 0.00 & 1.17 & 1.17 & \textbf{1.00} & 1.10 & \textbf{0.03} & \textbf{0.03} & 0.32 & \textbf{0.03} & 0.45 & 0.45 & \textbf{0.00} & 0.34 \\\hline
    Avg. & 0.12 & \textbf{0.12} & 0.17 & 0.17 & 1.30 & 1.27 & \textbf{1.14} & 1.39 & 0.22 & 0.21 & 0.32 & \textbf{0.21} & \textbf{0.29} & 0.33 & 0.31 & 0.36 \\\hline
    $\triangledown$ (\%) & \centering 25 & \centering 25 & \centering -0.2 &  \centering -  & \centering 7 & \centering 8 & \centering 18 &  \centering -  & \centering -7 & \centering -1 & \centering -55 & \centering  -  & \centering 19 & \centering 10   & \centering 14 &   - \\ 
    \hline
    \end{tabular}%
    }
    }
\end{table}%

\subsubsection{Effect on the Performance of Different Subsequent Outlier Detectors}

The effectiveness of $\text{SDRW}_{\text{fs}}$ and $\text{CBRW}_{\text{fs}}$ is further verified by the AUC performance of different subsequent outlier detectors using their resultant feature subsets. Two different outlier detectors, MarP and iForest, are used here, because they are commonly-used baselines.

The AUC performance of MarP and iForest working on the data sets with feature subsets is shown in Table \ref{tab:fsauc}. Both of $\text{SDRW}_{\text{fs}}$- or $\text{CBRW}_{\text{fs}}$-empowered MarP and iForest obtains substantial improvements than ENTR (12\%), RADM (17\%) and FULL (7\%), regardless of the difference working mechanisms of MarP and iForest. In particular, the $\text{SDRW}_{\text{fs}}$- and $\text{CBRW}_{\text{fs}}$-empowered MarP and iForest significantly outperform their counterparts empowered by ENTR and RADM at the 99\% confidence level. Although they use 50\% less features, they significantly outperforms MarP and iForest working on data with full feature sets at the 95\% confidence level. The superiority of $\text{SDRW}_{\text{fs}}$/$\text{CBRW}_{\text{fs}}$ is understandable, since it considerably reduces the levels of negative value couplings, heterogeneity and feature noise while retaining the outlier separability (i.e., the most relevant features), as shown in Table \ref{tab:fscomplexity}. Note that the aggressive removal of 50\% features may result in information loss in some cleaned data sets such as \textit{SF} and \textit{CT}, which leads to less effective performance than that working on the full feature space (i.e., FULL).

\begin{table}[htbp]
  \centering
  \ttabbox{\caption{AUC Performance of MarP and iForest Using $\text{SDRW}_{\text{fs}}$, $\text{CBRW}_{\text{fs}}$, ENTR, RADM, and FULL. The results of RADM and/or iForest are the average over 10 runs.}\label{tab:fsauc}}
  {
  \scalebox{0.80}{
    \begin{tabular}{|l@{}|ccp{0.8cm}p{0.8cm}p{0.9cm}|ccp{0.8cm}p{0.8cm}c|}
    \hline
       & \multicolumn{5}{|c|}{\textbf{MarP}} & \multicolumn{5}{c|}{\textbf{iForest}} \\
    \hline
    Data & $\text{SDRW}_{\text{fs}}$ & $\text{CBRW}_{\text{fs}}$ & ENTR & RADM & FULL & $\text{SDRW}_{\text{fs}}$ & $\text{CBRW}_{\text{fs}}$ & ENTR & RADM & FULL \\\hline
    BM & 0.5627 & \textbf{0.5926} & 0.4886 & 0.5181 & 0.5584 & 0.5618 & \textbf{0.5836} & 0.5297 & 0.5544 & 0.5762 \\
    Census & 0.6052 & \textbf{0.6258} & 0.4525 & 0.5490 & 0.5899 & 0.5801 & \textbf{0.6106} & 0.4403 & 0.5201 & 0.5378 \\
    AID362 & 0.6612 & \textbf{0.6620} & 0.5909 & 0.6074 & 0.6270 & \textbf{0.6641} & 0.6525 & 0.6155 & 0.6267 & 0.6485 \\
    w7a & 0.8413 & 0.7654 & \textbf{0.8633} & 0.4594 & 0.4748 & 0.8084 & 0.7432 & \textbf{0.8251} & 0.3946 & 0.4053 \\
    CMC & 0.6474 & \textbf{0.6474} & 0.5082 & 0.5062 & 0.5417 & \textbf{0.6609} & 0.6607 & 0.5288 & 0.5164 & 0.5746 \\
    APAS & 0.8454 & \textbf{0.8569} & 0.6346 & 0.5995 & 0.6193 & 0.8385 & \textbf{0.8426} & 0.6372 & 0.5543 & 0.4792 \\
    CelebA & \textbf{0.8652} & 0.8597 & 0.7785 & 0.7102 & 0.7358 & 0.8388 & \textbf{0.8438} & 0.7799 & 0.6764 & 0.6797 \\
    Chess & \textbf{0.7574} & \textbf{0.7574} & 0.6378 & 0.6076 & 0.6447 & \textbf{0.6859} & 0.6138 & 0.6241 & 0.5829 & 0.6124 \\
    AD & \textbf{0.8256} & 0.7624 & 0.6603 & 0.6888 & 0.7033 & \textbf{0.8206} & 0.7620 & 0.6592 & 0.6775 & 0.7084 \\
    SF & 0.8343 & 0.8157 & 0.6666 & 0.8181 & \textbf{0.8446} & 0.7838 & 0.7667 & 0.6856 & 0.7660 & \textbf{0.7865} \\
    Probe & \textbf{0.9837} & 0.9805 & 0.9307 & 0.8951 & 0.9800 & \textbf{0.9842} & 0.9751 & 0.8797 & 0.8990 & 0.9762 \\
    U2R & \textbf{0.9937} & 0.8846 & 0.8582 & 0.7911 & 0.8848 & \textbf{0.9879} & 0.9776 & 0.7854 & 0.8168 & 0.9781 \\
    LINK & \textbf{0.9985} & \textbf{0.9985} & 0.9938 & 0.9723 & 0.9977 & \textbf{0.9986} & 0.9984 & 0.9797 & 0.9636 & 0.9917 \\
    R10 & 0.8705 & \textbf{0.9893} & 0.7648 & 0.9627 & 0.9866 & 0.8705 & \textbf{0.9926} & 0.7566 & 0.9541 & 0.9796 \\
    CT & 0.8570 & 0.8570 & 0.8581 & 0.6154 & \textbf{0.9773} & 0.9122 & 0.9072 & 0.8816 & 0.6374 & \textbf{0.9364} \\\hline
    Avg. & \textbf{0.8099} & 0.8037 & 0.7125 & 0.6867 & 0.7444 & \textbf{0.7998} & 0.7954 & 0.7072 & 0.6760 & 0.7247 \\\hline
     \multirow{2}{4em}{p-value}  & SDRW vs. & 0.6772 & 0.0009 & 0.0016 & 0.0340 & \centering - & 0.2890 & 0.0005 & 0.0013 & 0.0262 \\
     &  \multicolumn{2}{l}{CBRW vs.}    & 0.0023 & 0.0005 & 0.0113 & \centering -   &  \centering -  & 0.0023 & 0.0004 & 0.0131 \\\hline
    \end{tabular}%
    }
    }
\end{table}%

MarP and iForest using ENTR perform much worse than that working on the full feature set on almost all the data sets. This is because ENTR wrongly removes highly relevant features and degrades the outlier separability of the data sets. As discussed above, this aggravates the detection performance of subsequent outlier detectors. It is interesting to note that MarP and iForest using $\text{SDRW}_{\text{fs}}$ and ENTR perform much better than all their counterparts on \textit{w7a}. This improvement is mainly because $\text{SDRW}_{\text{fs}}$ and ENTR remove more than 95\% of the noisy features with little or no loss to the outlier separability in this data, as shown in Table \ref{tab:fscomplexity}. The success of ENTR on \textit{w7a} indicates that many noisy features in \textit{w7a} have less skewed frequency distributions than outlying features. As a result, simply examining the frequency distributions of individual features is probably the best way to clean up those noisy features.

\subsection{Scalability Test}
Both SDRW/CBRW-based outlier detection and feature selection are linear consolidation of the value outlierness. Hence, they have similar scalability. Here, we show the scalability of $\text{SDRW}_{\text{od}}$ and $\text{CBRW}_{\text{od}}$. 

The scalability of $\text{SDRW}_{\text{od}}$ and $\text{CBRW}_{\text{od}}$ w.r.t. data size is evaluated using four subsets of the largest data set \textit{LINK}. The smallest subset contains 64,000 objects, and subsequent subsets are increased by a factor of four, until the largest subset which contains 4,096,000 objects. 

The scaleup test results w.r.t. data size are presented in the left panel in Fig. \ref{fig:scaleuptest}. As expected, all the seven detectors have runtime linear w.r.t. data size. $\text{SDRW}_{\text{od}}$ runs faster than iForest and Sp by a factor of more than 20 and 30, respectively. $\text{SDRW}_{\text{od}}$ runs slightly faster than $\text{CBRW}_{\text{od}}$, since $\text{SDRW}_{\text{od}}$ requires no iteration to obtain the value outlierness. Nevertheless, $\text{CBRW}_{\text{od}}$ runs faster than iForest and Sp. Both $\text{SDRW}_{\text{od}}$ and $\text{CBRW}_{\text{od}}$ are slightly slower than MarP but comparably fast to FPOF.

\begin{figure}[h!]
  \centering
    \includegraphics[width=0.6\textwidth]{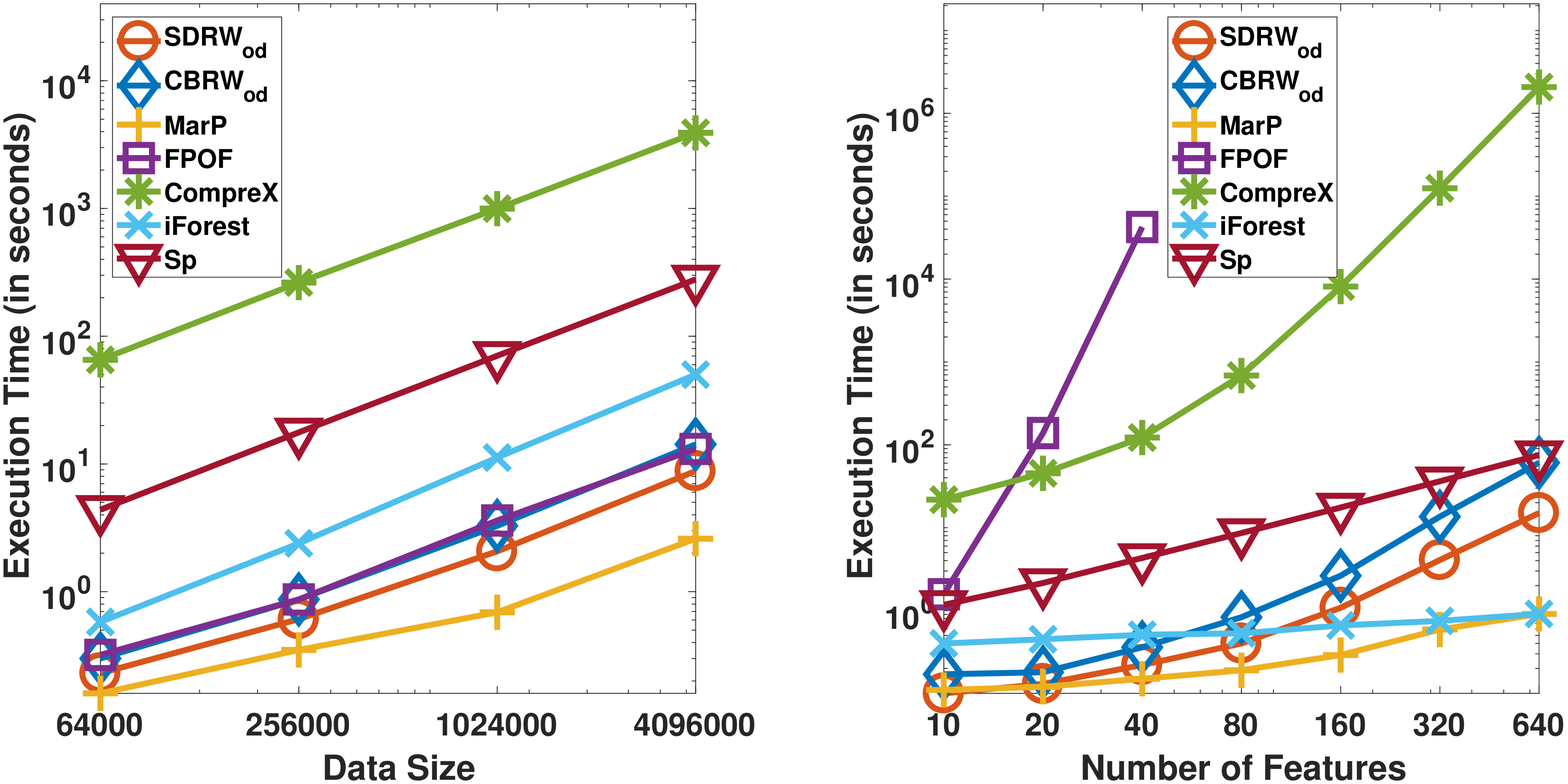}
  \caption{Scale-up Test Results of the Seven Detectors w.r.t. Data Size and Dimensionality. Logarithmic scales are used in both axes. 
  Note that FPOF runs out-of-memory when the number of features reaches 80.}
  \label{fig:scaleuptest}
\end{figure}

The scale-up test w.r.t. the number of features is conducted using seven synthetic data sets. The data sets have the same number of objects, i.e., 10,000 objects. The data set with the smallest number of features contains 10 features, and subsequent data sets are increased by a factor of two, until the data set with the largest number of features contains 640 features.

The results reported in the right panel in Fig. \ref{fig:scaleuptest} show that, as expected, $\text{SDRW}_{\text{od}}$ and $\text{CBRW}_{\text{od}}$ have runtime quadratic w.r.t. the number of features, which runs more than five orders of magnitude faster than FPOF. $\text{SDRW}_{\text{od}}$ and $\text{CBRW}_{\text{od}}$ run much faster than CompreX by a factor of more than 600 and 250 in terms of runtime ratio\footnote{Since CompreX was implemented in a different programming language to the other methods, the runtime between CompreX and other methods is incomparable. Instead, we compare them in terms of runtime ratio, i.e., the runtime on a larger/higher-dimensional data set divided by that on a smaller/lower-dimensional data set, for a fairer comparison. Since the data size and the increasing factor of dimensionality are fixed, the runtime ratio is comparable across the methods in different programming languages.}, respectively. Compared to Sp, $\text{SDRW}_{\text{od}}$ and $\text{CBRW}_{\text{od}}$ run faster on data sets with lower dimensions, but they may become slower on data sets with higher dimensions. This is because the runtime of $\text{SDRW}_{\text{od}}$ and $\text{CBRW}_{\text{od}}$ increase at a much faster rate than Sp. Since $\text{SDRW}_{\text{od}}$ and $\text{CBRW}_{\text{od}}$ model much more complex data characteristics than MarP and iForest, it runs substantially slower than these two competitors, but with significantly better accuracy in terms of AUC, as shown in Table \ref{tab:auc}.

\subsection{Evaluation Summary}

We summarize the empirical results as follows.

\begin{itemize}

\item $\text{SDRW}_{\text{od}}$ and $\text{CBRW}_{\text{od}}$ significantly outperform the five competitors FPOF, MarP, CompreX, iForest and Sp at the 95\% or 99\% confidence level. This is mainly because $\text{SDRW}_{\text{od}}$ and $\text{CBRW}_{\text{od}}$ well capture the non-IID outlierness in real-world data sets, enabling them to  work well not only on non-IID data but also on data with low outlier separability and/or noisy features. Compared to the five competitors, on average, $\text{SDRW}_{\text{od}}$ obtains 16\%-28\% AUC improvement, and $\text{CBRW}_{\text{od}}$ obtains 10\%-21\% improvement on the 10 most complex data sets. All seven outlier detectors perform comparably well on the five simple data sets.

\item $\text{SDRW}_{\text{od}}$ performs significantly better than $\text{CBRW}_{\text{od}}$, achieving more than 3\% and 5\% AUC improvement over $\text{CBRW}_{\text{od}}$ in all 15 data sets and the top-ranked 10 complex data sets, respectively. Also, $\text{SDRW}_{\text{od}}$ is the best performer among the seven variants of $\text{SDRW}_{\text{od}}$ and $\text{CBRW}_{\text{od}}$. This justifies the refinement of inter-feature outlierness influences and the addition of a subgraph density-based outlier factor in SDRW.

\item In terms of outlying feature selection, $\text{SDRW}_{\text{fs}}$ and $\text{CBRW}_{\text{fs}}$ show the best results on data sets with selected feature subsets, which have substantially lower data complexities according to three of the complexity indicators, complex value couplings, heterogeneity, and feature noise level, while retaining the outlier separability of the data sets. As a result, $\text{SDRW}_{\text{fs}}$ and $\text{CBRW}_{\text{fs}}$ enable two different outlier detectors, MarP and iForest, to obtain significantly better AUC performance than their competing methods.

\item As expected, SDRW and CBRW have time complexity linear w.r.t. data size and quadratic w..r.t. the number of features. Particularly, $\text{SDRW}_{\text{od}}$ and $\text{CBRW}_{\text{od}}$ run orders of magnitude faster than the pattern-based methods FPOF and CompreX, and they are comparably fast to the very efficient ensemble methods iForest and Sp.

\end{itemize}

\section{Related Work}\label{sec:relatedwork}

We first review the literature of four closely related areas, including IID and non-IID outlier detection, feature selection, and data characterization for outlier detection, and then provide a summary at the end of this section.

\subsection{IID Outlier Detection Methods}
Most of existing methods for outlier detection assume data is IID, e.g., ignoring the couplings and heterogeneities within and between outlier factors. The methods for both numerical and categorical data are reviewed below.

\subsubsection{Methods for Numeric Data}

Most of existing outlier detection methods focus on numeric data, including distance-based methods, density-based methods, and clustering-based methods \citep{aggarwal2017outlieranalysis}. They assume that objects in the regions of low density are outliers. To estimate the density of a region, they rely on expensive distance computations with a time complexity of $O(N^{2})$, where $N$ denotes the number of objects. This computational time may be reduced to $O(N\; \text{log}\; N)$ if the objects are pre-indexed by an indexing scheme like $R^{*}$-tree or $k$-d tree. It has not been well explored about how well these indexing techniques handle categorical data. In recent years, several ensemble-based methods \citep{guha2016robust,liu2012isolation,pang2015lesinn,sugiyama2013knn,ting2017defying,ting2013mass,zimek2013ensembles} perform $k$-nearest-neighbor searching or data space partition on small subsamples to compute outlier scores. These methods have a (nearly) linear time complexity w.r.t. data size and the number of features, but they can be ineffective when handling categorical data. This is mainly because categorical features do not contain the order information. Therefore, it is difficult to define a good similarity measure for different data sets \citep{aggarwal2017outlieranalysis,boriah2008similarity}, and the data partition methods in \citep{liu2012isolation,ting2013mass} are inapplicable. More importantly, these methods are based on the assumption that the outlier factor is identical and independent for all data entities. Unlike these methods, SDRW and CBRW are built on the non-IID outlier factors and do not involve distance computations.

\subsubsection{Methods for Categorical Data}
By contrast, significantly less research has been conducted on categorical data. Among existing methods, most are pattern-based to address its discrete nature in categorical values. They typically search for infrequent/frequent patterns (or rules) using approaches such as frequent pattern mining \citep{he2005fp,otey2006fast,smets2011krimp,tang2015contextualoutliers}, information-theoretic measures \citep{akoglu2012comprex,wu2013information}, probability tests \citep{das2007marp,wong2003bayesian}, and logic rules \citep{angiulli2008outlier,angiulli2010outlier}, and build pattern-based detection models. Those objects with infrequent patterns are considered outlying. These methods ignore the interdependence of patterns-based outlier factors and calculate the outlier scores of the individual patterns. As a result, they may fail to capture the genuine outlying degree of the patterns and overlook important outliers, especially in complex data. For example, in data sets with many noisy features, these methods identify a large proportion of misleading patterns. Since they treat potentially wrong and correct patterns independently, all patterns are scored in an identical way, which can mislead the outlier scoring process and incorrectly report many normal objects as outliers. In addition, the resultant patterns are derived from different feature combinations, where features may have very different frequency distributions. Accordingly, the semantic and importance of the pattern frequency may also differ significantly for different patterns. As such, existing pattern frequency-based methods cannot appropriately capture the outlierness. 

In practice, the pattern discovery-based detection \citep{das2007marp,he2005fp,otey2006fast,smets2011krimp} has time and space complexities that are exponential to the number of features. Though a heuristic search was used in \citep{akoglu2012comprex} to reduce the complexity from exponential to quadratic, the search is still computationally intensive in high dimensional data. Some other work accelerates the pattern discovery phase, e.g., by searching condensed representations of patterns \citep{koufakou2011non} or pattern sampling \citep{giacometti2016anytime}. However, they may overlook important outliers and/or perform unstably. 

In addition, it is non-trivial to tune the parameters (e.g., the minimum support and pattern length) involved in the pattern search, as the characteristics of patterns differ greatly between data sets \citep{akoglu2012comprex,giacometti2016anytime}. For example, in frequent pattern-based methods, a small minimum support generates a substantially large set of normal patterns, which may lead to false alarms and more expensive computations; while large minimum support results in an insufficient number of patterns, leading to a high false negative error. 

Our proposed CBRW and SDRW are not pattern-based. Instead, they model the non-IID outlying behaviors at the feature value level. Also, SDRW is parameter-free, and CBRW only involves one parameter, the damping factor, which guarantees algorithmic convergence and can be easily set to obtain the desired efficacy. There are some concurrent work \citep{pang2017selective,pang2017homophily,jian2018cure} on modeling the feature value couplings for outlier detection, but they mainly focus on pairwise value couplings. \cite{jian2018cure} uses the similarity between feature values and clustered value sets to capture richer value couplings, but they ignore the homophily couplings.

\subsection{Non-IID Outlier Detection Methods}

Many recent methods have been developed for outlier detection on data with explicit non-independent properties, e.g., graph or temporal data \citep{akoglu2015graphsurvey,gupta2014temporaloutlier}. However, little research is available on leveraging the implicit coupling/interdependent information to improve the aforementioned IID outlier detection methods, though it shows to be effective in various other domains, e.g., behavior analysis \citep{cao2012coupled}, image/text categorization \citep{zhou2009multi,ganiz2011noniidtc}, statistical learning \citep{FanXC16}, representation learning \citep{jian2017embedding}, and recommender systems and deep learning \citep{ijcai_ZhangCZLS18}. Particularly, homophily coupling is a commonly observed relation in many real-world applications \citep{fowler2008dynamic,koutra2011homophily,tang2013homophily}. The homophily couplings between outlying behaviors have been recently verified in fraud detection \citep{mcglohon2009homophily} and malware detection \citep{chau2011polonium}. These two studies assume the misstated user accounts and malicious files have homophily couplings, respectively. Our work is very different from these studies in two major aspects below. (i) They incorporate domain knowledge (e.g., some labeled data objects) into their homophily learning models through semi-supervised learning, whereas our methods do not require labeled data. (ii) They focus on domain-specific problems using graph data, e.g., account-account graph or file-file graph, whereas we investigate the homophily at the value level in generic categorical data, which has broader applications. Other relevant research lines may include contextual outlier detection and sequential ensemble methods. Contextual outlier detection aims at identifying outliers in some targeted dimensions conditioned on the behaviors demonstrated in other dimensions, in which the targeted dimensions are behavioral features of interest while the conditional dimensions are referred to as contextual features \citep{liang2016contextual,zheng2017contextual}. Although both contextual outlier detection and our work leverage conditional relations in their outlierness, they are very different in that the former learns contextual outlierness while the latter learns global outlierness with the focus on the propagation of the outlierness across all the features. Sequential ensembles \citep{rayana2016sequential,aggarwal2017outlieranalysis,pang2018sparse} are also relevant in that they build a set of sequentially dependent detection models to reduce potential bias and variance. Different from that, we focus on the interdependence between feature values with a single model.

On the other hand, learning different types of heterogeneities, such as heterogeneities between features, subspaces, views, objects, class labels, graphs, or learning methods \citep{cao2014computer,he2017heterogeneity,sun2012hin}, has been explored in a range of learning tasks. However, only a few studies \citep{rayana2016less,schubert2012outlierscore} have been reported on outlier detection. Moreover, these studies mainly focus on how to construct heterogeneous ensembles of different outlierness scoring methods (or one method with different parameter settings) by leveraging the techniques of ensemble learning. In contrast, our work explores the heterogeneous probability distributions in the original data to design more faithful outlierness scoring functions. In practice, a common heterogeneous data is the one containing both categorical and numeric features \citep{otey2006fast,koufakou2010mixed}. Our approach may extend to this setting via discretization of numeric features before applying our algorithms. However, data discretization for outlier detection remains an open problem because of the significant challenges brought by the unsupervised nature of outlier detection and the scattered and different distributions taken by different outliers. We therefore put this in our future work and focus on categorical data here.

\subsection{Feature Selection for Outlier Detection}
Proper feature selection can remove irrelevant and noisy features to improve the accuracy and efficiency of outlier detection on complex data. However, existing feature selection research focuses on regression, classification, and clustering \citep{li2016feature}. Very limited feature selection methods have been designed for outlier detection, e.g., selecting features for imbalanced data classification or supervised outlier detection \citep{azmandian2012icdm,maldonado2014imbalanced}. However, they are inapplicable for the context having no class label information or being costly to obtain class labels. Unfortunately, many real-world outlier detection applications fall in this scenario.

Even less work is available on unsupervised feature selection for outlier detection. This is due to the challenges associated with evaluating feature relevance in data that does not have class labels or have an extremely imbalanced distribution. Two pieces of related work are \citep{he2010coselection} and \citep{wu2013information}. In \citep{he2010coselection}, a partial augmented Lagrangian method is introduced to co-select objects and features that are relevant to rare class detection. While the feature selection and rare class detection are shown effective in unsupervised settings, as pointed out by the authors, they assume that the objects in rare classes are strongly self-similar. This assumption does not apply to the nature of outlier detection, where many outliers may be isolated objects and distributed far away from each other in the data space. The unsupervised feature weighting for outlier detection on categorical data in \citep{wu2013information} is more related to our work. This method employs an entropy-based measure to weight features and highlight strongly relevant features for subsequent outlier detection. However, it evaluates individual features without considering any feature interactions, and is thus very sensitive to noisy features. SDRW/CBRW-based feature selection considers intra-feature and inter-feature value interactions to evaluate and differentiate relevant and irrelevant features.

\subsection{Data Characterization for Outlier Detection}

Data characterization or data complexity quantification aims to understand and quantify the underlying data characteristics and complexities \citep{dst_Cao15} that are critical for perfect fitness between data and models. This has shown to be critical for the design and the evaluation of learning methods \citep{campos2016evaluation,cao2016nsp,emmott2013systematic,ho2002complexity,leyva2015set,smith2014instance}. A wide range of data indicators have been introduced to quantify data complexity at the feature and/or object levels for classification tasks or sequence analysis \citep{cao2016nsp,ho2002complexity,leyva2015set,smith2014instance}, while little work has been done for outlier detection.

Two relevant studies are \citep{campos2016evaluation,emmott2013systematic}. In \citep{campos2016evaluation}, a variety of k-nearest-neighbour-based outlier detection methods is employed to evaluate the complexity of many publicly accessible data sets via their detection performance. Two data indicators, \textit{difficulty} and \textit{diversity}, are defined based on agreements and conflicts in the performance of the detectors. This is fundamentally different from our work in that we quantify the data complexity from specific data aspects by designing various data indicators to capture different underlying data characteristics. A more related work is \citep{emmott2013systematic}, which introduces three indicators, \textit{point difficulty}, \textit{clusteredness}, and \textit{relative frequency}, to create benchmark data sets with different characteristics by varying these three indicators. These indicators are designed at the object level and are mostly proximity-based. In contrast, we introduce the four indicators that span the value to object levels to capture more affluent data characteristics. Another key difference is that we focus on creating benchmark categorical data, whereas \citep{campos2016evaluation,emmott2013systematic} pertains to numeric data.\\

\section{Conclusions and Future Work} \label{sec:con}

\subsection{Conclusions}

This paper introduces a novel outlier detection framework (CUOT). Compared to traditional pattern-based frameworks that focus on the identification of outlying/normal patterns (e.g., inter-feature outlier factor) and assume the independence and homogeneity of these patterns to compute pattern outlierness, CUOT synthesizes both intra-feature outlier factors and their complex interactions to learn value outlierness. As a result, CUOT has the ability to handle the complex non-IID outlier detection problems. We further instantiate CUOT to two methods, called CBRW and SDRW, to identify outliers in data with non-IID outlying behaviors, data with low outlier separability, and/or data with many noisy features. CBRW and SDRW define the mode-based normalization in the intra-feature outlier factors to handle the heterogeneity and low outlier separability issues. They also leverage graph-based methods (and dense subgraph discovery techniques) to model the homophily couplings between outlying behaviors and build more noise-tolerant outlier detectors. 

To provide a thorough empirical justification, we introduce a set of four data indicators, which describe some intrinsic data characteristics, to provide important insights into the outlier detection and/or feature selection performance. The four indicators quantify the data complexity of 15 publicly available data sets in terms of value heterogeneity, homophily couplings, outlier separability and feature noise level. We further make these data sets publicly accessible to create benchmark data sets with quantitative complexities to promote the development and evaluation of outlier detection methods in complex scenarios.

Extensive experiments show that (i) the four data indicators help to separate complex and simple data; (ii) our CBRW and SDRW-based outlier detector perform significantly better than five well-known detectors - MarP, FPOF, CompreX, iForest and Sp, at the 95\% confidence level; (iii) our CBRW and SDRW-based outlying feature selection method considerably reduce the complexities of different data sets while retain their outlier separability, and they significantly outperform the competing feature selection methods; and (iv) CBRW and SDRW have runtime linear w.r.t. data size and quadratic w.r.t. the number of features. All these characteristics make SDRW and CBRW a good candidate for outlier detection in large-scale data with complex underlying characteristics.

Compared to CBRW, SDRW is significantly better in terms of overall AUC performance, computational time, parameter tuning effort and tolerance to noisy features. Therefore, SDRW is generally recommended when the complexity of a given problem is unknown, or when the problem has high requirements on computational cost, users' inputs, or robustness to feature noise. However, SDRW seems less effective than CBRW to identify outliers in cleaner data sets, e.g., data with only a few noisy features and some strongly relevant features like \textit{Census} and \textit{R10}, since SDRW reduces the effect of $\delta$ by replacing $\delta$ with $\mathit{ad}$. Therefore, CBRW is recommended for outlier detection in clean data sets.

\subsection{Limitations and Possible Extensions}

\subsubsection{Issues in the Framework and Its Instantiations}
The CUOT framework and its instantiations use only one intra- and inter-feature value coupling function (i.e., $\delta$ and $\eta$) to define the outlier factors and their interdependence. This may limit its modeling capacity. Additional different coupling functions may be defined to capture other aspects of intra-feature outlier factors and their interactions. Also, we focus on pairwise inter-feature couplings, which may omit longer outlying patterns. On the other hand, simply using the patterns obtained by pattern mining approaches fail to work effectively, in particular for data with noisy features. Considering the couplings between high-order patterns may help address this issue. Therefore, instantiating our framework CUOT by incorporating arbitrary-length patterns and their complex couplings may further improve the performance in data sets with long-length patterns, but the time complexity would be a critical issue. Also, in addition to random walks, other homophily modeling techniques such as belief propagation \citep{koutra2011homophily} may be explored to better capture the inter-feature outlierness interdependence.

Additionally, we focus on estimating the outlierness of values in this work. CUOT may be extended to project the categorical values into a numeric low-dimensional outlier-resilient embedding space, in which each categorical value is represented by a low-dimensional vector, such as the work reported in \citep{jian2017embedding}, such that off-the-shelf numeric data-based learning methods can be applied to extract more useful knowledge from categorical data while being outlier-resilient. Actually, CUOT projects the categorical values onto a one-dimensional space where each categorical value is represented by a numeric outlier score, but it only attempts to capture the exceptional characteristics of the values. When CUOT captures more intrinsic data characteristics, the embeddings of values would have better representation power and facilitate stronger outlier-resilient learning performance. A large number of deep neural network-based outlier detection methods have been introduced in recent years \citep{pang2020deep}; they may be used to support the representation learning of the complex categorical data.

Note that, although our framework incorporates some feature-level information by looking at value relations within intra- or inter-features, it does not make full use of the feature information, such as feature correlation, which may be further leveraged to improve the performance. However, it is challenging to efficiently and effectively incorporate those information into our feature value-based framework. One possible way is to build a feature-feature graph and learn the outlierness from the interactions between the value graph and the feature graph.

\subsubsection{Applications}

We have only studied two applications of CBRW/SDRW in direct outlier detection and outlying feature selection. However, as CBRW/SDRW works on the very low-level value outlierness estimation, they hold great potential for other applications. For example, for the data embedding as discussed above, CBRW/SDRW can transform categorical data into numeric data by replacing each categorical value with a numeric value representing its outlierness. This would enable existing numeric data-based learning methods to work on the transformed data. Another possible application is to use the couplings between values and the value outlierness to provide an explanation of why objects are identified as outliers. Current outlier explanation approaches for categorical data \citep{angiulli2009detecting} focus on finding suspected value combinations or patterns, but only using these patterns may be ineffective due to the presence of noisy features. The complex couplings modeled by CBRW/SDRW may well complement existing approaches when data sets contain noisy features. 

This work deals with point outliers only. In practice, there may exist other types of outliers, such as contextual outliers or collective outliers \citep{chandola2009anomaly}. How to effectively consider the non-IID characteristics of those outliers remains an open problem.

Additionally, we focus on categorical data only. Extending our methods to handle data with both categorical and numeric features is also an important and interesting direction as this type of data is commonly encountered in many real-world problems. In some application scenarios, outliers exhibit outlying behaviors in both categorical and numeric features. In such cases using either categorical data or numeric data is sufficient to detect the outliers. For example, our network intrusion data sets \textit{Probe} and \textit{U2R} are taken from mixed categorical-numeric data and they are used with the categorical features only in our experiments, but our methods achieve nearly perfect AUC performance, indicating a very high coverage of outlying behaviors in the categorical data space. Categorical data-based methods would fail and the numeric data-based counterparts should be used when the outlying behaviors are only detectable in the numeric data space. The task becomes significantly more challenging when outliers can be detected only by considering the couplings between categorical and numeric features. 

\subsubsection{Data Indicators}

We observe from Table \ref{tab:datasummary} that the four data indicators are somehow correlated, e.g., data sets with complex couplings and/or strong heterogeneity are often among the most complex data according to the three other indicators. This is due to an underlying fact that different data characteristics are often correlated, which is one of the main reasons for why CBRW/SDRW can not only handle non-IID outlying behaviors but also data with high outlier inseparability and feature noise level.  

It is very difficult for single data indicators to accurately describe real-world data sets and provide insights into the performance of outlier detectors on those data. For example, a data set with low heterogeneity does not indicate it is simple, as the data may have low outlier separability and high feature noise level, such as \textit{w7a}. Therefore, multiple data indicators are required to work together to provide a more complete picture of the data. The four data indicators we introduce capture some of the intrinsic data characteristics that are sensitive to the performance of outlier detectors. These indicators enable us to more comprehensively explain the detection results on these 15 data sets, but it may not necessarily fit other data sets. More data indicators may be required to generalize to diverse domains.

\appendix
\section{Proofs of the Theorems}

\subsection{Proof of Theorem \ref{thm:convg}}

\begin{proof}

If the graph $\mathsf{G}$ is irreducible and aperiodic, then based on the Perron–Frobenius Theorem \citep{meyer2000matrix}, the URWs on $\mathsf{G}$ based on the adjacency matrix $\mathbf{A}$ will converge to a unique probability vector. 
\begin{lemma} [Equivalence between BRWs and URWs]\label{thm:lem}
BRWs based on the adjacency matrix $\mathbf{A}$ and the bias $\delta$ is equivalent to URWs on a graph $\mathsf{G}^{b}$ with an adjacency matrix $\mathbf{B}$, in which
\begin{equation}\label{eqn:adjb}
    \mathbf{B}(u,v)=\delta(u)\mathbf{A}(u,v)\delta(v), \; \forall u,v \in \mathcal{V}.
\end{equation}
\end{lemma}
This lemma holds iff the transition matrix $\mathbf{T}$ of $\mathsf{G}^{b}$ satisfies: $\mathbf{T} \equiv \mathbf{W}^{b}$. Since $\mathbf{B}(u,v)=\delta(u)\mathbf{A}(u,v)\delta(v)$, we have 
\begin{align*}
\mathbf{T}(u,v) &=\frac{\mathbf{B}(u,v)}{\sum_{v \in \mathcal{V}}\mathbf{B}(u,v)}= \frac{\delta(u)\mathbf{A}(u,v)\delta(v)}{\sum_{v\in V}\delta(u)\mathbf{A}(u,v)\delta(v)} \\
& =\frac{\mathbf{A}(u,v)\delta(v)}{\sum_{v\in V}\mathbf{A}(u,v)\delta(v)}=\mathbf{W}^{b}(u,v),
\end{align*}
which completes the proof of the lemma.

Since $\delta$ is always positive, the inclusion of $\delta$ into $\mathbf{A}$ does not change the graph's irreducibility and aperiodicity. Based on Lemma \ref{thm:lem}, $\mathbf{B}$ and $\mathbf{A}$ have the same irreducibility and aperiodicity. Therefore, if $\mathsf{G}$ is irreducible and aperiodic, so is $\mathsf{G}^{b}$. We therefore have $\boldsymbol\pi^{*}=\mathbf{W}^{b}\boldsymbol\pi^{*}$.

\end{proof}

\subsection{Proof of Theorem \ref{thm:closedform}}

\begin{proof}

To prove Equation (\ref{eqn:closedform}), we need to show that  when $\boldsymbol\pi^{\prime}(u) =\frac{d^{\prime}(u)}{\mathit{vol}(\mathsf{G}^{\prime})}$, $\forall u \in \mathcal{V}$, we have $\boldsymbol\pi^{\prime}=\mathbf{W}^{b\prime}\boldsymbol\pi^{\prime}$, i.e., $\boldsymbol\pi^{\prime}$ becomes steady w.r.t. the time step.

    First, the probability of visiting $v$ is $\boldsymbol\pi^{\prime,t+1}(v) = \sum_{u\in \mathcal{V}}\boldsymbol\pi^{\prime,t}(u)\mathbf{W}^{b\prime}(u,v)$. We then have 
$$    \boldsymbol\pi^{\prime,t+1}(v) = \sum_{u\in \mathcal{V}}\boldsymbol\pi^{\prime,t}(u)\frac{\mathbf{B}^{\prime}(u,v)}{\sum_{w \in \mathcal{V}}\mathbf{B}^{\prime}_{u,w}}.$$

When $\boldsymbol\pi^{\prime,t}(u) =\frac{d^{\prime}(u)}{\mathit{vol}(\mathsf{G}^{\prime})}$, we have
$$
    \boldsymbol\pi^{\prime,t+1}(v) =\sum_{u\in \mathcal{V}}\frac{d^{\prime}(u)}{\mathit{vol}(\mathsf{G}^{\prime})}\frac{\mathbf{B}^{\prime}(u,v)}{\sum_{w \in \mathcal{V}}\mathbf{B}^{\prime}_{u,w}}= \sum_{u\in \mathcal{V}}\frac{d^{\prime}(u)}{\mathit{vol}(\mathsf{G}^{\prime})}\frac{\mathbf{B}^{\prime}(u,v)}{d^{\prime}(u)}=\sum_{u\in \mathcal{V}}\frac{\mathbf{B}^{\prime}(u,v)}{\mathit{vol}(\mathsf{G}^{\prime})}.
$$

Since $\mathbf{B}^{\prime}(u,v)= \mathbf{B}^{\prime}(v,u)$, we further have
$$
\boldsymbol\pi^{\prime,t+1}(v) =\sum_{u\in \mathcal{V}}\frac{\mathbf{B}^{\prime}(u,v)}{\mathit{vol}(\mathsf{G}^{\prime})}=\sum_{u\in \mathcal{V}}\frac{\mathbf{B}^{\prime}(v,u)}{\mathit{vol}(\mathsf{G}^{\prime})}= \frac{d^{\prime}(v)}{\mathit{vol}(\mathsf{G}^{\prime})}.
$$

Therefore, we also have $\boldsymbol\pi^{\prime,t+1}(u)=\frac{d^{\prime}(u)}{\mathit{vol}(\mathsf{G}^{\prime})}=\boldsymbol\pi^{\prime,t}(u)$, i.e., $\boldsymbol\pi^{\prime}$ becomes steady.

\end{proof}

\subsection{Proof of Theorem \ref{thm:gapdelta}}
\begin{proof}

We show $\mathit{diff}=\left(\delta(u^{\prime}) - \phi_{\mathit{freq}}(u^{\prime})\right) - \left(\delta(v^{\prime}) - \phi_{\mathit{freq}}(v^{\prime})\right) > 0$ to complete the proof.

First, we have
\begin{align*}
\delta(u^{\prime}) - \phi_{\mathit{freq}}(u^{\prime}) &= \left(1 - \mathit{freq}(m) + \frac{\mathit{freq}(m) - \mathit{freq}(u^{\prime})}{\mathit{freq}(m)}\right) - \left(1 - \mathit{freq}(u^{\prime})\right) \\
&= \frac{\mathit{supp}(m)N-\mathit{supp}(u^{\prime})N + \mathit{supp}(m)\mathit{supp}(u^{\prime})-\mathit{supp}(m)^{2}}{\mathit{supp}(m)N}.
\end{align*}

Let $C = N - \mathit{supp}(m)$ and $H = \mathit{supp}(m) + \mathit{supp}(u^{\prime})$. Then after some algebra, we have

$$\delta(u^{\prime}) - \phi_{\mathit{freq}}(u^{\prime}) = \frac{NH-\mathit{supp}(m)H}{\mathit{supp}(m)N}.$$

Similarly, we can obtain 
$$\delta(v^{\prime}) - \phi_{\mathit{freq}}(v^{\prime}) = \frac{\mathit{supp}(m)N-\mathit{supp}(m)^{2}-\mathit{supp}(v^{\prime})N+\mathit{supp}(v^{\prime})\mathit{supp}}{\mathit{supp}(m)N}.$$

Let $\mathit{supp}(v^{\prime}) = \mathit{supp}(m) - I$. Then after some algebra, we obtain
$$\delta(v^{\prime}) - \phi_{\mathit{freq}}(v^{\prime}) = \frac{NI-\mathit{supp}(m)I}{\mathit{supp}(m)N}.$$

Therefore, 
\begin{align*}
\mathit{diff} &= \frac{NH-\mathit{supp}(m)H}{\mathit{supp}(m)N} - \frac{NI-\mathit{supp}(m)I}{\mathit{supp}(m)N} \\
&= \frac{(N-\mathit{supp}(m))(H- I)}{\mathit{supp}(m)N}.
\end{align*}

We always have $N>\mathit{supp}(m)$. Moreover, as $u^{\prime}$ and $v^{\prime}$ are outlying and normal values respectively, we have $\mathit{supp}(u^{\prime}) < \mathit{supp}(v^{\prime})$, and thus $H > I$. Therefore, $\left(\delta(u^{\prime}) - \phi_{\mathit{freq}}(u^{\prime})\right) - \left(\delta(v^{\prime}) - \phi_{\mathit{freq}}(v^{\prime})\right)>0$.

\end{proof}

\section{Convergence and Sensitivity Test Results of CBRW}\label{sec:selfcbrw}

The empirical convergence analysis for CBRW is provided in the first subsection, followed by the sensitivity test of CBRW w.r.t. the parameter $\alpha$.

\subsection{Convergence Test}

The convergence rate of random walks is governed by two key graph properties - the graph diameter and the Cheeger constant \citep{diaconis1991cheegerbound,fill1991cheegerbounddigraph}. The runtime for computing the Cheeger constant is prohibitive for large graphs, so we replace this constant with clustering coefficients. The graph's diameter and clustering coefficients of the value graph for each data set are presented in Table \ref{tab:valuegraph}. It is clear that all the value graphs has small graph diameter and large clustering coefficient. This is because a value in one feature often co-occurs with most, if not all, of the values in other features. Moreover, there exist linkages between values as long as the values co-occur together, resulting in a highly connected dense value graph. Fast convergence rates are expected for random walks on such graphs \citep{diaconis1991cheegerbound,fill1991cheegerbounddigraph}. 

The convergence test results in Fig. \ref{fig:convergence} show that CBRW converges quickly on all 15 data sets, i.e., within 70 iterations. CBRW converges after about 10 iterations on 13 data sets, but takes about 70 iterations to converge on \textit{Probe} and \textit{U2R}. This is because these two data sets contain a large proportion of feature values having frequencies of less than three. This is particularly true for \textit{Probe}. As a result, although their overall clustering coefficient is high, its Cheeger constant can be quite small, which leads to slower convergence.

\begin{figure}[h!]
\CenterFloatBoxes
\begin{floatrow}
\ttabbox
  {\caption{Two Key Properties of a Value Graph. Data is sorted by clustering coefficient. `$\circ$' indicates out-of-memory exceptions.}\label{tab:valuegraph}}
  {
  \scalebox{0.80}{
  \begin{tabular}{|l|c|c|}
    \hline
    \textbf{Data}  & \textbf{Diameter} & \textbf{Coefficient} \\
    \hline    
    Census & 2  & 0.76 \\
    Chess & 2  & 0.79 \\
    U2R & 2  & 0.80 \\
    SF & 2  & 0.81 \\
    Probe & 2  & 0.82 \\
    BM & 2  & 0.85 \\
    LINK & 2  & 0.86 \\
    CT & 2  & 0.87 \\
    CMC & 2  & 0.89 \\
    APAS & 2  & 0.90 \\
    R10 & 2  & 0.91 \\
    AID362 & 2  & 0.92 \\
    w7a & 2  & 0.93 \\
    CelebA & 2  & 0.99 \\
    AD & $\circ$ & $\circ$ \\
    \hline
  \end{tabular}
  }
  }

\ffigbox
  {\includegraphics[width=0.35\textwidth]{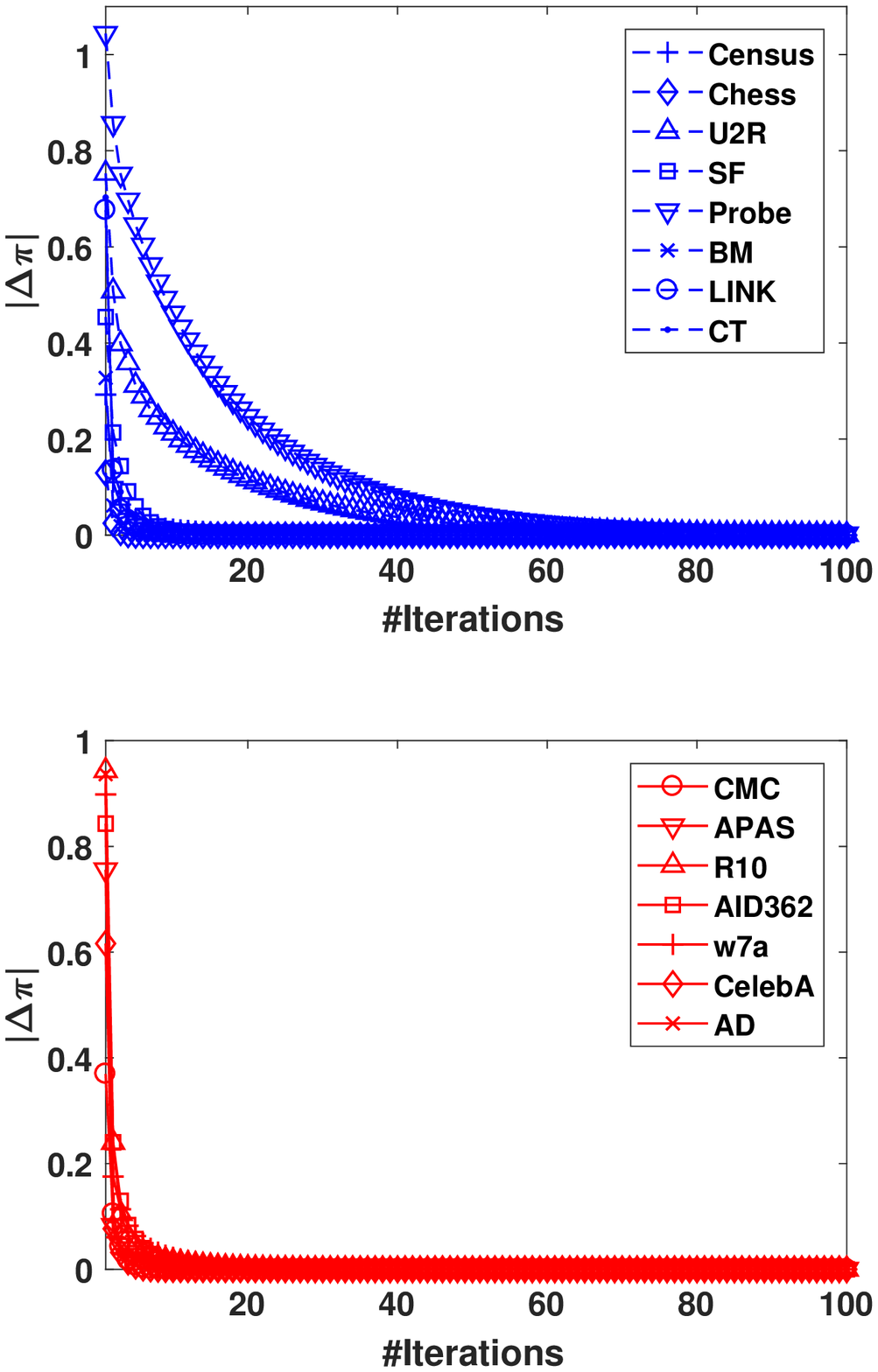}}
  {\caption{Convergence Test Results.}\label{fig:convergence}}

\end{floatrow}
\end{figure}

\subsection{Sensitivity Test w.r.t. the Damping Factor $\alpha$}\label{subsec:alpha}

CBRW only has one parameter, the damping factor $\alpha$. The use of $\alpha$ is to avoid the random walking getting stuck in isolated nodes by offering a small restart probability $(1-\alpha)$, which guarantees the algorithmic convergence while does not affect the effectiveness. $\alpha=1.0$ is not recommended as this may break the convergence condition. Also, $\alpha$ should be sufficiently large, e.g., $\alpha\geq 0.85$, and the underlying graph structure is ignored otherwise. Below we examine the sensitivity of CBRW w.r.t. $\alpha$ in a wide range of values $[0.85,0.99]$ by performing direct outlier detection (i.e., $\text{CBRW}_{\text{od}}$ is used). Fig. \ref{fig:sentest} reports the AUC results w.r.t. $\alpha$ on all 15 data sets.  

\begin{figure}[h!]
  \centering
    \includegraphics[width=0.70\textwidth]{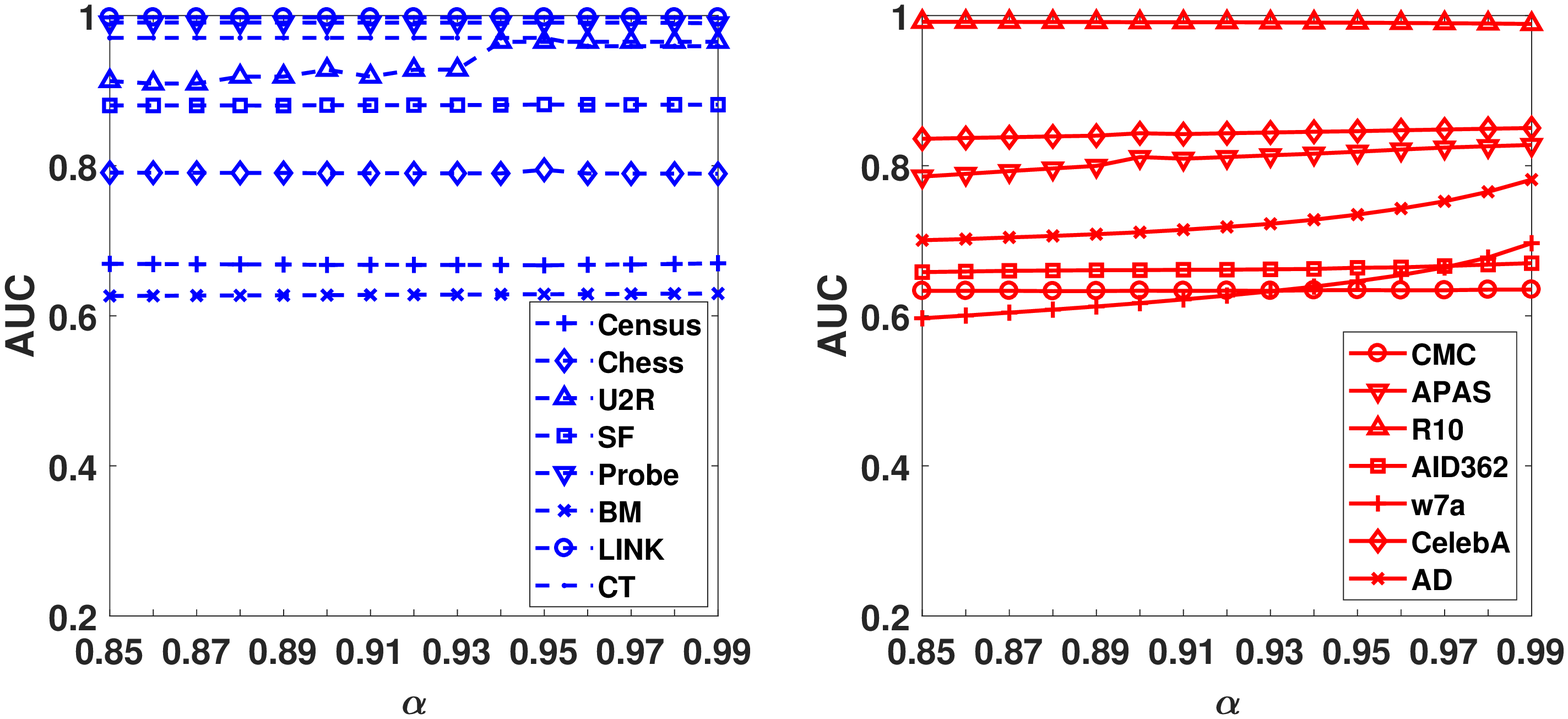}
  \caption{Sensitivity Test Results w.r.t. the Parameter $\alpha$. }
  \label{fig:sentest}
\end{figure}

The results show that CBRW performs very stably over a large range of tuning options on most of the data sets, and a large $\alpha$ is more preferable than a small one. This is because (i) $\alpha$ is introduced to guarantee the convergence of the CBRW algorithm and it is data-insensitive in terms of effectiveness, which is different from some data-sensitive parameters in other detectors, such as the minimum support in FPOF and the subsampling size in iForest; and (ii) the graph structure and edges weights are carefully designed to highlight the outlying values, and we need to make use of this graph nature by setting a large $\alpha$. A large $\alpha$ is needed to achieve the best performance on some data sets, e.g., \textit{U2R}, \textit{APAS}, \textit{w7a} and \textit{AD}. These data sets may contain some highly noisy values. A large $\alpha$ is required to increase the gap between the outlierness of outlying values and the highly noisy values. On the other hand, a medium $\alpha$ is needed to obtain the best performance on other data sets, like \textit{CT}. This may be because some outlying values in these data sets cannot attract sufficiently large outlierness in the original graph structure, but rather rely on some outlierness propagated through restart probabilities. Therefore, we recommend using a relatively large $\alpha$ (e.g., $\alpha=0.95$) to leverage both cases.

\section{Key Outputs of CBRW and SDRW}

We provide several key outputs of CBRW to enable an in-depth understanding of its algorithmic procedures. All these outputs are built upon the toy dataset in Table \ref{tab:toyexample}. In Table \ref{tab:tm} we present the transition matrix used by the BRWs in CBRW, i.e., $\mathbf{W}^b$ in Equation (\ref{eqn:brws1}), in which each entry $\mathbf{W}^b(u, v)$ is determined by the inter-feature outlierness influence $\eta$ and the intra-feature outlier factor $\delta$. After having $\mathbf{W}^b$, the power iteration method is used to perform random walks and obtains a stationary probability vector $\boldsymbol\pi^{*}$. In CBRW, these stationary probabilities are used as the outlierness of the values in the value-value graph according to Equation (\ref{eqn:score}). Table \ref{tab:value_outlierness} shows the outlierness of each value of the toy dataset, with each outlierness corresponding to one entry of $\boldsymbol\pi^{*}$. It is then followed by the calculation of the outlierness of data objects in Equation (\ref{eqn:od}) based on the value outlierness. Finally, CBRW produces the outlierness of all data objects as in Table \ref{tab:object_outlierness}. The first data object that is the only genuine outlier is assigned with larger outlierness than all data objects, including the noisy data object \#10.

\begin{table}[htbp]
  \centering
  \ttabbox{\caption{Transition Matrix Resulted in CBRW for the Value Graph Derived from Table \ref{tab:toyexample}.}\label{tab:tm}}
  {
  \scalebox{0.68}{
    \begin{tabular}{@{}cc|cc|ccc|ccc|ccc}
    \hline \hline
          &       & \multicolumn{2}{c|}{$\mathsf{F}_1$} & \multicolumn{3}{c|}{$\mathsf{F}_2$} & \multicolumn{3}{c|}{$\mathsf{F}_3$} & \multicolumn{3}{c}{$\mathsf{F}_4$} \\ 
          &       & \textbf{male} & \textbf{female} & \textbf{bachelor} & \textbf{master} & \textbf{PhD} & \textbf{married} & \textbf{single} & \textbf{divorced} & \textbf{low} & \textbf{medium} & \textbf{high} \\ \hline
    \multirow{2}[0]{*}{$\mathsf{F}_1$} & \textbf{male} & 0.0000 & 0.0000 & 0.2124 & 0.0607 & 0.0759 & 0.0637 & 0.0850 & 0.1077 & 0.1821 & 0.0303 & 0.1821 \\
          & \textbf{female} & 0.0000 & 0.0000 & 0.0000 & 0.0897 & 0.2242 & 0.1256 & 0.0628 & 0.3184 & 0.0000 & 0.1794 & 0.0000 \\ \hline
    \multirow{3}[0]{*}{$\mathsf{F}_2$} & \textbf{bachelor} & 0.1136 & 0.0000 & 0.0000 & 0.0000 & 0.0000 & 0.1591 & 0.1591 & 0.0000 & 0.4545 & 0.1136 & 0.0000 \\
          & \textbf{master} & 0.0446 & 0.1116 & 0.0000 & 0.0000 & 0.0000 & 0.0313 & 0.0937 & 0.3170 & 0.1786 & 0.0446 & 0.1786 \\
          & \textbf{PhD} & 0.0538 & 0.2688 & 0.0000 & 0.0000 & 0.0000 & 0.2258 & 0.0753 & 0.0000 & 0.0000 & 0.1613 & 0.2151 \\ \hline
    \multirow{3}[0]{*}{$\mathsf{F}_3$} & \textbf{married} & 0.0500 & 0.1667 & 0.2333 & 0.0333 & 0.2500 & 0.0000 & 0.0000 & 0.0000 & 0.0000 & 0.1333 & 0.1333 \\
          & \textbf{single} & 0.0588 & 0.0735 & 0.2059 & 0.0882 & 0.0735 & 0.0000 & 0.0000 & 0.0000 & 0.2353 & 0.0294 & 0.2353 \\
          & \textbf{divorced} & 0.0500 & 0.2500 & 0.0000 & 0.2000 & 0.0000 & 0.0000 & 0.0000 & 0.0000 & 0.4000 & 0.1000 & 0.0000 \\ \hline
    \multirow{3}[0]{*}{$\mathsf{F}_4$} & \textbf{low} & 0.0735 & 0.0000 & 0.3431 & 0.0980 & 0.0000 & 0.0000 & 0.1373 & 0.3480 & 0.0000 & 0.0000 & 0.0000 \\
          & \textbf{medium} & 0.0240 & 0.2404 & 0.1683 & 0.0481 & 0.1803 & 0.1346 & 0.0337 & 0.1707 & 0.0000 & 0.0000 & 0.0000 \\
          & \textbf{high} & 0.1471 & 0.0000 & 0.0000 & 0.1961 & 0.2451 & 0.1373 & 0.2745 & 0.0000 & 0.0000 & 0.0000 & 0.0000 \\ \hline \hline
    \end{tabular}%
    }
    }
\end{table}%

\begin{figure}[h!]
\CenterFloatBoxes
\begin{floatrow}
\ttabbox
  {\caption{Value Outlierness Yielded by CBRW.}
  \label{tab:value_outlierness}}
  {
  \scalebox{0.85}{
    \begin{tabular}{cc}
    \hline
    \textbf{Value} & \textbf{Outlierness} \\\hline
    male  & 0.0598 \\
    female & 0.0983 \\
    bachelor & 0.1075 \\
    master & 0.0794 \\
    PhD   & 0.0836 \\
    married & 0.0756 \\
    single & 0.0845 \\
    divorced & 0.1228 \\
    low   & 0.1403 \\
    medium & 0.0744 \\
    high  & 0.0739 \\\hline
    \end{tabular}%
  }
  }

\ttabbox
  {\caption{Final Object Outlierness Yielded by CBRW.}
  \label{tab:object_outlierness}}
  {
  \scalebox{0.80}{
    \begin{tabular}{ccc}
    \hline
    \textbf{ID} & \textbf{Outlierness}& \textbf{Ground Truth} \\\hline
    1     & 0.0982 & Yes\\
    2     & 0.0739 & No \\
    3     & 0.0702 & No \\
    4     & 0.0751 & No \\
    5     & 0.0863 & No \\
    6     & 0.0689 & No \\
    7     & 0.0702 & No \\
    8     & 0.0772 & No \\
    9     & 0.0690 & No \\
    10    & 0.0951 & No \\
    11    & 0.0749 & No \\
    12    & 0.0882 & No \\\hline
    \end{tabular}%
  }
  }
\end{floatrow}
\end{figure}

We also provide similar outputs for SDRW. In Table \ref{tab:tm_sdrw} we present the adjacency matrix used by SDRW, i.e., $\mathbf{C}$ in Equation (\ref{eqn:valuegraph2}). Note that the value graph in SDRW is undirected, so we have $\mathbf{C}(u,v)=\mathbf{C}(v,u)$. We then incorporate the subgraph density-based outlier factor into the matrix and calculate the value outlierness using the closed-form solution in Equation (\ref{eqn:closedform_sdrw}). The resulting value outlierness is shown in Table \ref{tab:value_outlierness_sdrw}. SDRW finally uses the same object outlierness calculation as in CBRW, i.e., Equation (\ref{eqn:od}) to obtain the object-level outlier scores. As shown in Table \ref{tab:object_outlierness_sdrw}, SDRW can also easily identify the outliers in the toy dataset.

\begin{table}[htbp]
  \centering
  \ttabbox{\caption{Adjacency Matrix Resulted in SDRW for the Value Graph Derived from Table \ref{tab:toyexample}.}\label{tab:tm_sdrw}}
  {
  \scalebox{0.68}{
    \begin{tabular}{@{}cc|cc|ccc|ccc|ccc}
    \hline \hline
          &       & \multicolumn{2}{c|}{$\mathsf{F}_1$} & \multicolumn{3}{c|}{$\mathsf{F}_2$} & \multicolumn{3}{c|}{$\mathsf{F}_3$} & \multicolumn{3}{c}{$\mathsf{F}_4$} \\ 
          &       & \textbf{male} & \textbf{female} & \textbf{bachelor} & \textbf{master} & \textbf{PhD} & \textbf{married} & \textbf{single} & \textbf{divorced} & \textbf{low} & \textbf{medium} & \textbf{high} \\ \hline
    \multirow{2}[0]{*}{$\mathsf{F}_1$} & \textbf{male} & 0.0000 & 0.0000 & 0.0122 & 0.0035 & 0.0043 & 0.0036 & 0.0049 & 0.0062 & 0.0104 & 0.0017 & 0.0104 \\
          & \textbf{female} & 0.0000 & 0.0000 & 0.0000 & 0.0087 & 0.0217 & 0.0122 & 0.0061 & 0.0308 & 0.0000 & 0.0174 & 0.0000 \\ \hline
     \multirow{3}[0]{*}{$\mathsf{F}_2$} & \textbf{bachelor} & 0.0122 & 0.0000 & 0.0000 & 0.0000 & 0.0000 & 0.0170 & 0.0170 & 0.0000 & 0.0486 & 0.0122 & 0.0000 \\
          & \textbf{master} & 0.0035 & 0.0087 & 0.0000 & 0.0000 & 0.0000 & 0.0024 & 0.0073 & 0.0247 & 0.0139 & 0.0035 & 0.0139 \\
          & \textbf{PhD} & 0.0043 & 0.0217 & 0.0000 & 0.0000 & 0.0000 & 0.0182 & 0.0061 & 0.0000 & 0.0000 & 0.0130 & 0.0174 \\ \hline
     \multirow{3}[0]{*}{$\mathsf{F}_3$} & \textbf{married} & 0.0036 & 0.0122 & 0.0170 & 0.0024 & 0.0182 & 0.0000 & 0.0000 & 0.0000 & 0.0000 & 0.0097 & 0.0097 \\
          & \textbf{single} & 0.0049 & 0.0061 & 0.0170 & 0.0073 & 0.0061 & 0.0000 & 0.0000 & 0.0000 & 0.0194 & 0.0024 & 0.0194 \\
          & \textbf{divorced} & 0.0062 & 0.0308 & 0.0000 & 0.0247 & 0.0000 & 0.0000 & 0.0000 & 0.0000 & 0.0493 & 0.0123 & 0.0000 \\ \hline
     \multirow{3}[0]{*}{$\mathsf{F}_4$} & \textbf{low} & 0.0104 & 0.0000 & 0.0486 & 0.0139 & 0.0000 & 0.0000 & 0.0194 & 0.0493 & 0.0000 & 0.0000 & 0.0000 \\
          & \textbf{medium} & 0.0017 & 0.0174 & 0.0122 & 0.0035 & 0.0130 & 0.0097 & 0.0024 & 0.0123 & 0.0000 & 0.0000 & 0.0000 \\
          & \textbf{high} & 0.0104 & 0.0000 & 0.0000 & 0.0139 & 0.0174 & 0.0097 & 0.0194 & 0.0000 & 0.0000 & 0.0000 & 0.0000 \\ \hline \hline
    \end{tabular}%
    }
    }
\end{table}%

\begin{figure}[h!]
\CenterFloatBoxes
\begin{floatrow}
\ttabbox
  {\caption{Value Outlierness Yielded by SDRW.}
  \label{tab:value_outlierness_sdrw}}
  {
  \scalebox{0.85}{
    \begin{tabular}{cc}
    \hline
    \textbf{Value} & \textbf{Outlierness} \\\hline
    male  & 0.0175 \\
    female & 0.1089 \\
    bachelor & 0.1350 \\
    master & 0.1222 \\
    PhD   & 0.0661 \\
    married & 0.0807 \\
    single & 0.0507 \\
    divorced & 0.1446 \\
    low   & 0.1446 \\
    medium & 0.0952 \\
    high  & 0.0344 \\\hline
    \end{tabular}%
  }
  }

\ttabbox
  {\caption{Final Object Outlierness Yielded by SDRW.}
  \label{tab:object_outlierness_sdrw}}
  {
  \scalebox{0.80}{
    \begin{tabular}{ccc}
    \hline
    \textbf{ID} & \textbf{Outlierness}& \textbf{Ground Truth} \\\hline
    1     & 0.1124 & Yes\\
    2     & 0.0942 & No \\
    3     & 0.0603 & No \\
    4     & 0.0870 & No \\
    5     & 0.1106 & No \\
    6     & 0.0509 & No \\
    7     & 0.0603 & No \\
    8     & 0.0701 & No \\
    9     & 0.0664 & No \\
    10    & 0.0925 & No \\
    11    & 0.0777 & No \\
    12    & 0.0886 & No \\\hline
    \end{tabular}%
  }
  }
\end{floatrow}
\end{figure}

It should be noted that the above results are built upon a simple synthetic toy dataset to demonstrate the procedure of our methods; they do not imply any bias and discrimination issues in the applications of our methods to real-world datasets.  

\begin{acknowledgements}
This work was partially supported by the Australian Research Council discovery grant (DP190101079) and ARC Future Fellowship grant (FT190100734). 
\end{acknowledgements}

\bibliographystyle{plainnat}
\bibliography{AnomalyDetection}  
\end{document}